\documentclass[smallcondensed]{svjour3}     
\RequirePackage{fix-cm}
\usepackage{multirow}
\usepackage{algorithm}
\usepackage{algorithmic}
\usepackage{rotating}
\usepackage{subcaption}
\usepackage{booktabs}
\usepackage{graphicx}
\usepackage{mathtools}
\usepackage{slashbox}
\usepackage{empheq}
\usepackage{makecell}
\usepackage{float}
\usepackage{wrapfig}

\usepackage{hyperref}

\usepackage{xcolor}
\hypersetup{
	colorlinks,
	linkcolor={red},
	citecolor={black},
	urlcolor={blue}
}
\usepackage{makecell}
\usepackage{amssymb}
\usepackage{amsmath}

\newcommand{\imageo}{{\tt ImagesI}}
\newcommand{\imaget}{{\tt ImagesII}}
\newcommand{\imageth}{{\tt ImagesIII}}
\newcommand{\digito}{{\tt DigitI}}
\newcommand{\digitt}{{\tt DigitII}}
\newcommand{\cancer}{{\tt BrCancer}}

\newcommand{\arr}{{\tt Arrythmia}}
\newcommand{\wine}{{\tt Wine}}
\newcommand{\yeast}{{\tt Yeast}}

\newcommand{\hide}[1]{}

\newcommand{\method}{{\sc x-PACS}} 
\newcommand{\p}{{\em pack}}
\newcommand{\ps}{{\em packs}}
\newcommand{\pg}{{\em packing}}
\newcommand{\sub}{{\sc SubClus}}

\newcommand{\ben}{\begin{enumerate}}
	\newcommand{\een}{\end{enumerate}}
\newcommand{\bit}{\begin{itemize}}
	\newcommand{\eit}{\end{itemize}}

\newcommand{\LINEIF}[2]{\STATE \algorithmicif\ {#1}\ \algorithmicthen\ {#2} }

\newcommand{\mA}{\mathcal{A}} 
\newcommand{\mF}{\mathcal{F}} 
\newcommand{\mN}{\mathcal{N}} 
\newcommand{\mD}{\mathcal{D}} 
\newcommand{\mP}{\mathcal{P}} 
\newcommand{\mR}{\mathcal{R}} 
\newcommand{\mS}{\mathcal{S}}

\newcommand{\mE}{\mathcal{E}}
\newcommand{\mT}{\mathcal{T}}

\newcommand{\mc}{\mathbf{c}} 
\newcommand{\mx}{\mathbf{x}} 
\newcommand{\mw}{\mathbf{w}} 
\newcommand{\mM}{\mathbf{M}} 
\newcommand{\mU}{\mathbf{U}} 
\newcommand*\widefbox[1]{\fbox{\hspace{1em}#1\hspace{1em}}}

\usepackage{multibib}
\newcites{latex}{Appendix References}%
\usepackage{amssymb}
\usepackage{balance}

\begin{document}


\title{Explaining Anomalies in Groups with Characterizing Subspace Rules}

\author{Meghanath Macha \and Leman Akoglu}


\institute{Meghanath Macha and Leman Akoglu  \at
              Heinz College, Carnegie Mellon University \\
              5000 Forbes Avenue, Pittsburgh, PA 15213\\
              \email{\{meghanam, lakoglu\}@andrew.cmu.edu}           
         }

\date{Received: date / Accepted: date}

\maketitle

\begin{abstract} 

\sloppy{
Anomaly detection has numerous applications and has been studied vastly.
We 
consider a complementary problem that has a much sparser literature: anomaly \textit{description}.
Interpretation of anomalies 
is crucial for practitioners for sense-making, troubleshooting, and planning actions. 
To this end, we present a new approach called \method~(for eXplaining Patterns of Anomalies with Characterizing Subspaces), which ``reverse-engineers'' the known anomalies by identifying (1) the groups (or patterns) that they form, and (2) the characterizing subspace and feature rules that separate each anomalous pattern from normal instances.
Explaining anomalies in groups not only saves analyst time and gives insight into various types of anomalies, but also draws attention to potentially critical, repeating anomalies. In developing \method, we first construct a desiderata for the anomaly description problem.
 From a descriptive data mining perspective, our method exhibits five desired properties in our desiderata. Namely, it can unearth anomalous patterns
 ($i$) of  {multiple different types},
($ii$) hidden in {arbitrary subspaces} of a high dimensional space, 
($iii$) {interpretable} by human analysts,
($iv$)  {different from normal} patterns of the data, and finally
($v$) succinct, providing a short data description.
No existing work on anomaly description satisfies all of these properties simultaneously.
Furthermore, \method~is highly parallelizable; it is linear on the number of data points and exponential on the (typically small) largest characterizing subspace size. 
The anomalous patterns that \method~finds constitute interpretable ``signatures'', and 
while it is not our primary goal,  they can be used for anomaly detection.
Through extensive experiments on real-world datasets, we show the effectiveness and superiority of \method~in anomaly explanation over various baselines, and demonstrate its competitive detection performance as compared to the state-of-the-art. 
}
 
\end{abstract}


\section{Introduction}
\label{sec:introduction}

Given a large dataset containing normal and {labeled} anomalous points, how can we \textit{characterize} the anomalies?
What combinations of features and feature values make the anomalies stand out?
Are there {\em anomalous  patterns}, that is, do anomalies form \textit{groups}?
How many different types of anomalies (or groups) are there, and how can we  \textit{describe} them succinctly for downstream investigation and decision-making by analysts?

Anomaly mining is important for numerous applications in security, medicine, finance, etc., for which many \textit{detection} methods exist \cite{books/sp/Aggarwal2013}.
In this work, we consider a complementary problem to this vast body of work: the problem of anomaly \textit{description}.
Simply put, we aim to find human-interpretable explanations to already identified anomalies.
Our goal is ``reverse-engineering''  known anomalies by unearthing their hidden characteristics---those that make them stand out. 

The problem arises in  a variety of scenarios, 
in which we obtain labeled anomalies, albeit no description 
of the anomalies that could facilitate their interpretation. 
Example scenarios are those where
\bit
 \item[(1)] the detection algorithm is a ``black-box'' and only provides labels, due to intellectual property or security reasons (e.g., Yelp's review filter \cite{conf/icwsm/MukherjeeV0G13}), 
 \item[(2)]  the detection algorithm does not produce an interpretable output and/or cannot explicitly identify anomalous patterns (e.g., ensemble detectors like bagged LOF \cite{conf/kdd/LazarevicK05} or isolation forest \cite{conf/icdm/LiuTZ08}), and 
 \item[(3)]  the anomalies are identified via external mechanisms (e.g., when software or compute-jobs on a cluster  crash,  loan customers default,  credit card transactions get reported by card owners as fraudulent, products get reported by consumers as faulty, etc.).
This setting also arises when security experts set up {``honeypots''} to attract malicious users, and later study their operating mechanisms (often manually). Examples 
include fake followers of honeypot Twitter accounts \cite{conf/icwsm/LeeEC11} and fraudulent bot-accounts that click honeypot ads \cite{conf/sigcomm/DaveGZ12}.
\eit


Explaining anomalies is extremely useful in practice as anomalies are to be investigated by human analysts in almost all scenarios. Interpretation of the anomalies help the analysts in sense-making and knowledge discovery, troubleshooting and decision making (e.g., planning and prioritizing actions), and building better prevention mechanisms (e.g., policy changes).

Our work taps into the gap between anomaly detection and its end usage by analysts, and introduces \method~for \textit{characterizing the anomalies in high-dimensional datasets}.
 Our emphasis is explaining the anomalies {\em in groups}\footnote{In this text, phrases `anomalous pattern', `clustered anomalies', and `group of anomalies' are interchangeable.}.
We model the anomalies to consist of ($i$) various patterns (i.e., sets of clustered anomalies) and ($ii$) outliers (i.e., scattered anomalies different from the rest).
For example in fraud, malicious agents that follow similar strategies, or those who work together in ``coalition'', exhibit similar properties and form anomalous groups. 
Bots deployed for e.g., click or email spam also tend to produce similar footprints as they follow the same source of command-and-control.
At the same time, there may be multiple groups of fraudsters or bots with different strategies.

Explaining anomalies in groups has three key advantages: (1) it saves investigation time by providing a compact explanation, rather than the analyst having to go through anomalies one by one,
(2) it provides insights into the characteristics of different anomaly types,
and (3) importantly, it draws attention to anomalies that form patterns, which are potentially more critical as they are {repetitive}.

To lay out the challenges from a data mining perspective, we first introduce a list of desired properties (Desiderata 1--5) that  approaches to the problem of anomaly explanation should satisfy. We then summarize our contributions.

\subsection{Desiderata for Anomaly Description}
\label{ssec:des}

In a nutshell, anomaly explanation methods 
should effectively characterize different kinds of anomalies present in the data, 
handle high dimensional datasets, and produce human-interpretable explanations that are distinct from normal patterns as well as succinct in length.

 {D1~~\bf Identifying \textit{different types} of anomalies:} Anomalies are generated by mechanisms other than the normal.
Since such mechanisms can vary (e.g., different fraud schemes), it is likely for the anomalies to form multiple patterns in potentially different feature subspaces. A description algorithm should be able to characterize all types of anomalies.

{D2~~\bf Handling \textit{high-dimensionality}:} Data instances typically have tens or even hundreds of features.
It is meaningful to assume that the anomalies in a pattern exhibit only a (small) fraction of features in common. In other words, anomalies are likely to ``hide'' in sparse subspaces of the full space.

{D3~~\bf \textit{Interpretable} descriptions:} It is critical that the explanation of the anomalies can be easily understood by analysts.
In other words, descriptions should convey what makes a group of instances anomalous in a human-interpretable way.

{D4~~\bf \textit{Discriminative (or detection)} power:} Explanations of anomalies should not also be valid for normal points. In other words, descriptions should be discriminative and separate the anomalies from the normal points sufficiently well.
As a result, they could also help detect future anomalies of the same type.

{D5~~\bf \textit{Succinct} descriptions:} It is particularly important to have simple and concise representations, for ease of visualization and avoiding information overload. This follows the Occam's razor principle.

\subsection{Limitations of Existing Techniques}
\label{ssec:cont}

Providing interpretable explanations for anomalies is a relatively new area of study compared to anomaly detection. However, the problem has similarities to description based techniques for imbalanced datasets. Almost all existing work in the anomaly detection literature assume anomalies to be scattered, and try to explain them one at a time \cite{conf/icde/DangANZS14,conf/pkdd/DangMAN13,conf/cikm/KellerMWB13,conf/vldb/KnorrN99,conf/aaai/KuoD16,pevny2014interpreting}. Related work on collective data description \cite{conf/pkdd/GornitzKB09,journals/ml/TaxD04}, including rare class characterization \cite{conf/sdm/HeC10,conf/icdm/HeTC10}, assume a single pattern and/or do not look for subspaces. Other closely related areas are subgroup discovery and inductive rule learners. Subgroup discovery techniques  \cite{gamberger2002expert,herrera2011overview,klosgen1996explora,klosgen2002census,conf/cikm/LoekitoB08,vreeken2011krimp} aim to describe individual classes while inductive rule learners \cite{cohen1995fast,clark1989cn2,deng2014interpreting,friedman2008predictive,hara2016making} 
focus on describing multiple classes with the aim of generalization rather than explanation. Another related line of work involves techniques for explaining black-box classifiers \cite{fong2017interpretable,koh2017understanding,montavon2017methods,ribeiro2016should} where the emphasis is on explaining the prediction of a single instance rather than a group of instances.
Moreover, none of these work has an explicit emphasis on succinct, minimal descriptions.

To the best of our knowledge and as we expand in related work in \S \ref{sec:related}, there is no existing work that provides a principled and general approach to the anomaly description problem that meet all of the goals in our desiderata adequately. 
(See Table \ref{tab:related} for an overview and comparison of related work.)

\subsection{Summary of Contributions}
\label{ssec:cont}
Our work sets out to fill the gap, with the following main contributions:

\bit
\setlength{\itemsep}{0.05in}
\item {\bf Desiderata for Anomaly Description:} We introduce a new desiderata and target five rules-of-thumb (D1--D5) for designing our approach. 

\item {\bf Description-in-Groups ({\sc DiG}) Problem:} We formulate the explanation problem as one of identifying the various groups that the anomalies form (D1) within low-dimensional subspaces (D2).

\item {\bf Description Algorithm \method:} We introduce a new algorithm that produces interpretable  rules (D3)---intervals on the features in each subspace, that are also discriminative (D4)---characterizing the anomalies in the group within a subspace but as few normal points as possible.

\item {\bf A New Encoding Scheme:} We design a new encoding-based objective for describing the anomalies in groups,
based on the minimum description length (MDL) principle \cite{Rissanen78}. Through non-monotone submodular optimization we carefully select the minimal subspace rules (D5) that require the fewest `bits' to collectively describe all the anomalies. 
\eit

\noindent
\textbf{Reproducibility:} All of our code and datasets are open-sourced at {{\url{https://github.com/meghanathmacha/xPACS}}}.

\section{Overview and Problem Statement}
\label{sec:overview}

	In a nutshell, \textit{our goal is to identify a few, small micro-clusters of anomalies hidden in arbitrary feature subspaces 
		that collectively and yet succinctly represent the anomalies and separate them from the normal points}.
	Specifically, our proposed 
	\method~finds a small set of low-dimensional hyper-ellipsoids (i.e., micro-clusters corresponding to \textit{anomalous patterns} each enclosing a subset of the anomalies),
	and reveals scattered anomalies (i.e., outliers not contained in any ellipsoid).
	
	Features that are part of the subspace in which a hyper-ellipsoid lies constitute its \textit{characterizing subspace}.
	Ranges of values these features take are further characterized by the location (center and radii) of a hyper-ellipsoid within the subspace.
	Each hyper-ellipsoid is simply a ``{\em pack} of anomalies'' (hence the name \method\footnote{{\sc X-} refers to the number of packs, which we automatically identify via our data encoding scheme (\S \ref{ssec:summarize}). 
		We use this naming convention after {\sc X-means} \cite{pelleg00xmeans}, which finds the number of k-means clusters automatically in an information-theoretic way.}).
	The rest of the paper uses `hyper-ellipsoid' or `\p'~in reference to anomalous patterns.
	
\subsection{Example \method~input-output} In Fig.~\ref{fig:eg} we show an example of the input and output of \method. We consider the face images dataset, in which \method~identifies a minimum-description packing with two anomalous patterns and an outlier. In Fig.~\ref{fig:ega}, we visualize the dataset, where pixels are dimensions/features, that contains 9 labeled `anomalies': [images 1--8] of 2 types (people w/ sunglasses or people w/ white t-shirt or both) + [image 9] an outlier (one person w/ beard). We also show [image 10], which is representative of 82 normal samples (people w/ black t-shirt w/out beard or sunglasses). In Fig.~\ref{fig:egb}, we display the anomalous patterns found by \method~(characterizing subspaces are 1-d, feature rules/intervals shown at the bottom with arrows---the smaller, the darker the pixel) together explain anomalies 1--8 succinctly; 1-d pack (left) encloses images \{1--6\}, 1-d pack (right) encloses images \{1,2,5,7,8\}. Corresponding features/pixels highlighted on enclosed images. In Fig.~\ref{fig:egc}, we plot the description length (in bits, see \S\ref{sec:mdlcoding}) of anomalies individually (0 packs), vs. w/ 1--5 packs. \method~automatically finds the best number of patterns (=2) that describe the anomalies and reveals the (unpacked) outlier [image 9].

%

\begin{figure*}[!t]
        \centering
    \begin{subfigure}[t]{0.35\textwidth} 

        \centering
        	 \hspace{-0.2in}
              \includegraphics[width = 0.95\textwidth]{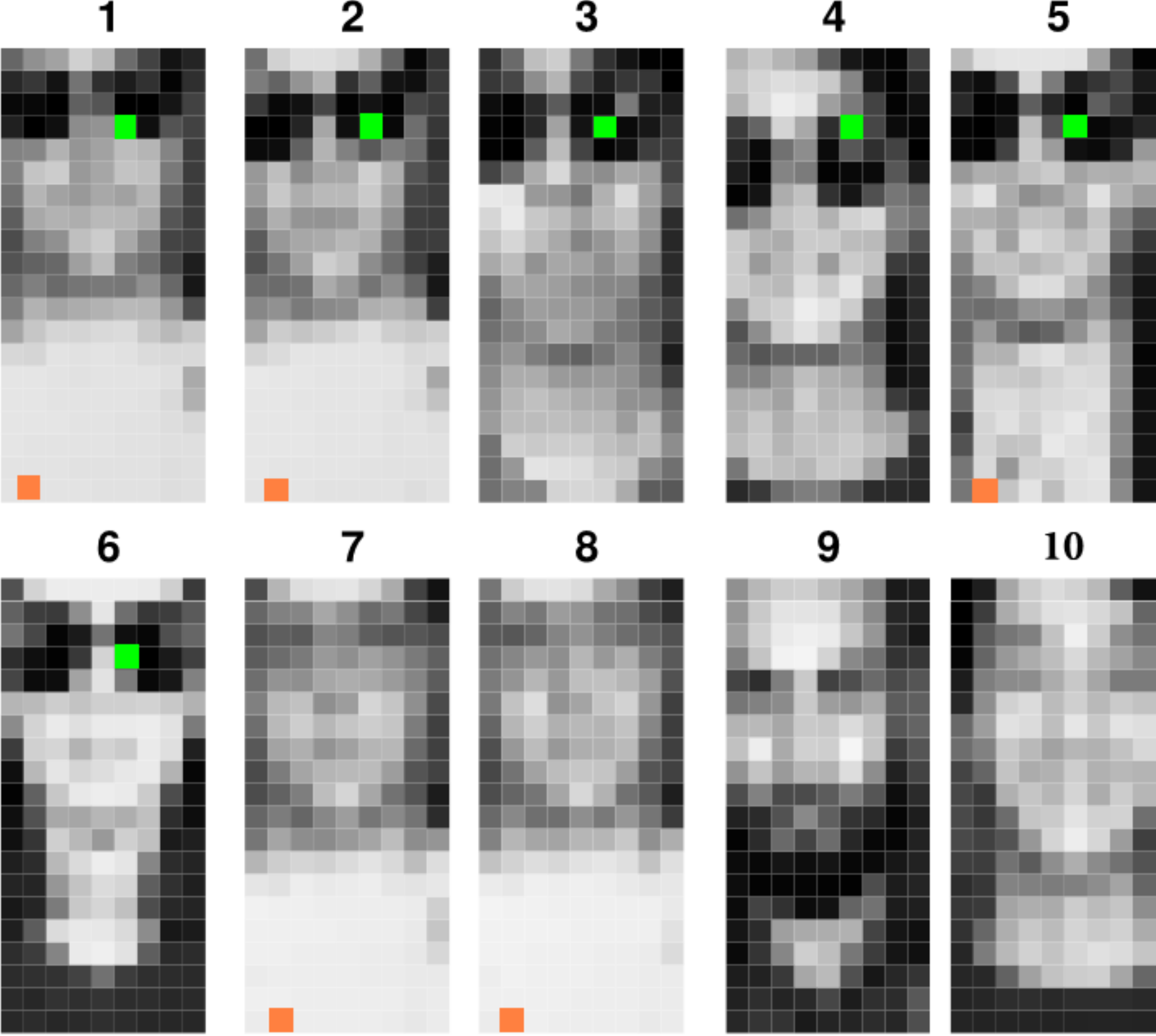}
    \hspace{-0.2in}     \caption{[1--9]: anomalies, \\and [10]: typical normal image}
        \label{fig:ega}
    \end{subfigure}%
    \begin{subfigure}[t]{0.26\textwidth} 
    
        \centering
         \includegraphics[width = 0.75\textwidth]{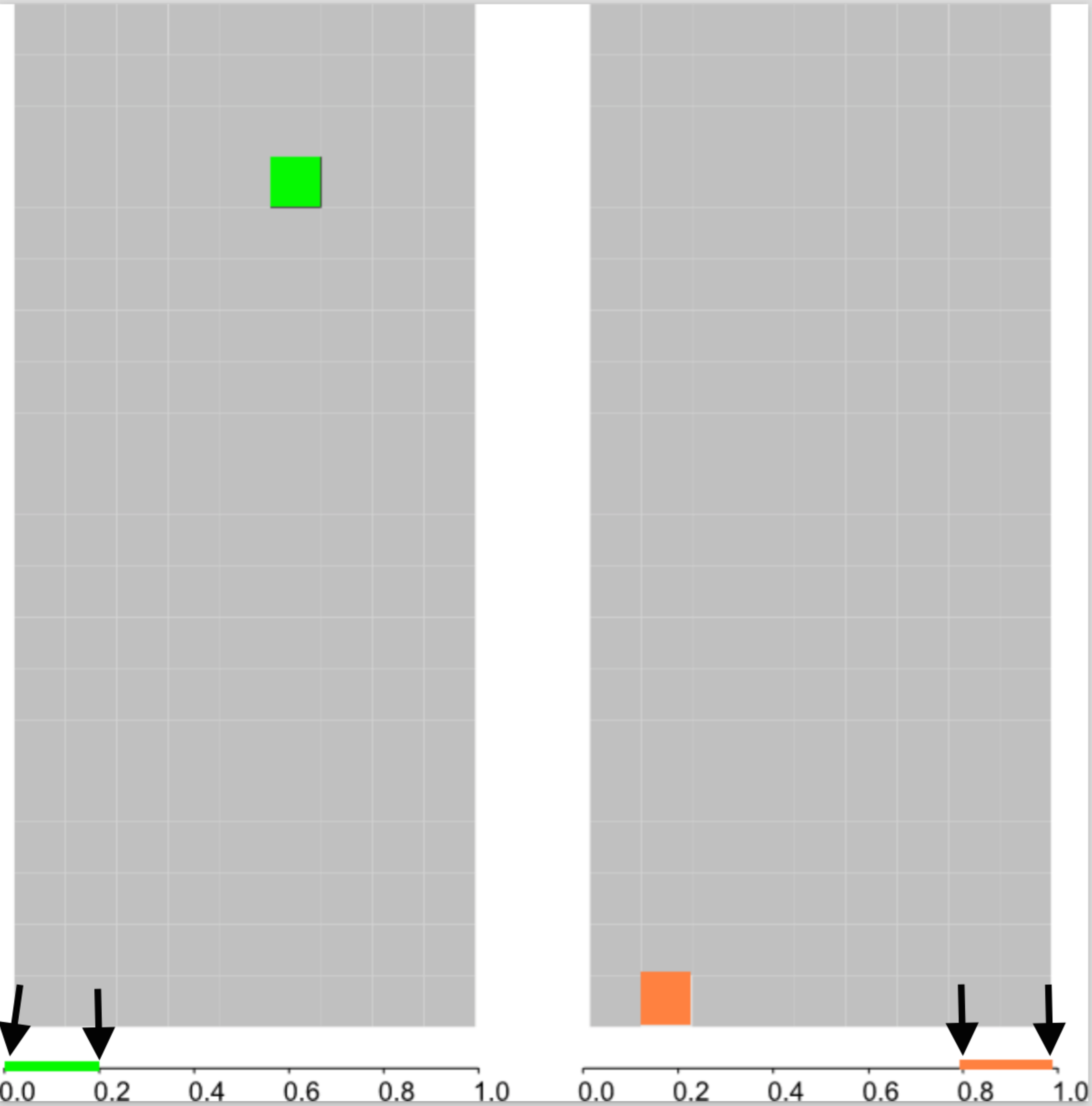}
         
        \caption{2 anomalous patterns found by \method}
        \label{fig:egb}
    \end{subfigure}%
    \begin{subfigure}[t]{0.41\textwidth} 
   
        \hspace{-0.1in}
              \includegraphics[width=0.98\textwidth]{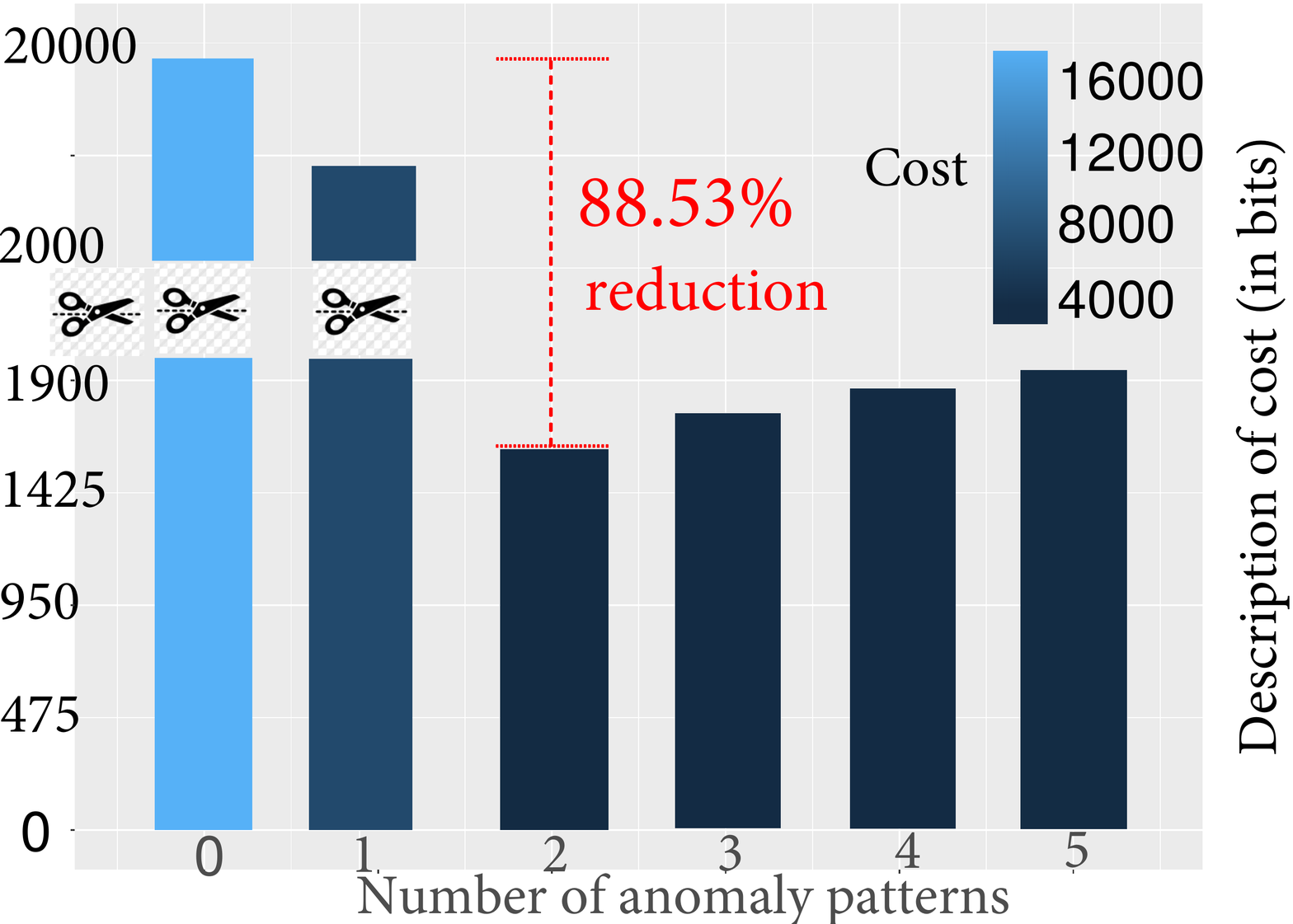}  
        \caption{description cost vs. \#\ps}
         \label{fig:egc}
    \end{subfigure}
           \hspace{0.1in}\caption{(best in color) Example \method~input--output.}
           \label{fig:eg}
\end{figure*}

\subsection{Main Steps}
	
	Our \method~consists of three main steps, each aiming to meet various criteria in our desiderata (D1--D5).
	\begin{enumerate}
		\setlength{\itemsep}{0.05in}
		\item
		First we employ subspace clustering to automatically identify
		multiple clusters of anomalies (D1) embedded in various feature subspaces.
		Advantages of subspaces are two-fold: handling ``curse of dimensionality'' (D2) and explaining each pattern with only a few features (D5).
		
		\item
		In the second step, we represent anomalies in each subspace cluster by an axis-aligned hyper-ellipsoid.
		Ellipsoids, in contrast to hyperballs, allow for varying spread of anomalies in each dimension. Axis-alignment ensures interpretable explanation with original features, which typically have real meaning to a user (D3).
		Moreover, we introduce a convex formulation to ensure that the ellipsoids are ``pure'' and enclose very few non-anomalous points, if at all, such that the characterization is discriminative (D4).
		
		\item
		Final step is summarization, where we strive to generate minimal descriptions for ease of comprehension (D5). To decide which patterns describe the anomalies most succinctly, we introduce an encoding scheme based on the Minimum Description Length (MDL) principle \cite{Rissanen78}.
		Our encoding-based objective lends itself to non-monotone submodular function maximization. Using an algorithm with approximation guarantees we identify a short list of patterns (hyper-ellipsoids) that are 
		($i$) compact with small radii, i.e., range of values that anomalies take per feature in the characterizing subspace is narrow;  
		($ii$) non-redundant, which ``pack'' (i.e., enclose) mostly different anomalies in various subspaces, and
		($iii$) pure, which enclose either none or only a few normal points.
		Importantly, the necessary number of packs is automatically identified based on the MDL criterion.
		
	\end{enumerate}

	{\bf Remark:} Note that while \method~identifies descriptive patterns of the anomalies, those can also be used for detection.
	Each pattern, along with its characterizing features and its enclosing boundary within that subspace can be seen as a discriminative {\em signature} (or set of rules), and can be used to label future instances---a new instance that falls within any of the \ps~is labeled as anomalous. Instead of a single signature or an abstract classifier function or model, however, \method~identifies multiple, interpretable signatures.


\subsection{Notation and Definitions}


Input dataset is denoted with $\mD=\{(\mx_1,y_1),\ldots,(\mx_m,y_m)\}$, containing $m$ points in $d$-dimensions, where $\mF$ depicts the feature set. A subset $\mA\subset \mD$ of points are labeled as $y_{\mA}=$ `anomalous', $|\mA|=a$. The rest are $y_{\mD\backslash \mA}=$ `normal' points, denoted by $\mN$, $|\mN|=n$, $a+n=m$.

Our goal is to find ``enclosing shapes'', called \ps, that collectively contain as many of the anomalies as possible. While arbitrary shapes would allow for higher flexibility, we restrict these shapes to the hyper-ellipsoids family for ease of interpretation.
This is not a strong limitation, however, since anomalous patterns are expected to form compact micro-clusters in some feature subspaces, rather than lie on arbitrarily shaped manifolds. A \p~is formally defined as follows.
  
\begin{definition}[pack]
	\label{def:pack}
	A \p~$p_k$ is a hyper-ellipsoid in a feature subspace $\mF_k\subseteq \mF$, $|\mF_k| = d_k$,  characterized by its center $\mc_k \in \mathbb{R}^{d_k}$ and matrix $\mM_k \in \mathbb{R}^{d_k\times d_k}$ where
	$$
	p_k(\mc_k,\mM_k) = \{ \mx| \; (\mx-\mc_k)^T \mM_k^{-1} (\mx-\mc_k) \leq 1 \} \;.
	$$ 
	We denote the anomalies that $p_k$ encloses by $\mA_k \subseteq \mA$, and the normal points that it encloses by $\mN_k\subset \mN$.
\end{definition}

\begin{definition}[packing]
A \pg~$\mP$ is a collection of \ps~as defined above; $\mP = \{p_1(\mc_1,\mM_1), \ldots, p_K(\mc_K,\mM_K)\}$ with size $K$.
\end{definition}

\subsection{Problem Statement}

Based on the above definitions, our description-in-groups problem is formally:

\begin{problem}
	\label{prob}
	\textbf{Given} a dataset $\mD \in \mathbb{R}^{m\times d}$ containing $a$ anomalous points in $\mA$ and $n$ non-anomalous or normal points in $\mN$, $a\ll n$; 
	
	\textbf{Find} a set of anomalous patterns (\ps) $\mathcal{P} = \{p_1, p_2, \ldots, p_K\}$, each containing/enclosing a subset of the anomalies $\mA_k$,
	where $\bigcup_{1\leq k\leq K} \mA_k \subseteq \mA$,
	
	\textbf{such that} $\mP$ provides the minimum description length $L(\mA|\mD,\mP)$ (in bits) for the anomalies in $\mD$. (We introduce our MDL-based encoding scheme and cost function $L(\cdot)$ later in \S \ref{ssec:summarize}.)
\end{problem}

Note that while packs enclose different subsets of anomalies in general, any two packs can have some anomalous points in common (since an anomaly can be explained in different ways), i.e., $\mA_k\cap \mA_l \neq \emptyset \;\; \exists k,l$. Packs can also share common features in their subspaces (as different types of anomalies may share some common characteristics), i.e., $\mF_k\cap \mF_l \neq \emptyset \;\; \exists k,l$.
Moreover, the enclosing boundary of a \p~ may also contain some non-anomalous points.
These issues related to the redundancy and purity of the packs would play a key role in the ``description cost'' of the anomalies. When it comes to identifying a small set of \ps~out of a list of candidates, we formulate an encoding scheme as a guiding principle to selecting the smallest, least redundant, and the purest collection of \ps~that would yield the shortest description of all the anomalies.


\section{\method: Explaining Anomalies in Groups}
\label{sec:projections}

Next we present the details of \method, which consists of three building blocks:
\bit
\setlength{\itemindent}{.15in}
\setlength{\itemsep}{0.05in}
\item [\S\ref{ssec:rectangle}] \textbf{Subspace Clustering}: Identify clusters of anomalies in various subspaces
\item [\S\ref{ssec:refine}] \textbf{Refinement}: Transform box-like subspace clusters to pure and compact hyper-ellipsoids (or \ps)
\item [\S\ref{ssec:summarize}] \textbf{Summarization}: Select subset of \ps~that yields the minimum description length of anomalies
\eit
We present our algorithms for each of these next.

\subsection{Subspace Clustering: Finding Hyper-rectangles}
\label{ssec:rectangle}

In our formulation, we allow for anomalies to form multiple patterns, intuitively each containing anomalies of a different kind.
We model anomalous patterns as compact ``micro-clusters'' in various feature subspaces.

In the first step, we use a subspace clustering algorithm, similar to CLIQUE \cite{conf/sigmod/AgrawalGGR98} and ENCLUS \cite{conf/kdd/ChengFZ99}, that discovers subspaces with high-density anomaly clusters in a bottom-up, Apriori fashion.
There are two main differences that we introduce. First, while prior techniques focus on a density (minimum count or mass) criterion, we use two criteria: ($i$) mass and ($ii$) purity, in order to find clusters that respectively contain many anomalous points, but also a low number of normal points. 
Second, we do not enforce a strict grid over the features but find varying-length high-density intervals through density estimation in a data-driven way. 

Simply put, 
the search algorithm starts with identifying 1-dimensional intervals in each feature that meet a certain mass threshold.
These intervals are then combined to generate 2-dimensional candidate rectangles.  In general, $k$-dimensional hyper-rectangles are generated by merging $(k-1)$-dimensional ones that meet the mass criterion in a hierarchical fashion.
Thanks to the monotonicity property of  mass, the search space is pruned efficiently.
Hyper-rectangles generated during the course of the bottom-up algorithm that meet \textit{both} the mass and purity criteria are reported as clusters.
A hyper-rectangle is formally defined as follows.

\begin{definition}[hyper-rectangle]
	\label{def:hrec}
	Let $\mF = f_1 \times f_2 \times \ldots \times f_d$ be our original $d$-dimensional numerical feature space.
	A hyper-rectangle $R = (s_1, s_2, \ldots, s_{d'})$, $d'\leq d$, resides in a space $f_{t_1}\times f_{t_2} \times \ldots \times f_{t_{d'}}$ where $t_i<t_j$ if $i<j$, and has $d'$ sides, $s_z = [lb_z,ub_z]$, that correspond to individual intervals with lower and upper bounds in each dimension.
	A point $\mathbf{x}=\langle x_1, x_2, \ldots, x_d \rangle$ 
	is said to be contained or enclosed in hyper-rectangle $R = (s_1, s_2, \ldots, s_{d'})$, if $lb_z\leq \mathbf{x}_{t_z} \leq ub_z$ $\;\forall z =\{1,\ldots, {d'}\}$.
\end{definition}

		\begin{algorithm}[!t]{\caption{{\sc SubClus}~($\mD, ms,\mu$)}
				\label{alg:subclus}}
			{
				\begin{algorithmic}[1]
					\REQUIRE dataset $\mD=\mA\cup \mN \in \mathbb{R}^{m\times d}$ with labeled anomalous and normal points, mass threshold $ms \in \mathbb{Z}$, purity threshold $\mu \in \mathbb{Z}$
					\ENSURE set of hyper-rectangles $\mR = \{R_1, R_2, \ldots\}$ each containing min. $ms$ anomalous \& max. $\mu$ normal points \\
					\STATE Let $\mR^{(k)}$ denote $k$-dimensional hyper-rectangles. 
					Initialize $\mR^{(1)}$ by kernel density estimation with varying quantile thresholds in $q=\{80,85,90,95\}$, set $k=1$
					
					\FOR{{\bf each} hyper-rectangle $R \in \mR^{(k)}$}
					\IF {$\text{mass}(R) \geq ms$}
					\LINEIF{$\text{impurity}(R) \leq \mu$}{$\mR_{pure}^{(k)} = \mR_{pure}^{(k)} \cup R$\\ {\bf else } $\mR_{\neg pure}^{(k)} = \mR_{\neg pure}^{(k)} \cup R$}
					\ENDIF
					\ENDFOR
					\STATE $\mR=\mR\cup \mR_{pure}^{(k)}$
					\STATE $\mR^{(k+1)} := \text{generateCandidates}(\mR_{pure}^{(k)} \cup \mR_{\neg pure}^{(k)} )$
					\LINEIF{$\mR^{(k+1)}=\emptyset$}{
						{\bf return} $\mR$} 
					\STATE  $k=k+1$, go to step 2
					
				\end{algorithmic}
			}
		\end{algorithm}

The outline of our subspace clustering  is  in Algorithm \ref{alg:subclus}.
It takes  dataset $\mD$ as input with anomalous and normal points, a mass threshold $ms$ equal to the minimum number of required anomalous points and a purity threshold $\mu$ equal to the maximum number of allowed normal points to be contained inside, and returns hyper-rectangles that meet the desired criteria.

To begin (line 1), we find 1-dimensional candidate hyper-rectangles, equivalent to intervals in individual features.
To create promising candidate intervals initially, we find dense intervals with many anomalous points.
To this end, we perform kernel density estimation (KDE\footnote{KDE involves two parameters - the number of points sampled to construct the smooth curve and the kernel bandwidth. We set the sample size to 512 points and use the Silverman's rule of thumb \cite{silverman2018density} to set the bandwidth.}) on the anomalous points and extract the intervals of significant peaks.\footnote{For categorical features, we would instead use histogram density estimation.} This is achieved by extracting the contiguous intervals in each dimension with density larger than the $q$-th quantile of all estimated densities. $q$ is varied in $[80,95]$ to obtain candidate intervals of varying length. An illustration is given in Fig. \ref{fig:kde}. Since multiple peaks may exist, multiple intervals can be generated per dimension  as $q$ is varied. 


\begin{figure}[h]
	\centering
	\includegraphics[trim=0cm 1cm 0cm 0cm,  width=0.7\textwidth]{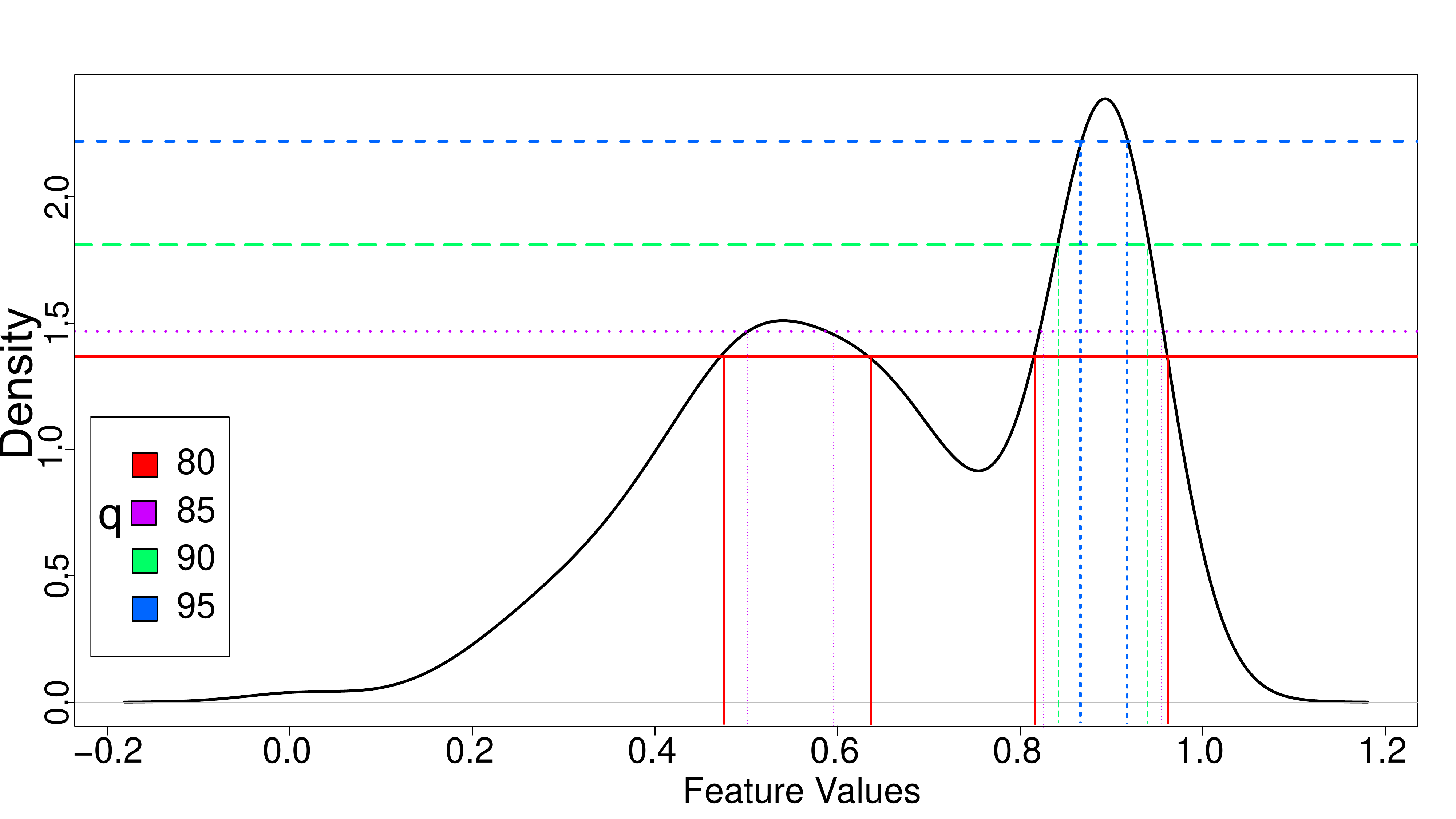}
		\caption{Identifying candidate hyper-rectangles in 1-d (equivalent to intervals) by KDE for varying quantile thresholds $q$.}
	\label{fig:kde} 
\end{figure}

At any given level (or iteration) of the Apriori-like {\sc SubClus} algorithm, we scan all the candidates at that level (line 2--6) and filter out the ones that meet the mass criterion (line 3).
Those that pass the filter are later merged to form candidates for the next level. Others with mass less than required are discarded, with no implications on accuracy. 
The correctness of the pruning procedure follows from the {\em downward closure property} of the mass criterion: for any $k$-dimensional hyper-rectangle with mass $\geq ms$, its projections in any one of $(k- 1)$-dimensions must also have mass $\geq ms$.

At each level, we also keep track of the hyper-rectangles that meet both the mass and the purity criteria (line 4). Purity exhibits the {\em upward closure property}: for any $(k-1)$-dimensional hyper-rectangle that is pure (i.e. contains $\leq \mu$ normal points), any $k$-dimensional hyper-rectangle that subsumes it is also pure.
This property could help us stop growing pure candidates by excluding them from the candidate generation step and speeding up the termination. While correct, however, such early-termination would prevent us from finding even purer hyper-rectangles later up in the hierarchy. 
To obtain as many candidate \ps~as possible, we continue our search for all hyper-rectangles that meet the mass criterion, and use the purity criterion for selecting the ones to be output (line 7).

The algorithm proceeds level by level. Having identified $k$-dimensional hyper-rectangles that satisfy the mass criterion, denoted $\mR_{\geq ms}^{(k)}=\mR_{pure}^{(k)} \cup \mR_{\neg pure}^{(k)}$ (respectively for pure and not-pure sets), $(k+1)$-dimensional candidates are generated (line 8) in two steps: join and prune.
The join step combines hyper-rectangles having first $(k-1)$ dimensions as well as sides in common.
That is, if $(s_{u_1},s_{u_2},\ldots,s_{u_k})$ and $(s_{v_1},s_{v_2},\ldots,s_{v_k})$ are two $k$-dimensional hyper-rectangles in $\mR_{\geq ms}^{(k)}$, we require $u_i = v_i$ and $s_{u_i} = s_{v_i}$ $\forall i \in \{1,\ldots,(k-1)\}$ and $u_k < v_k $ to form candidate $(k+1)$-dimensional hyper-rectangles of the form $(s_{u_1},s_{u_2},\ldots,s_{u_k},s_{v_k})$.
The prune step discards all $(k+1)$-dimensional hyper-rectangles that have a $k$-dimensional projection outside $\mR_{\geq ms}^{(k)}$. Again, the
correctness of this procedure follows from the downward closure property of mass.

\hide{
Finally, we filter out dominated hyper-rectangles before we return all of them that meet our mass and purity criteria (line 9). We define a hyper-rectangle $r \in S$ to be dominated, if there is some higher dimensional hyper-rectangle $r'$ in $S$ that $r$ subsumes, i.e. $r'\subseteq r$, where $mass(r') = mass(r)$.
In other words, provided equal mass, we prefer higher dimensional hyper-rectangles.
The reasons are two-fold: First,
$impurity(r') \leq impurity(r)$ as $r'\subseteq r$, i.e. $r'$ is likely also purer by the upward closure property of purity. Second, given that the two explain the same anomalies, we prefer the one with the more specific/detailed description.
}

{\bf Choice of $(ms,\mu)$:} To obtain hyper-rectangles of varying size and quality, packing potentially different anomalies (and non-anomalies), we run Algorithm \ref{alg:subclus} with ``conservative'' parameters, i.e., low $ms$ and high $\mu$. As such, to generate a good volume of candidates, we set $(ms,\mu)$ as the median of the number of anomalous points, normal points from the 1-dimensional hyper rectangles. Setting a higher $ms$ and lower $\mu$ would prune more (and potentially undesirably many) candidates in exchange of reduced time. We use the median to strike a balance between the quality and running time. As we describe later in \S \ref{ssec:summarize}, all these candidate packs are forwarded to a selection algorithm, which carefully chooses the subset that yields the shortest description of all the anomalies.
As such, even though there are parameters input to Algorithm \ref{alg:subclus}, we do not expect them from the user, rather we vary and set those so as to generate various candidate packs. Having more candidates is likely to increase our chance of finding a combination that explains the anomalies the best (i.e., fewest bits).

To conclude the description of the first step of our proposed approach, we note that other subgroup discovery techniques that aim to find subgroups in a given class, such as Krimp \cite{vreeken2011krimp}, could be used as an alternative to Algo. \ref{alg:subclus} provided necessary modifications are incorporated to enforce the purity criterion.


\subsection{Refining Hyper-rectangles into Hyper-ellipsoids}
\label{ssec:refine}

Grid or interval-based subspace clustering algorithms are limited to finding box-shaped rectangular clusters,  
and they may miss clusters inadequately oriented or shaped.
To allow  more flexibility, we refine each hyper-rectangle found by \sub~into a hyper-ellipsoid (which we call a \p, recall Definition \ref{def:pack}).
An ellipsoid with center $\mc$ is written as 
$$
p(\mc,\mM) = \{ \mx| \; (\mx-\mc)^T \mM^{-1} (\mx-\mc) \leq 1 \}
$$
for positive semi-definite matrix $\mM \succ 0$.

Given a hyper-rectangle $R$, let us denote the anomalous points it contains by $\mx_i \in \mA$ for $i=1, \ldots, a_R$ (See Def.n \ref{def:hrec})
and anomalous points outside $R$ by $\mx_j \in \mA$ for $j=a_{R+1}, \ldots, a$.
The normal points are denoted by $\mx_l \in \mN$ for $l = 1,\ldots, n$.

When we convert a given $R$ to an ellipsoid, we would like all $\mx_i$'s (anomalous points) it already contains to reside inside the ellipsoid. In contrast, we would like all $\mx_l$'s (normal points) to remain outside the ellipsoid.   
The refinement is achieved by enclosing as many as the other anomalous points ($\mx_j$'s) that are in the vicinity of $R$ inside the ellipsoid as well. Those would be the points that were left out due to axis-aligned interval-based box shapes that hyper-rectangles are limited to capture. An illustration is given in Fig. \ref{fig:refine}.

\begin{figure}[h]
	\centering
	\includegraphics[width = 0.55\textwidth]{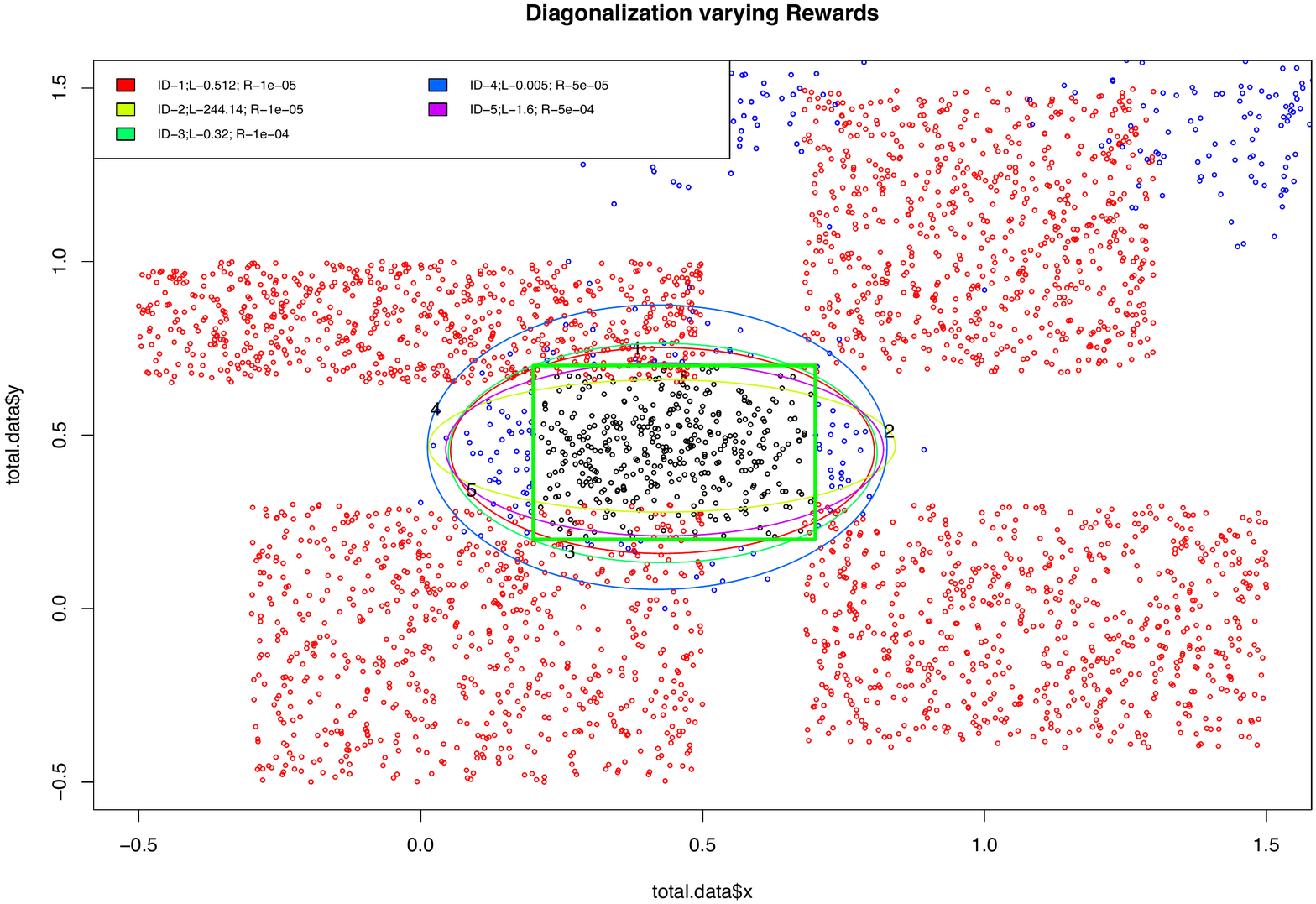}
	\caption{Example illustration of refining hyper-rectangles to ellipsoids in 2-d. Anomalous points (black) captured by {\sc SubClus} (Alg. \ref{alg:subclus}) in a (green) rectangle, other anomalous points (blue) in the vicinity, and normal points (red).}
	\label{fig:refine} 
	\vspace{-0.25in}
\end{figure}

First we describe our approach for $\mx_i$'s and $\mx_l$'s, the positive and negative points that we respectively aim to include and exclude.
The goal is to find a discriminating function $h(\cdot)$ where $h(\mx_i) > 0$ and $h(\mx_l)<0$.
To this end, we use the quadratic function $h(\mx) = \mx^T\mU\mx + \mw^T\mx + w_0$, with parameters $\mathbf{\Theta}=\{\mU,\mw,w_0\}$.
We solve for $\mathbf{\Theta}$ by setting up an optimization problem based on a semi-definite program (SDP), that satisfies $\mx_i^T\mU\mx_i + \mw^T\mx_i + w_0 >0$ for all $i$ and $\mx_l^T\mU\mx_l + \mw^T\mx_l + w_0 <0$ for all $l$. Most SDP solvers do not work well with strict inequalities, thus we modify to a non-strict feasibility problem by adding a margin, and solve (for each $R$):

{{
\begin{align*}
\min_{\mU,\mw,w_0} & \;\; \sum \varepsilon_i + \lambda \sum \varepsilon_l \\
s.t. \;\; & \;\; \mx_i^T\mU\mx_i + \mw^T\mx_i + w_0 \geq 1 - \varepsilon_i, \;\;\;\; i=1,\ldots,a_R \\
& \;\; \mx_l^T\mU\mx_l + \mw^T\mx_l + w_0 \leq -1 + \varepsilon_l, \;\; l=1,\ldots,n \\
& \;\; \mU \preceq -I, \;\; \varepsilon_i \geq 0, \;\; \varepsilon_l \geq 0
\end{align*}
}}
\noindent
where $\mU$ is a negative semi-definite matrix. 
We can show that $(\mU,\mw,w_0)$ define an ellipsoidal enclosing boundary, 
wrapping $\mx_i$'s inside and leaving $\mx_l$'s outside, for which we allow some slack $\varepsilon$. $\lambda$ is to account for the imbalance between the number of positive and negative samples.
The optimization problem is convex, which we solve using an efficient off-the-shelf solver, where each hyper-rectangle output by \sub~can be processed independently.

Having set up our refinement step as a convex quadratic discrimination problem, we next describe how we incorporate $\mx_j$'s (anomalous points outside $R$) into the optimization.
Intuitively, we would like to include as many other anomalies as possible inside the ellipsoid, but only those that are nearby $\mx_i$'s and not necessarily those that are far away. 
In other words, we only want to ``recover'' the $\mx_j$'s surrounding a given $R$ and not grow the ellipsoid to include far away $\mx_j$'s to the extent that it would end up including many normal points as well.

To this end, we treat $\mx_j$'s similar to $\mx_i$'s but incur a lower penalty of excluding an $\mx_j$ than excluding an $\mx_i$ or including an $\mx_l$. The optimization is re-written as

\vspace{0.1in}
\begin{empheq}[box=\widefbox]{align}
\min_{\mU,\mw,w_0} & \;\; \sum \varepsilon_i + \alpha \sum \varepsilon_j + \lambda \sum \varepsilon_l \nonumber\\
{s.t.\;\;\;}  &  \mx_i^T\mU\mx_i + \mw^T\mx_i + w_0 \geq 1 - \varepsilon_i, \;\;\;\; i=1,\ldots,a_R \nonumber\\
 &  \mx_j^T\mU\mx_j + \mw^T\mx_j + w_0 \geq 1 - \varepsilon_j, \;\;\; j=a_{R+1},\ldots,a \nonumber\\ 
&  \mx_l^T\mU\mx_l + \mw^T\mx_l + w_0 \leq -1 + \varepsilon_l, \;\; l=1,\ldots,n\nonumber \\
&  \mU \preceq -I, \;\; \varepsilon_i \geq 0, \;\; \varepsilon_j \geq 0, \;\; \varepsilon_l \geq 0 \nonumber
\end{empheq}
\vspace{0.05in}

Here, setting $\alpha$ (penalty constant for $\mx_j$'s) smaller than both $1$ and $\lambda$ is likely a good choice. However, we do not know which $(\alpha,\lambda)$ pair would provide a good trade-off in general. Therefore, we sweep over a grid of possible values\footnote{We use $\alpha=\{10^{-6}, 10^{-5},\ldots, 1\} \times \lambda=\{10^{-3}, 10^{-2},\ldots, 10^3\}$.} and generate various ellipsoids, as illustrated for the example case in Fig. \ref{fig:refine}. 
A last but important step is to sweep over the collection to discard \textit{dominated} packs.
Specifically,
 we output only the set of $p$'s in the Pareto frontier w.r.t. mass versus purity. In this set there are \textit{no two packs where one strictly dominates the other}---by enclosing \textit{both} higher number of anomalous points (higher mass) and lower number of normal points (higher purity).

We refine a hyper-rectangle
$R = (s_1, s_2, \ldots, s_{d'})$ into an ellipsoid within the same subspace, in other words, $\mU\in \mathbb{R}^{d'\times d'}$ and $\mw\in \mathbb{R}^{d'}$. For interpretability, we constrain $\mU$ to be diagonal to obtain \textit{axis-aligned} ellipsoids as shown in Fig. \ref{fig:refine}, since the original features have meaning to the user.\footnote{If the anomalous patterns are to be used for detection, we estimate a full $\mU$ matrix (i.e., possibly rotated ellipsoid).} 

Our explanation consists of one rule on each feature in the subspace.
A feature rule is a $\pm$ radius interval around the ellipsoid's center. Formally:

\begin{definition}[Feature rules] Given an axis-aligned ellipsoid $p(\mc,\mM)$ in a subspace $f_{t_1} \times \ldots \times f_{t_{d'}}$, a rule on feature $t_z$ is an interval  $(\mc[z]-${\em radius}$_z, \mc[z]+${\em radius}$_z)$,
	where {\em radius}$_z = {\sqrt{\mM_{zz}}}$, $\;\forall z =\{1,\ldots, {d'}\}$.
	Conjunction of all $d'$ feature rules constitute the {\em signature} of $p$.
\end{definition}


To wrap up, we show how to compute $\mc$ and $\mM^{-1}$ from $(\mU,\mw,w_0)$ to obtain the center and radii for an ellipsoid, using which we generate the feature rules.

\noindent
{\em Obtaining $\mc$:} At the boundary of the ellipsoid, $h(\mx)=0$ and inside $h(\mx)>0$.
Center is the point where $h(\mx)$ is the maximum. Hence;
\begin{equation}
\label{eq:c}
\mc := \max_{\mx} \;\; \mx^T\mU\mx + \mw^T\mx + w_0 = - \frac{1}{2} \mU^{-1} \mw
\end{equation}
{\em Obtaining $\mM^{-1}$:} 
\begin{align}
\mx^T (-\mU) \mx - \mw^T\mx - w_0 & < 0 \\\nonumber
\mx^T (-\mU) \mx + 2 \mc^T \mU \mx - w_0 & < 0 \;\;\; \text{using Eq. \eqref{eq:c}}\\\nonumber
(\mx - \mc)^T (-\mU) (\mx - \mc) + \mc^T \mU \mc - w_0 & < 0\\\nonumber
(\mx - \mc)^T \frac{-\mU}{(w_0-\mc^T \mU \mc)} (\mx - \mc)    & < 1  \implies \mM^{-1} = \frac{-\mU}{(w_0-\mc^T \mU \mc)}\nonumber
\end{align}
{\em Obtaining radii:}
\begin{align}
(\mx - \mc)^T \mM^{-1} (\mx - \mc) = \sum_{z=1}^{d'} (\mx[z]-\mc[z])^2 (\mM^{-1})_{zz} \leq 1 \nonumber
\end{align}

To compute radius in dimension $z$, we find point $\mx$ where $\mx[z']=\mc[z'],\; \forall z'\neq z$, and
$(\mx[z]-\mc[z])^2 \frac{1}{\mM_{zz}}  = 1$. It is easy to see that radius$_z = \big|{\mx[z]-\mc[z]}\big| = \sqrt{\mM_{zz}}$.

\subsection{Summarization: Pack Selection for Shortest Description}
\label{ssec:summarize}

Our ultimate goal is to find anomalous patterns that explain or summarize the given anomalies in the dataset as succinctly as possible. Intuitively, ``good'' patterns enclose similar groups of points and hence help compress the data. To this end, we formulate our summarization objective by an encoding scheme and then devise an algorithm that carefully chooses a few patterns, in particular \ps~produced in \S \ref{ssec:refine}, that yield the minimum encoding length. In the following, we describe our encoding scheme, followed by the proposed subset selection algorithm.

\subsubsection{\bf MDL formulation for encoding a \underline{given} \pg}
\label{sec:mdlcoding}

Our encoding scheme involves a Sender (us) and a Receiver (remote). We assume both of them have access to dataset $\mD \in \mathbb{R}^{m\times d}$ but only the Sender knows the set of anomalous points $\mA$. The goal of the Sender is to transmit (over a channel) to the Receiver the information about which points are the anomalies  
\textit{using as few bits as possible}.
Na\"ively encoding all feature values of every anomalous point {\em individually} would cost $|\mA|d\log_2f$ bits.\footnote{Value of $f$ is chosen according to the required floating point precision in the normalized feature space $\mathbb{R}^{d}$.} 
The idea is that by encoding the enclosing boundary of \ps~(ellipsoids) found in \S \ref{ssec:refine}, we (the Sender) could  have the Receiver identify the anomalies \textit{in groups}, which could save bits.

Obviously we would want to avoid ``noisy'' \ps~that include many normal points---that would necessitate spending extra bits for encoding those exceptions (i.e. ``telling'' the Receiver which points in a pack are {\em not} anomalies). Moreover, we would want to avoid using \ps~that encode largely overlapping group of anomalies, as bits would be wasted to redundancy.
While identifying the \pg~that yields the fewest bits is the main problem, we first lay out our description length objective, for a \underline{given} \pg~$\mP = \{p_1(\mc_1,\mM_1), \ldots, p_K(\mc_K,\mM_K)\}$:
  
\bit
\setlength{\itemindent}{-.1in}
\setlength{\itemsep}{0.05in}
\item Transmit number of \ps~  $ = \log^\star K$\footnote{Cost of encoding an arbitrary integer $K$ is $L_{\mathbb{N}}(K) = \log^\star (K) + \log_2(c)$, where $c
	\approx 2.865064$ and  $\log^\star (K) = \log_2(K) + \log_2(\log_2(K)) +
	\ldots$ summing only the positive terms \cite{Rissanen78}. We drop $\log_2(c)$ as it is constant for
	all {\pg}s.
	}
\item For each pack $p_k \in {\mP}$:

\bit
\setlength{\itemindent}{.05in}
\setlength{\itemsep}{0.05in}
\item Transmit number of dimensions  $=\log^\star  d_k$, $d_k \leq d$
\item Transmit \textit{identity} of dimensions  $=\log_2  \binom{d}{d_k}$ 
\item Transmit the center $\mc_k$  $=d_k \log_2 f$ 
\item Transmit $\mM_k$ $= d_k^2\log_2 f$ ($d_k\log_2 f$ if diagonal)
\item Transmit \underline{exceptions} (i.e., non-anomalies in $p_k$):
\bit
\setlength{\itemindent}{.1in}
\setlength{\itemsep}{0.05in}
\item number of normal points in $p_k$  $=\log^\star n_k$
\item \textit{identity} of normal points; by forming all possible subsets of size $n_k$ of $m_k$ (total number of points in $p_k$) $=\log_2 \binom{m_k}{n_k}$ (based on a canonical ordering of subsets, where points are ordered by distance to center)\footnote{Another way to identify the normal points in a \textit{pack}:
	sort points by their distance to center and send the index of normal points in this list of length $m_k$. This costs more for $n_k\geq 2$: $n_k\log_2 m_k > \log_2 \frac{m_k^{n_k}}{n_k!} > \log_2 \binom{m_k}{n_k}$.}
\eit
\eit
\eit
\noindent
Total  cost of encoding with \pg~${\mP}$ is then
\begin{equation}
\label{cost}
\ell(\mP) = \log^\star K + \sum_{k = 1}^{K} L(p_k), \; \;\;\text{where}
\end{equation} 
\begin{equation}
L(p_k) = \log^\star d_k + \log_2 \binom{d}{d_k} + d_k(d_k+1)\log_2 f  + \log^\star  n_k + \log_2 \binom{m_k}{n_k}
\end{equation}

\subsubsection{\bf MDL \underline{objective function}}
\label{sec:mdlobj}
Our objective is to find a \pg, that is to identify a subset of \ps, which provides the minimum encoding length. 
However, we do not assume that all anomalies would be covered by a packing, i.e., $\bigcup_k \mA_k \subseteq \mA$, as there could be anomalous points (outliers) that do not belong in any pattern but lie away from the others.
The outliers $\mA \backslash \{\bigcup_k \mA_k\}$ are yet to be encoded individually. 

\noindent
Description length of all anomalies $\mA$ with \pg~$\mP$ is

$$
L(\mA|\mD,\mP) = \big(|{\mA}| - |\bigcup_{p\in \mP} \mA_p|\big) d\log_2 f + \big[ \log^\star |\mP| + \sum_{p\in \mP} L(p) \big]
$$

where the second term [in brackets] is $\ell(\mP)$: cost of transmitting $\mP$ (and the anomalies covered by it) by Eq. \eqref{cost}, and the first term is the cost of individually encoding the remaining anomalies not covered by $\mP$.

Notice that the objective of finding a subset $\mS$ that minimizes the description length is equivalent to selecting a \pg~that reduces the na\"ive encoding cost of $|\mA|d\log_2f$ the most, i.e.:
\begin{empheq}[box=\fbox]{equation}
\label{objective}
\max\limits_{\mS}\;  R_\ell(\mS) \;=\;  |\bigcup_{p\in \mS} \mA_p| c_u - \log^\star |\mS| - \sum_{p\in \mS} L(p)   + \big[ \log^\star |\mE| + \sum_{p'\in \mE} L(p') \big]  
\end{empheq}
where $c_u = d\log_2 f$ is a constant unit-cost to encode a point, and set $\mathcal{E}$ denotes all the ellipsoids returned from the second part (refinement), as such, $\mS\subseteq \mE$. 
First three terms of the objective capture the overall {\em reduction} in encoding cost due to the packing with ellipsoids in $\mS$. 
We can read it as aiming to {\em find a packing that covers as many anomalies as possible (expressive), while having small model cost (low complexity)---containing only a few packs in low dimensions.}
The constant term [in brackets] ensures that $R_\ell(\mS)$ is a {non-negative} function.

\subsubsection{\bf Subset \underline{selection algorithm} for MDL \pg}

To devise a subset selection algorithm, we start by studying the properties of our objective function $R_\ell$, such as submodularity and monotonicity that could enable us to use fast heuristics with approximation guarantees.
Unfortunately, $R_\ell$ is not submodular as it is given in Eq. \eqref{objective}. However, with a slight modification where we fix the solution size (number of output \ps) to $|\mS|=K$, such that the second term is constant $\log^\star K$, the function becomes submodular, as we show below.

\begin{theorem}
	Our cardinality-constrained objective set function $R'_\ell(\mS)$ is \textit{submodular}. That is, for all subsets $\mS\subseteq \mT\subseteq \mE$ and \ps~ $p\in \mE\backslash \mT$, it holds that
	$$
	R'_\ell(\mS\cup \{p\}) - R'_\ell(\mS) \geq R'_\ell(\mT\cup \{p\}) - R'_\ell(\mT) \;.
	$$
\end{theorem}
\begin{proof}
 Let $Cover(\mS) = |\bigcup_{p\in \mS} \mA_p|$ return the number of anomalies contained by the union of \ps~in $\mS$. 
Canceling the equivalent terms and constants on each side of the inequality, we are left with 
$Cover(\mS\cup \{p\}) - Cover(\mS) \geq Cover(\mT\cup \{p\}) - Cover(\mT)$.
The inequality follows from the submodularity property of the $Cover$ function. \qed
\end{proof}

It is also easy to see that $R'_\ell$ is not monotonic.

\begin{theorem}
	Our modified objective set function $R'_\ell(\mS)$ is \textit{non-monotonic}. That is, there exists $\exists \mS\subseteq \mT$ where $R'_\ell(\mT) < R'_\ell(\mS)$.
\end{theorem}
\begin{proof}
	For $\mS\subseteq \mT$, $Cover(\mT) \geq Cover(\mS)$ due to monotonicity of $Cover$ function.
	On the other hand, description cost of \ps~in $\mT$ is $\sum_{p\in \mT} L(p) = \sum_{p'\in \mS} L(p') + \sum_{p''\in \mT\setminus \mS} L(p'')$ and hence is strictly greater than those of $\mS$. 
	As such, for two \textit{packing}s $\mS\subset \mT$ with the \textit{same} coverage, we would have $R'_\ell(\mT) < R'_\ell(\mS)$.\footnote{Intuitively, this is where $R_{\ell}$ drops when we add a new \textit{pack} to $\mS$ (with  positive cost) that does not cover any new anomalies.} \qed
\end{proof}

Maximizing a submodular function is NP-hard as it captures  problems such as Max-Cut and Max k-cover \cite{conf/soda/GharanV11}.
Nevertheless the structure of submodular functions makes it possible to achieve non-trivial results.
In particular, there exist approximation algorithms for \textit{non-monotone submodular} functions that are \textit{non-negative}, like our objective function $R'_\ell$.
In particular, one can achieve an approximation factor of $0.41$ for the maximization of any non-negative non-monotone submodular function \textit{without} constraints \cite{conf/soda/GharanV11}.


In our case, we need to solve our objective under the cardinality (i.e., subset size) constraint, where $|\mS|$ is fixed to some $K$ (since only then  $R_\ell$ is submodular).
To this end, we use the {\sc Random-Greedy} algorithm by Buchbinder et al. \cite{conf/soda/BuchbinderFNS14}, which provides the best known guarantee for the cardinality-constrained setting, with approximation factors in $[0.356,\frac{1}{2}-o(1)]$.
The algorithm is quite simple; at each step of $K$ iterations, it computes the marginal gain of adding a single \p~ $p\in \mE\setminus \mS$~to $\mS$ and 
selects one  among the top $K$ highest-gain \ps~uniformly at random.

{\bf Choice of $K$:} We identify $K$, the number of \ps~to describe the anomalies, automatically, best of which is unknown apriori. Concretely, we solve to obtain subset $\mS^\star_K$ each time for a fixed $K = |\mS^\star _K| = 1, 2, \ldots, a$, and return the solution with the largest objective value of 
$R_\ell(\mS^\star _K) = R'_\ell(\mS^\star _K) - \log^\star K$ in Eq. \eqref{objective}. This is analogous to model selection with regularization for increasing model size.

\subsection{Overall Algorithm \method}
\label{ssec:final}

Algorithm \ref{alg:final} puts together all three components of \method~as described through \S \ref{ssec:rectangle}--\S \ref{ssec:summarize}. We conclude  this section with the computational complexity analysis.

\begin{algorithm}[h]
{
	\caption{\method ~($\mA\cup \mN$): Explaining Anomalous Patterns\label{alg:final}}
	\begin{algorithmic}[1]
		\REQUIRE dataset $\mD=\mA\cup \mN$ with labeled anomalies
		\ENSURE set of anomalous patterns (represented as hyper-ellipsoids) \\$\mP=\{p_1(\mc_1,\mM_1), \ldots, p_K(\mc_K,\mM_K)\}$
		\STATE Set of hyper-rectangles $\mR = \emptyset$ 
		\STATE Obtain $\mR^{(1)}$ (1-d intervals) by kernel density estimation, varying cut-off threshold in $q=\{80,85,90,95\}$
		\STATE $\widehat{f}_a :=$  distribution of number of anomalies across $\mR^{(1)}$
		\STATE $\widehat{f}_n :=$  distribution of number of normal points across $\mR^{(1)}$ 
		\STATE  $\mR := \; $\sub($\mD$, $ms=q(\widehat{f}_a,50)$, $\mu=q(\widehat{f}_n),50$)) by Alg. \ref{alg:subclus} in \S \ref{ssec:rectangle}
		\STATE Set of hyper-ellipsoids $\mE = \emptyset$ 
		\FOR{$R\in \mR$}
			\STATE $\mE_R = \emptyset$ 
			\FOR{$\alpha=\{10^{-6}, 10^{-5},\ldots, 1\}$}
			\FOR{$\lambda=\{10^{-3}, 10^{-2},\ldots, 10^3\}$}
			\STATE  $\mE_R := \mE_R \; \cup $ solve  optimization problem in \S\ref{ssec:refine} for ($R,\alpha,\lambda$) 
			\ENDFOR
			\ENDFOR
			\STATE $\mE:= \mE \; \cup$ {ParetoFrontier}($\mE_R$)
		\ENDFOR
		
		\STATE \textbf{for} $K=1,\ldots, |\mA|$: select a subset $\mS^*_K\subset \mE$ of $K$ \ps~using the cardinality-constrained {\sc Random-Greedy} algorithm by Buchbinder et
al. \cite{conf/soda/BuchbinderFNS14} to optimize the description length reduction objective $R_{\ell}(\cdot)$ in \S\ref{ssec:summarize}.\\
		\RETURN $\mP:= \arg\max_{\mS^*_K} \; R'_\ell(\mS^*_K) - \log^*K$
	\end{algorithmic}
}
\end{algorithm}


{\bf Computational complexity:}
We analyze the complexity of each part separately.
Main computation of \S \ref{ssec:rectangle} is the {\sc SubClus} algorithm.
Preliminary KDE to create 1-d intervals is independently done per dimension in parallel, only on the anomalous points. We use a constant number of sampling locations, as such, KDE complexity is $O(a)$ where $a$ is the number of anomalies.
{\sc SubClus} then proceeds level-by-level and makes as many passes over the data as the number of levels.
 For a $d'$ dimensional hyper-rectangle that meets the mass and purity criteria, all its $2^{d'}$ projections in any subset of the dimensions also meet the mass criterion (although may not be pure). 
As such, running time of {\sc SubClus} is exponential in the highest dimensionality of the hyper-rectangle that meets both criteria.
Total time complexity of this step is $O(c^{d_{\max}} + md_{\max})$ for a constant $c$\footnote{For instance, if we have $t$ $d_{\max}$-dimensional hyper-rectangles, then the complexity would be  $O(t2^{d_{\max}} + md_{\max})$, we could rewrite this as $O(c^{d_{\max}} + md_{\max})$} that accounts for possibly multiple $d_{\max}$-dimensional hyper-rectangles and the smaller ones.  The second term captures the passes over the data over $d_{\max}$ levels. 
 
The main computation of \S\ref{ssec:refine} is solving the SDP optimization problem, for which we use the popular
cvx SDPT3 solver that takes $O([d_{\max}+m]^3)$ for an axis-aligned ellipsoid (or diagonal $\mU$) per iteration.\footnote{In practice, the solver converges in 20-100 iterations.}
To speed up, we filter bulk of the points beyond a certain distance of a given hyper-rectangle, since its refined hyper-ellipsoid would mostly include/exclude points  inside and nearby it. Filtering takes $O(m)$, after which we solve the SDP for a near-constant number of points.
It is easy to show that finding the Pareto frontier set of non-dominating packs (line 14)---such that no pack that has strictly larger mass {\em and} smaller impurity exists---can be done through two passes over all alternative hyper-ellipsoids generated for different $(\alpha,\lambda)$. This procedure does not change the overall  complexity but is likely to yield a much smaller set of ellipsoids per rectangle.
We refine each hyper-rectangle independently in parallel.
 
\hide{
Finally, we discuss finding the Pareto frontier set of a set of packs, which contains non-dominating packs---there exists no pack that has strictly larger mass {\em and} smaller impurity. 
Mass and impurity are integers, so for each impurity $\mu = 0\ldots \mu_{\max}$, where $\mu_{\max}=\max_i \mu_i$, 
we build a size $\mu_{\max}+1$ array $arr$ with entries holding the maximum mass for a given purity across all packs, i.e. $arr(\mu) = \max_{i:\mu_i=\mu} ms_i$, making a single pass over the packs. We also record the identity of the pack with the largest mass for a given impurity, which dominates the others with smaller mass.
In a second step, we scan $arr$ (from left to right) and return the packs that correspond to an increasing sequence of entries. Intuitively, for increasing impurity, a pack is non-dominated only if its mass is strictly larger. This procedure does not change the overall big-O complexity but is likely to yield a much smaller set of hyper-ellipsoids per rectangle (line 18).
}

The main computation in the last part is the {\sc Random-Greedy} algorithm, which makes $K$ iterations
for a given number of packs $K$.
In each iteration, it makes a pass over the not-yet-selected hyper-ellipsoids, computes the marginal reduction in bits by selecting each, and picks randomly among the top $K$ with the highest reduction. We use a size-$K$ min-heap to maintain the top $K$ as we make a pass over the packs. Worst case cost is $O(|\mE|\log K)$, multiplied by $K$ iterations.
We run {\sc Random-Greedy} for $K=1,\ldots,a$, each of which is parallelized. Total complexity of \S\ref{ssec:summarize} is $O(|\mE|a\log a)$.

The number of ellipsoids, $|\mE|$, is in the same order of the number of hyper-rectangles from \S \ref{ssec:rectangle}, i.e., $O(c^{d_{\max}})$. Thus, the overall complexity can be written as $O(md_{\max} + c^{d_{\max}}a\log a)$; linear on the number of data points $m$, near-linear on $a$, and exponential in the largest pack dimensionality $d_{\max}$.

\section{Experiments}
\label{sec:eval}

Through experiments on real-world datasets we answer the following questions.
A quick reference to the UCI datasets used in our experiments  is in Table \ref{tab:data}. Last column gives \% savings (in bits) in describing/encoding the anomalies by \method.
\bit 
\setlength{\itemsep}{-1\itemsep}
\setlength{\itemindent}{.15in}
\item[Q1.] {\bf Effectiveness:} How accurate, interpretable, and succinct are our explanations?
How do they compare to descriptions by Decision Trees?

\item[Q2.] {\bf Detection performance:} Do our explanations generalize? Can they be used as signatures to detect future anomalies?
To this end, we compare \method~to 7 different baselines.

\item[Q3.] {\bf Scalability:} How does \method's running time scale in terms
of data size and dimensionality?
\eit

\subsection{Effectiveness of Explanations}
Our primary focus is anomaly \textit{description} where we unearth interpretable characteristics for known anomalies. 
To this end, we present 6 case studies with ground truth, followed by quantitative comparison to decision trees.

\subsubsection{Case Studies\\}

Our {\bf {Image} dataset} contains gray-scale headshot images of various people. We designate the majority wearing dark-color t-shirts as the normal samples.
We create 3 versions containing different number of anomalous patterns, as we describe below.
We compare \method's findings to the ground truth.

\vspace{0.075in}
{\textbf{Case I:} \imageo~~} We label 8 images of people wearing sunglasses as anomalies as shown in (a) below, and combine them with the normal samples none of which has sunglasses.
In this simple scenario \method~successfully identifies a single, 1-d pattern shown in (b), which packs all the 8 anomalies but no normal samples. Also shown at the bottom of (b) is the interval of values, that is the 
$\pm$ radius range around the \p's center, for the corresponding dimension (the lower, the darker the pixel).

\vspace{0.075in}
\begin{tabular}{cc} 
		\includegraphics[width=0.5\textwidth]{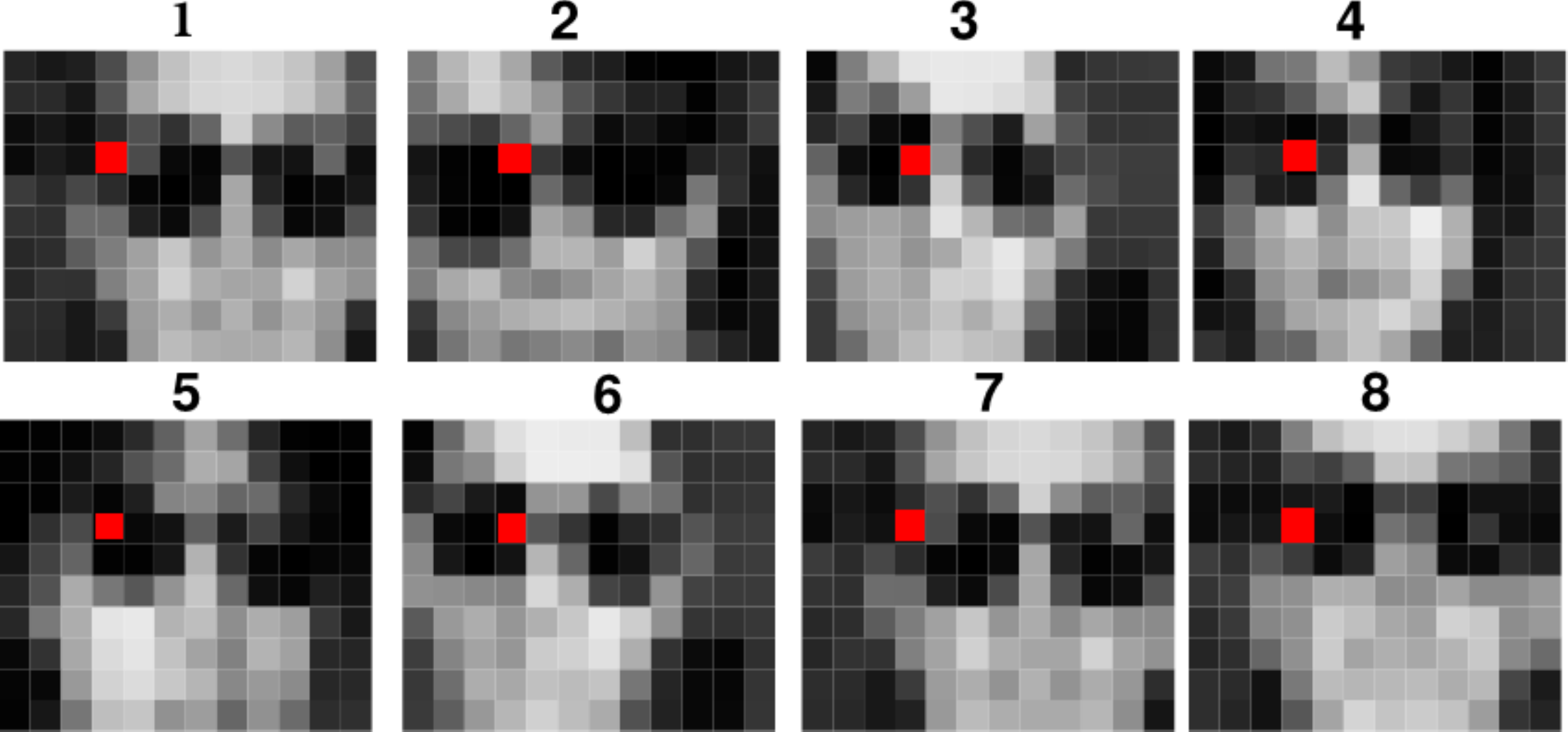}&
    	\includegraphics[width=0.25\textwidth]{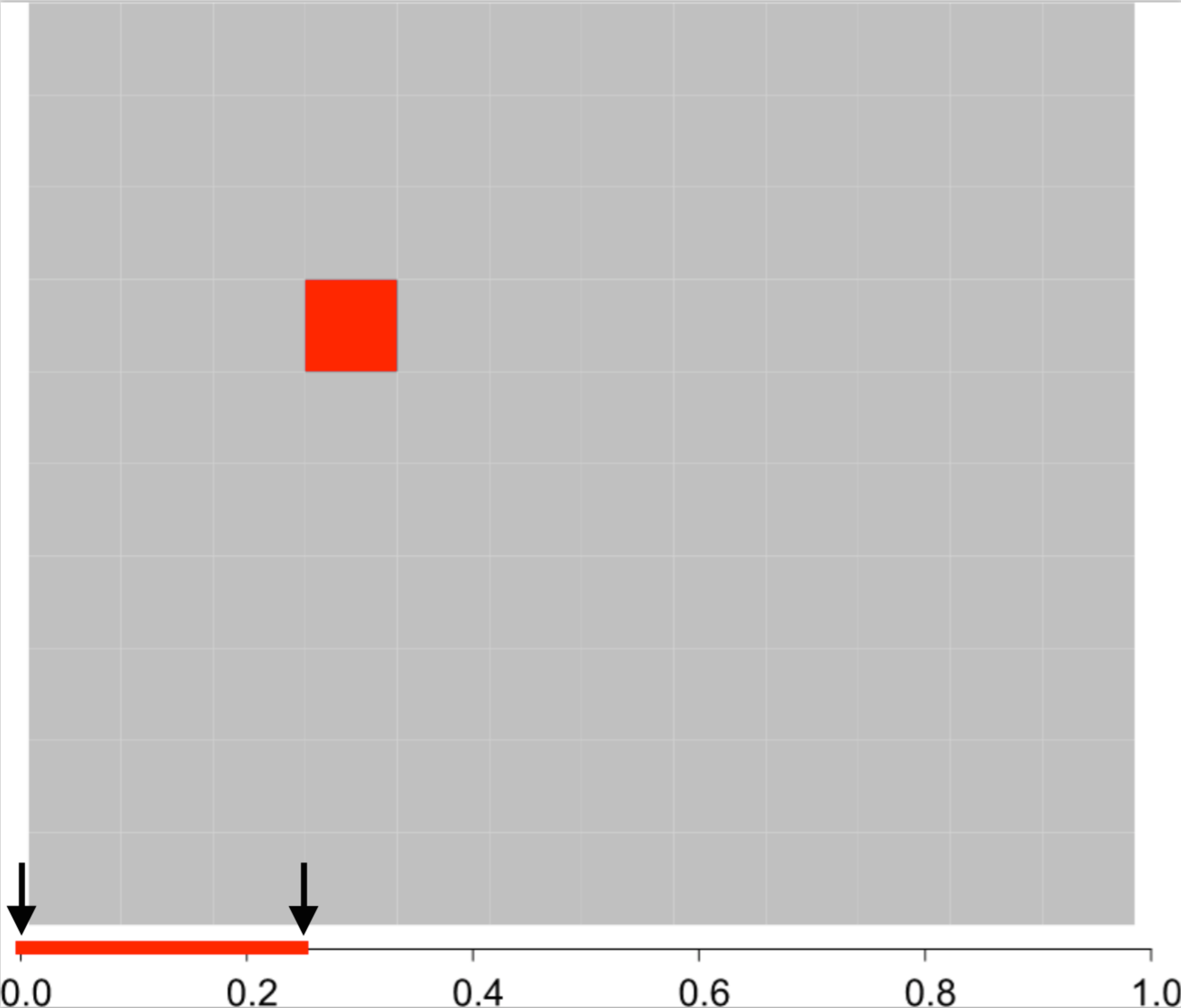} \\
		{(a) anomalies} & {(b) \method~\pg }
\end{tabular}
\vspace{0.1in}

\begin{table}[!t]
	\caption{Dataset statistics. \method~achieves significant savings (in bits) by explaining anomalies in groups.} 
	\begin{center}
		{{
				\vspace{-0.1in}
				\begin{tabular}{l|rrr||r} 
					\hline
					{\bf Name} & size $m$ & dim. $d$  & anom. $a$  & \%-savings  \\\hline
					\imageo 	&88 &120 &8 & 99.75 \\\hline 
					\imaget	&91 &180 &9 &  88.53\\\hline
					\imageth &110 &180 &12 &  99.51\\\hline
					\digito	&1371 &16 &228 & 99.83 \\\hline
					\digitt	&1266 &16 &211 &  99.72 \\\hline
					\cancer	&683 &9 &239 & 93.74 \\\hline
					\arr	&332 &172 &87 & 92.92 \\\hline
					\wine	&95 &13 &24 & 97.04 \\\hline
					\yeast	&592 &8 &129 & 98.04 \\\hline
				\end{tabular}
		}}
		\label{tab:data}
	\end{center}
	\vspace{-0.1in}
\end{table}

{\textbf{Case II:} \imaget~} Next, we construct the 9 anomalies as shown earlier in \S\ref{sec:overview} in Fig. \ref{fig:eg}: 6 wearing sunglasses, 4 white t-shirt (2 wearing both), plus 1 person with a beard (normal samples has no beard).
As detailed in the caption of the figure, \method~finds 2 pure \ps, each 1-d, that collectively describe the 8 anomalies and none of the normal samples. The bearded image does not belong to any \p~and is left out as an outlier.

{\textbf{Case III:} \imageth~} We construct the third dataset with 12 anomalies: the same 9 from \imaget~plus 3 faces (10--12) with beard as shown below.
In this case, \method~finds that characterizing the bearded images as a separate pattern is best to reduce the description cost, and outputs 3 pure, 1-d \ps~shown in (b).

\vspace{0.1in}
\begin{tabular}{cp{0.1in}c} 
	\includegraphics[width=0.3\textwidth,height=1.01in]{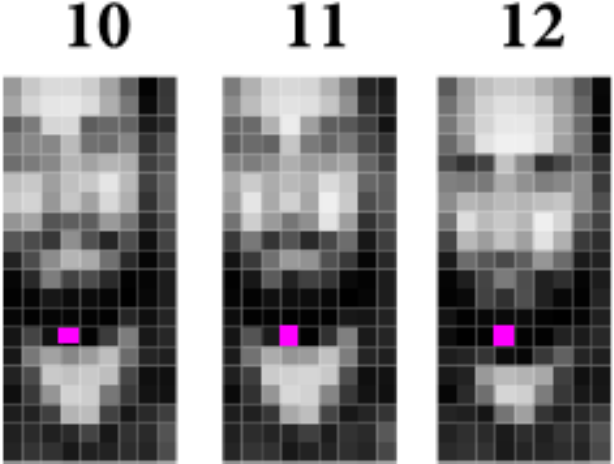}&&
	\includegraphics[width=0.43\textwidth,height=1.01in]{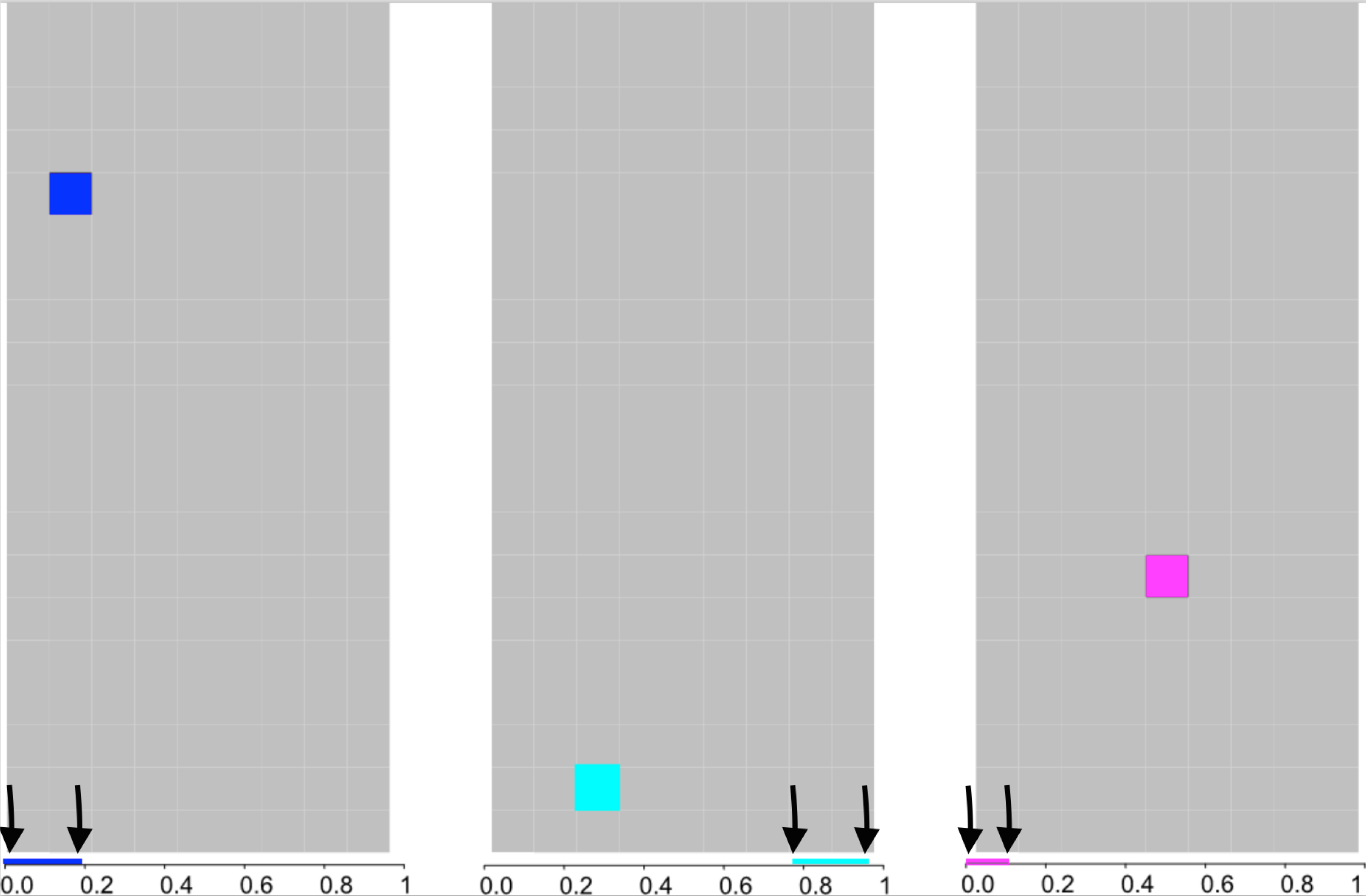} \\
	{(a) +3 anomalies} && {(b) \method~\pg }
\end{tabular}
\vspace{0.15in}

In all scenarios, \method~is able to unearth simple (low-dimensional) and pure (discriminative) characteristics of the anomalies. Also, it automatically identifies the correct number of anomalous patterns that yield the shortest data description as shown in Fig. \ref{fig:mdl}.


\begin{figure}[h]
	\centering
	\includegraphics[width=0.6\textwidth]{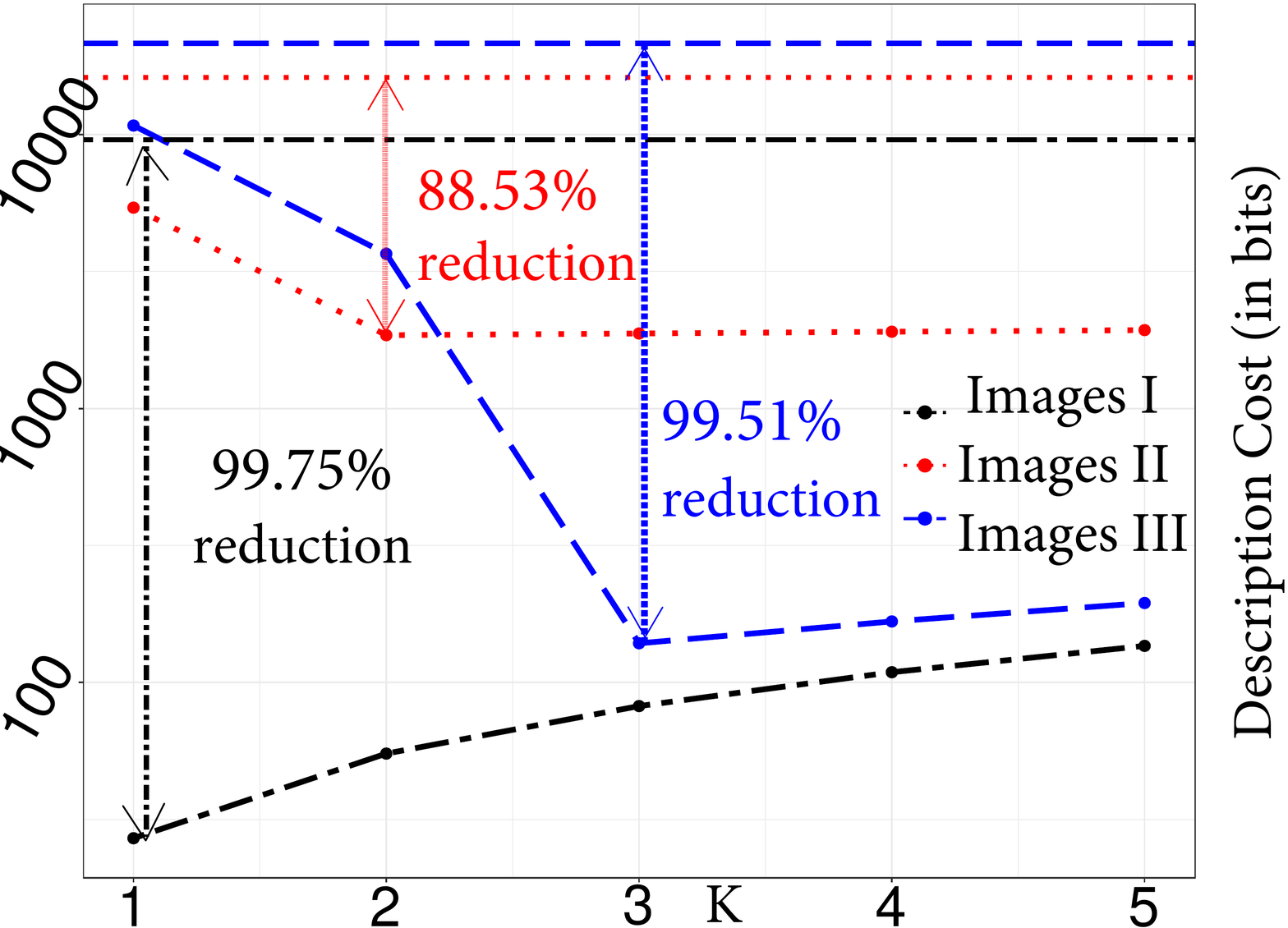}
	\caption{\method's description cost of anomalies in {image} datasets for $K=1,\ldots,5$. Na\"ive/base cost ($K=0$) is shown w/ a horizontal line per dataset. \method~finds the appropriate number of patterns automatically and  significantly reduces the description cost.}
	\label{fig:mdl} 
\end{figure}

Next we study a different domain.
The {\bf Digit dataset} contains instances of digit hand-drawings in time. Features are the $x$ and $y$ coordinates of the hand in 8 consecutive time ticks during which a human draws each digit on paper. As such, each drawing has 16 features. 

{\textbf{Case IV:} \digito~} We designate all drawings of digit `0' as normal and a sample of digit `7' as `anomalous' to study the characteristics of drawing a `7' as compared to a `0'.
8 different positions of the hand in time averaged over all corresponding samples of these two digits is shown below (a--b).

\vspace{0.075in}
\begin{tabular}{ccc} 
		\includegraphics[width=0.25\textwidth, height = 1.1in]{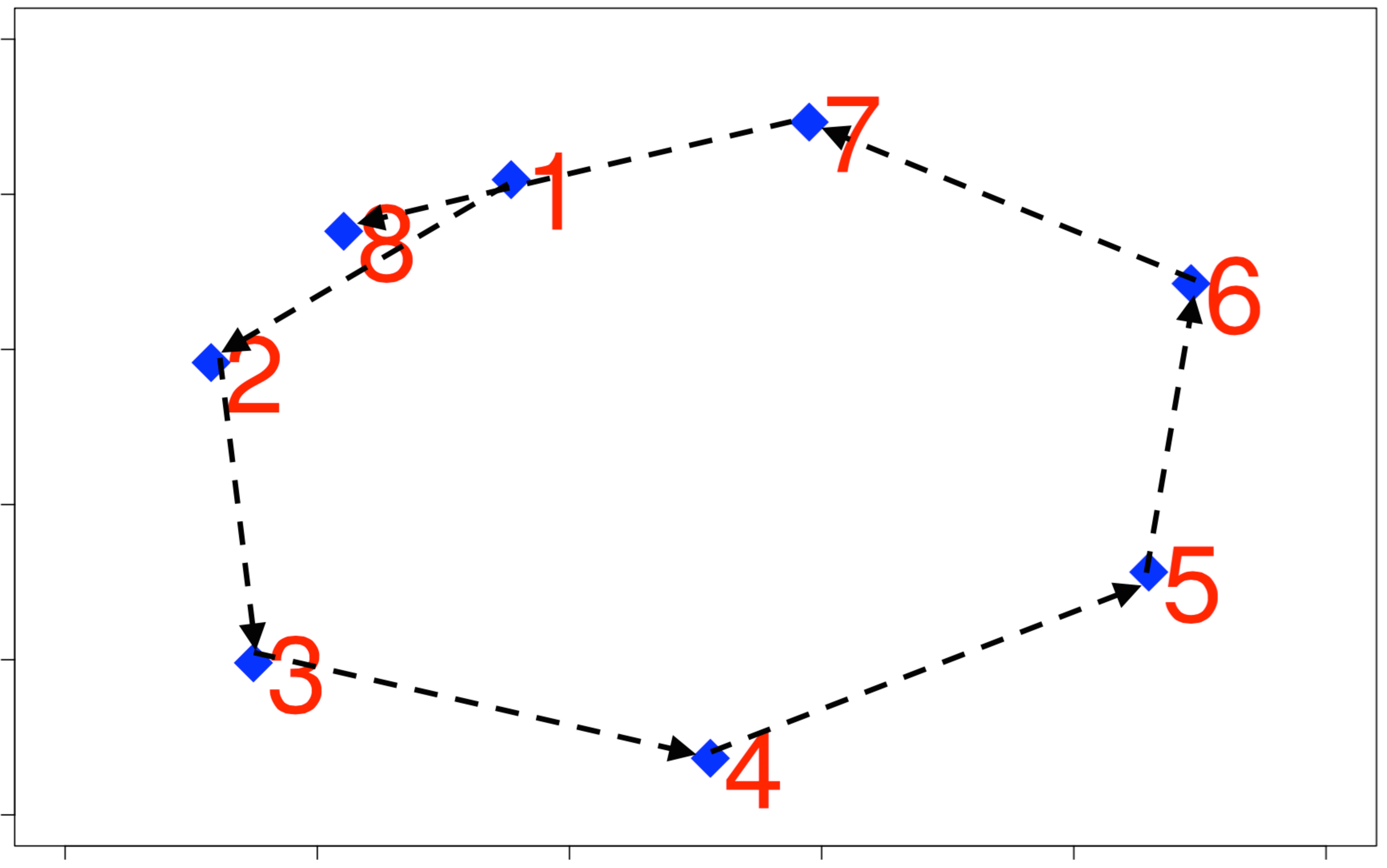}&
	\includegraphics[width=0.25\textwidth, height = 1.1in]{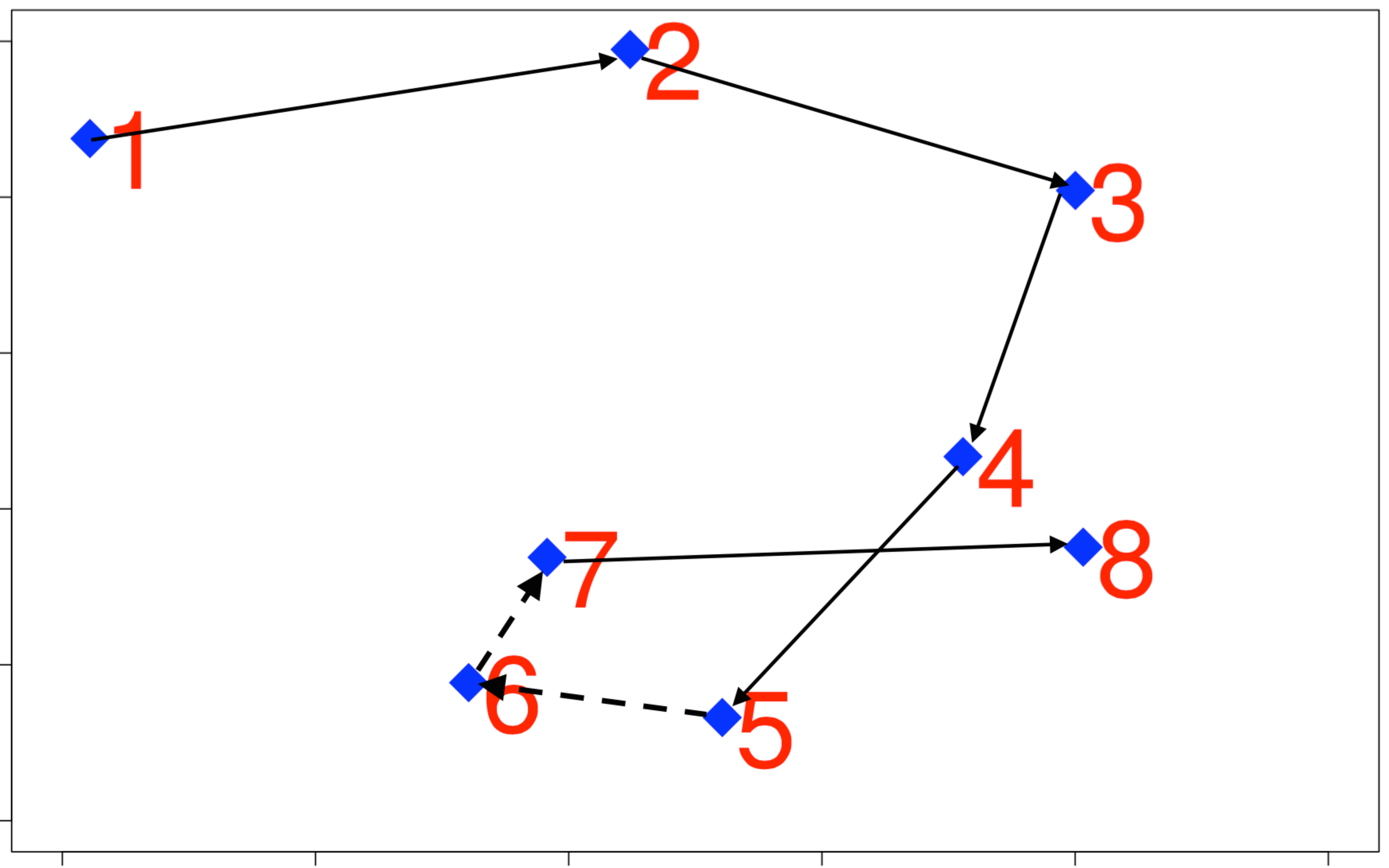} &
		\includegraphics[width=0.25\textwidth, height = 1.1in]{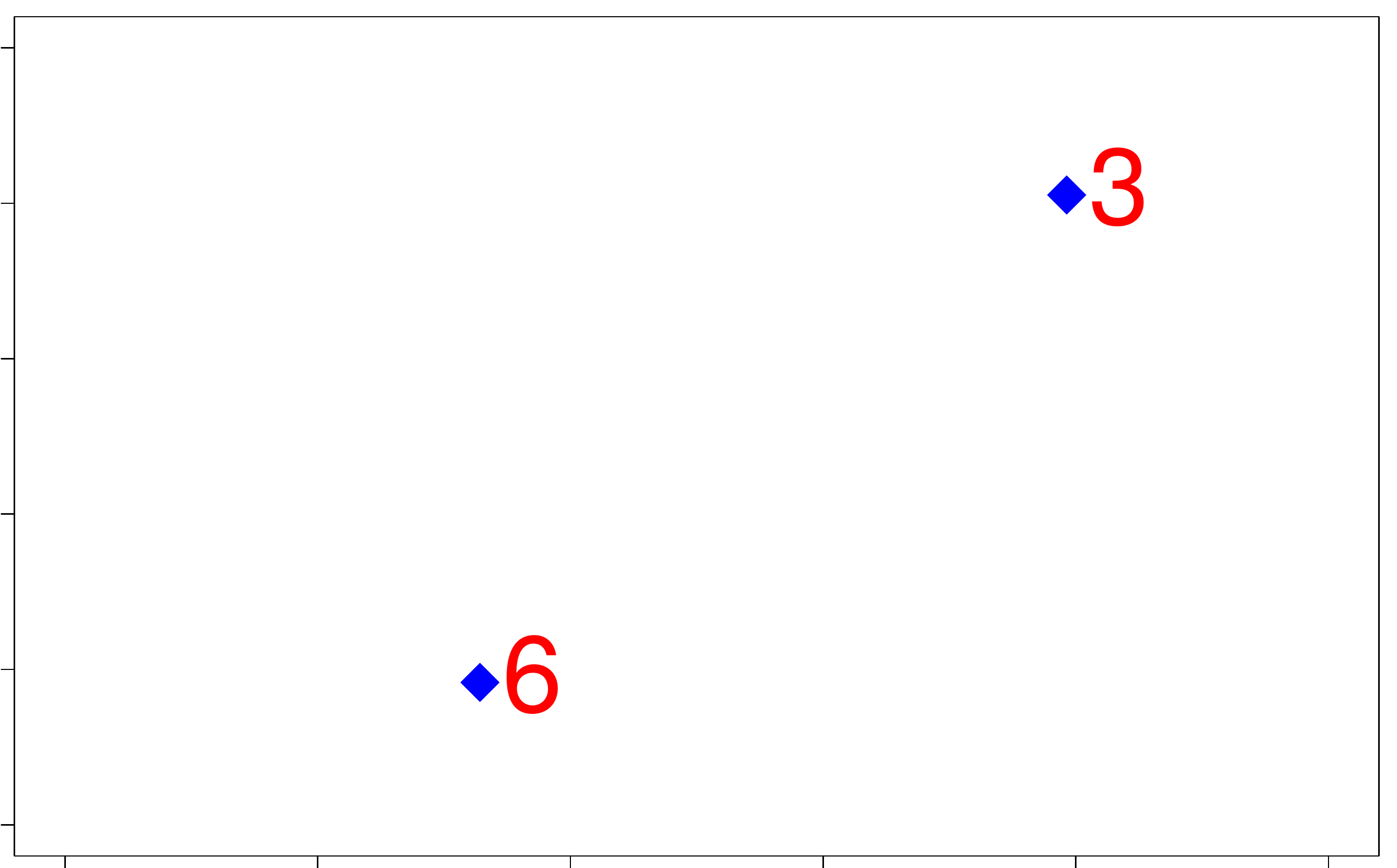}\\
		{(a) avg. `0'-drawing} &  {(b) avg. `7'-drawing} &{(c) \method~\pg }
\end{tabular}
\vspace{0.15in}

\method~identifies a single, 2-d \p~containing all 228 instances of `7's and no `0's, as 
given in Table \ref{tab:digito}, where we list the ellipsoid center and the $\pm$radius interval where the hand is positioned for the characterizing features. 
The anomalous pattern suggests right \& bottom positioning of the hand respectively at times t3 \& t6, which follows human intuition---in contrast, typical hand positions for `0' at those ticks are opposite; at the left \& top. Corresponding avg. hand positions in 2-d is shown in (c) above.

\begin{table}[h]
	\vspace{-0.15in}
	\caption{\digito~`0' vs. `7': \method~finds one 2-d \p.}
	 \vspace{-0.075in}
	\centering
	{{
			\begin{tabular}{cccc|cc}
				\hline
				packID & feature & center & interval & $|\mA_k|$ & $|\mN_k|$ \\
				\hline
				$k=$ 1 & $x@t3$ & 0.82    & (0.66, 0.98)	& 228 	& 	0 \\
			& $y@t6$ & 0.17	& (0.02, 0.31)	& 	&		\\\hline
			\end{tabular}
	}}
	\label{tab:digito}
\end{table}

{\textbf{Case V:} \digitt~~} We perform a second case study where we designate digit `8' drawings as normal and `2' and `3' as the anomalies. Avg. drawings are illustrated in (a--b) below. \method~is able to describe 210 of the 211 anomalies in a single, 4-d pack listed in Table \ref{tab:digitt} and illustrated in (c).
The single unpacked drawing is shown in (d) and looks like an odd `3'.

 \vspace{0.1in}
\begin{tabular}{cccc} 
\hspace{-0.3in}
	\includegraphics[width=0.22\textwidth, height = 1in]{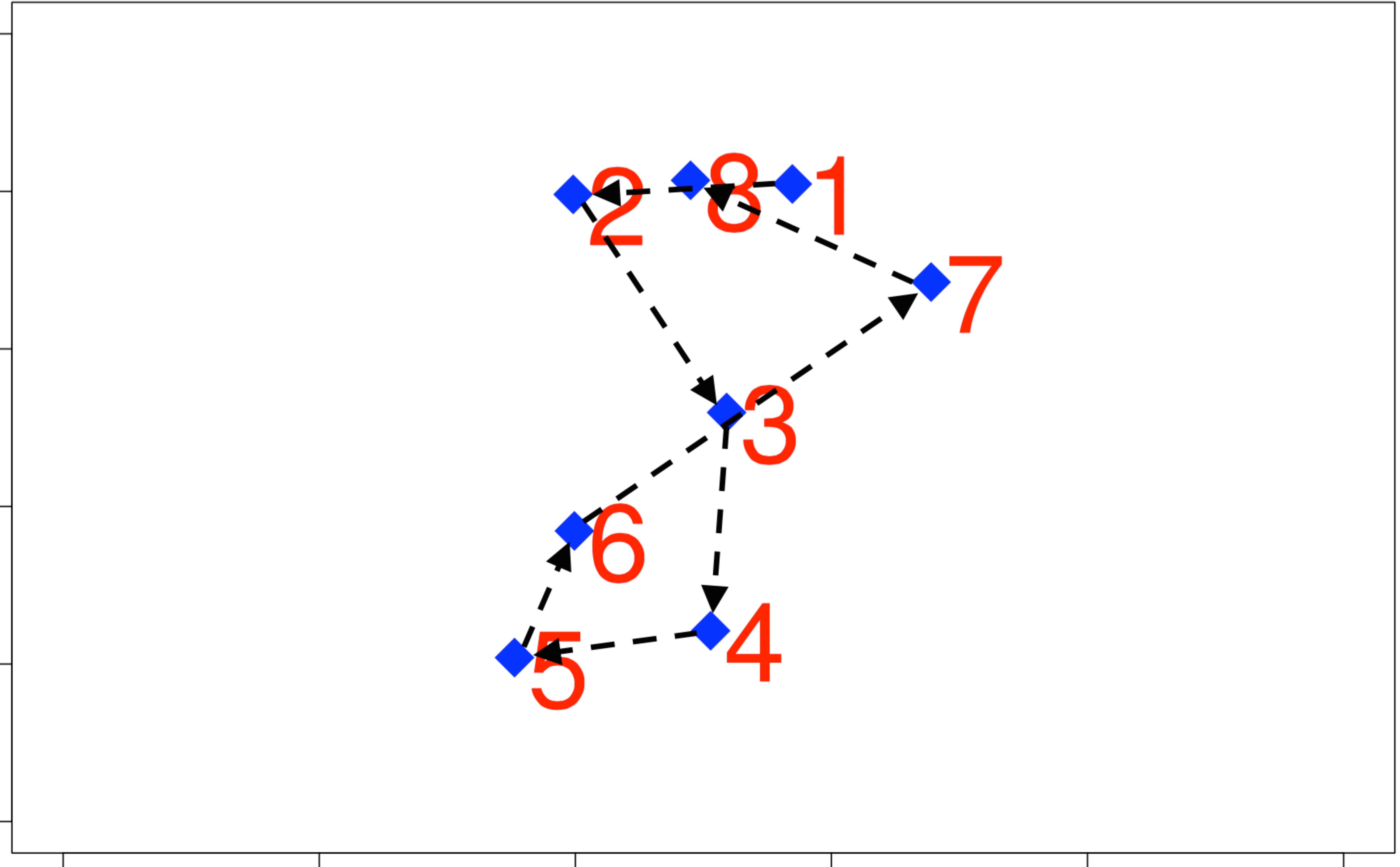}&
	\hspace{-0.1in}
	\includegraphics[width=0.22\textwidth, height = 1in]{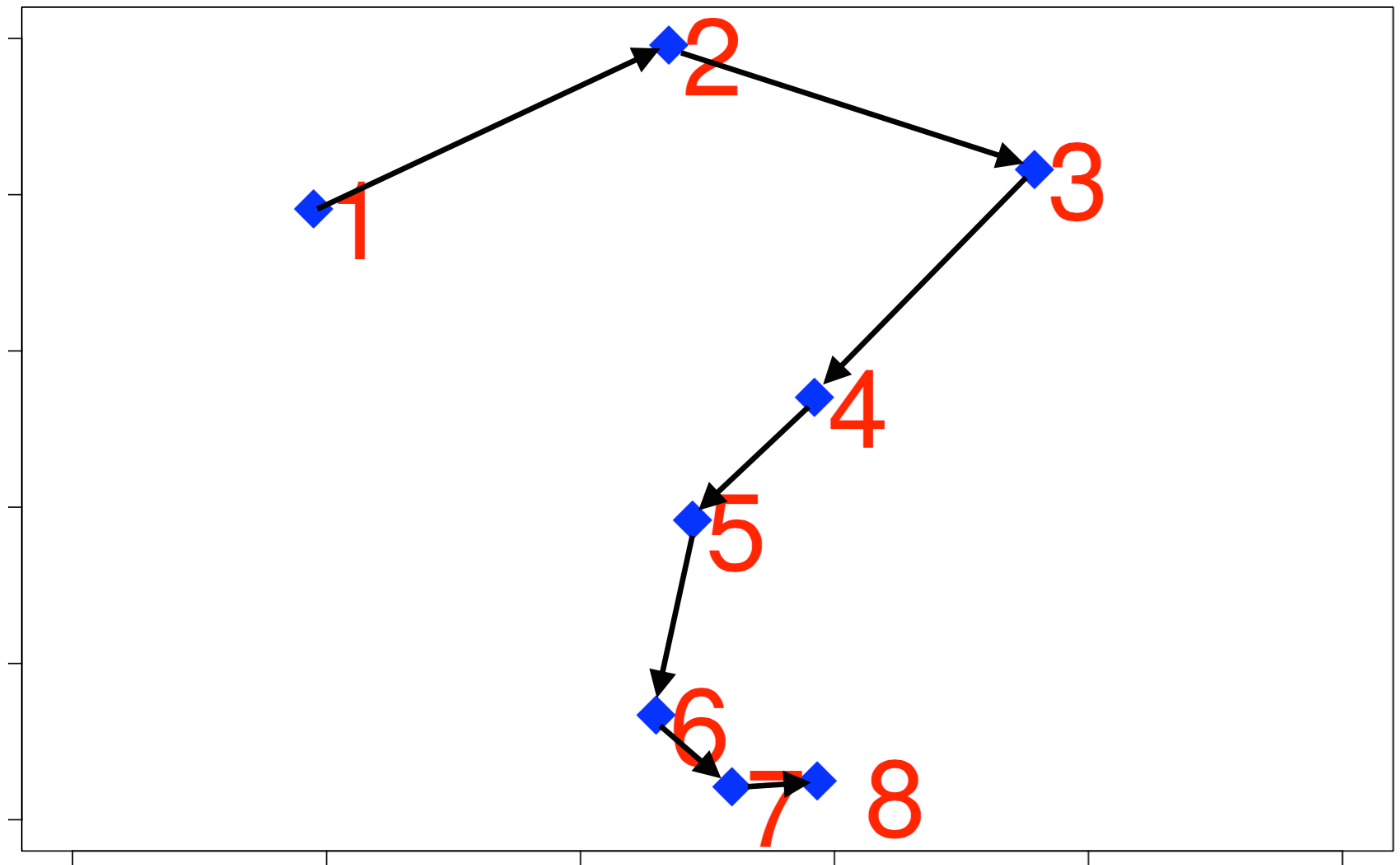} &
	\hspace{-0.1in}
	\includegraphics[width=0.22\textwidth, height = 1in]{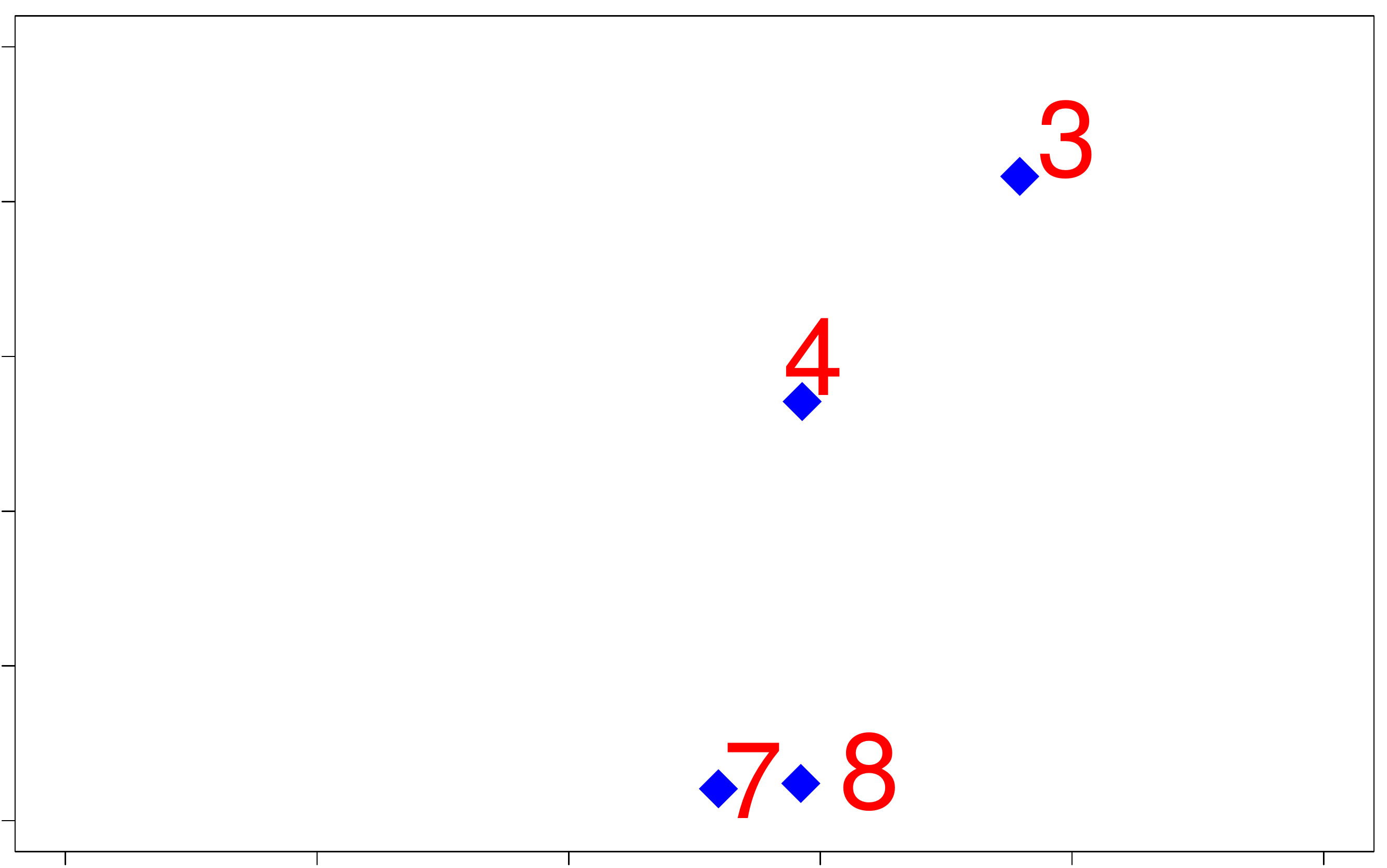} &
	\hspace{-0.1in}
	\includegraphics[width=0.22\textwidth, height = 1in]{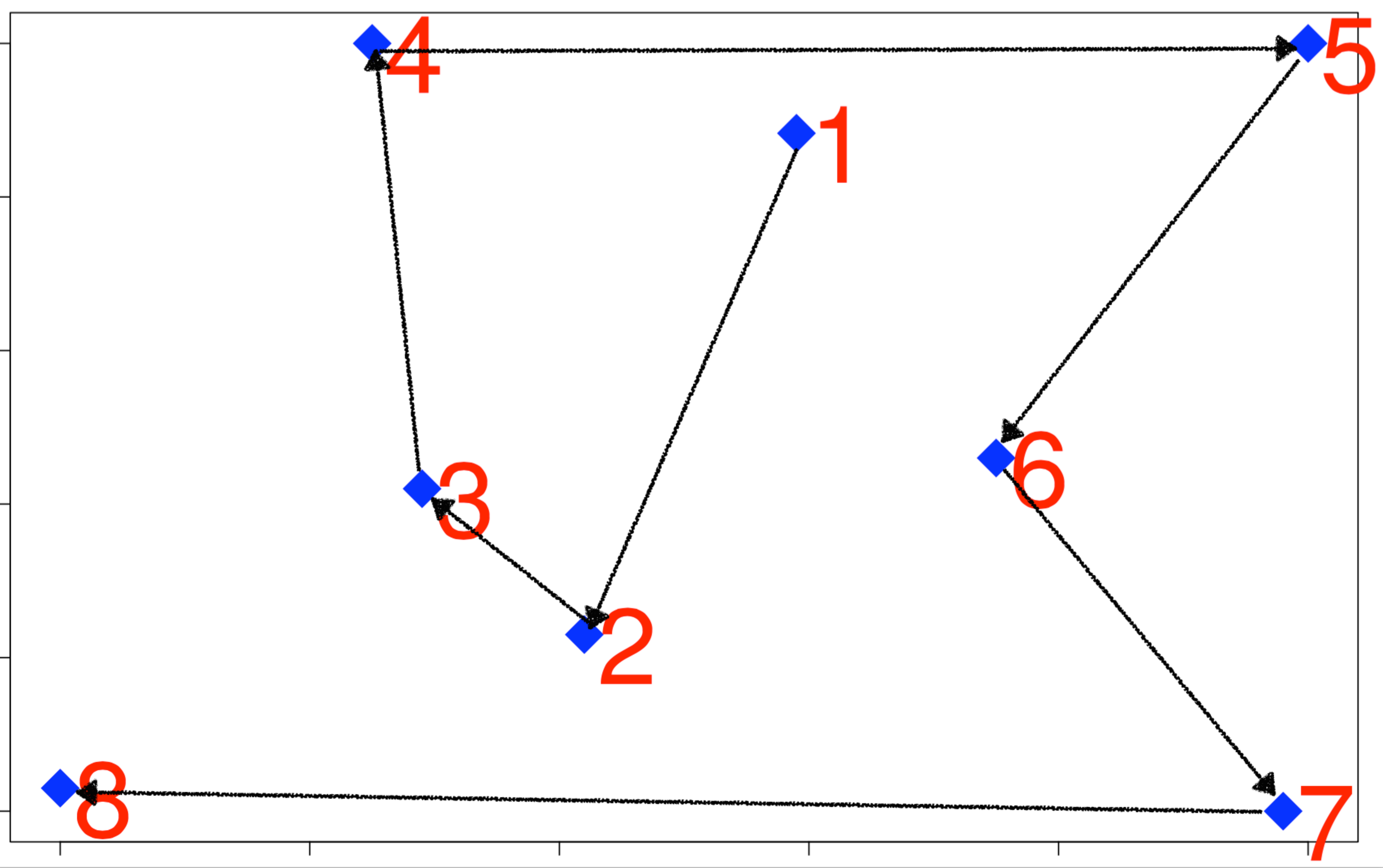}\\
		{(a) avg. `8'} &  {(b) avg. `2,3'} & {(c) \method~\pg } & {(d) outlier }
\end{tabular}
 \vspace{0.05in}

\begin{table}[h]
	\caption{\digitt~`8' vs. `2',`3': \method's one 4-d \p.}
	 \vspace{-0.075in}
	\centering
	{{
			\hspace{-0.15in}
			\begin{tabular}{cccccc}
				\hline
				packID & feature & center & interval & $|\mA_k|$ & $|\mN_k|$ \\
				\hline
			$k=$ 1 & $y@t3$ & 0.83  	&  (0.71, 0.95)	& 210 & 0 \\
		& $y@t4$ & 0.54	& (0.38, 0.69)	& 	&	\\
		& $y@t7$ & 0.04	& (0.00, 0.11)	& 	& 	\\
		& $y@t8$ & 0.05	& (0.00, 0.12)	& 	& 	\\\hline
			\end{tabular}
	}}
	\label{tab:digitt}
\end{table}

\begin{wrapfigure}{r}{0.25\textwidth}
	\vspace{-0.3in}
	\centering
	\includegraphics[width=0.23\textwidth, height = 1.1in]{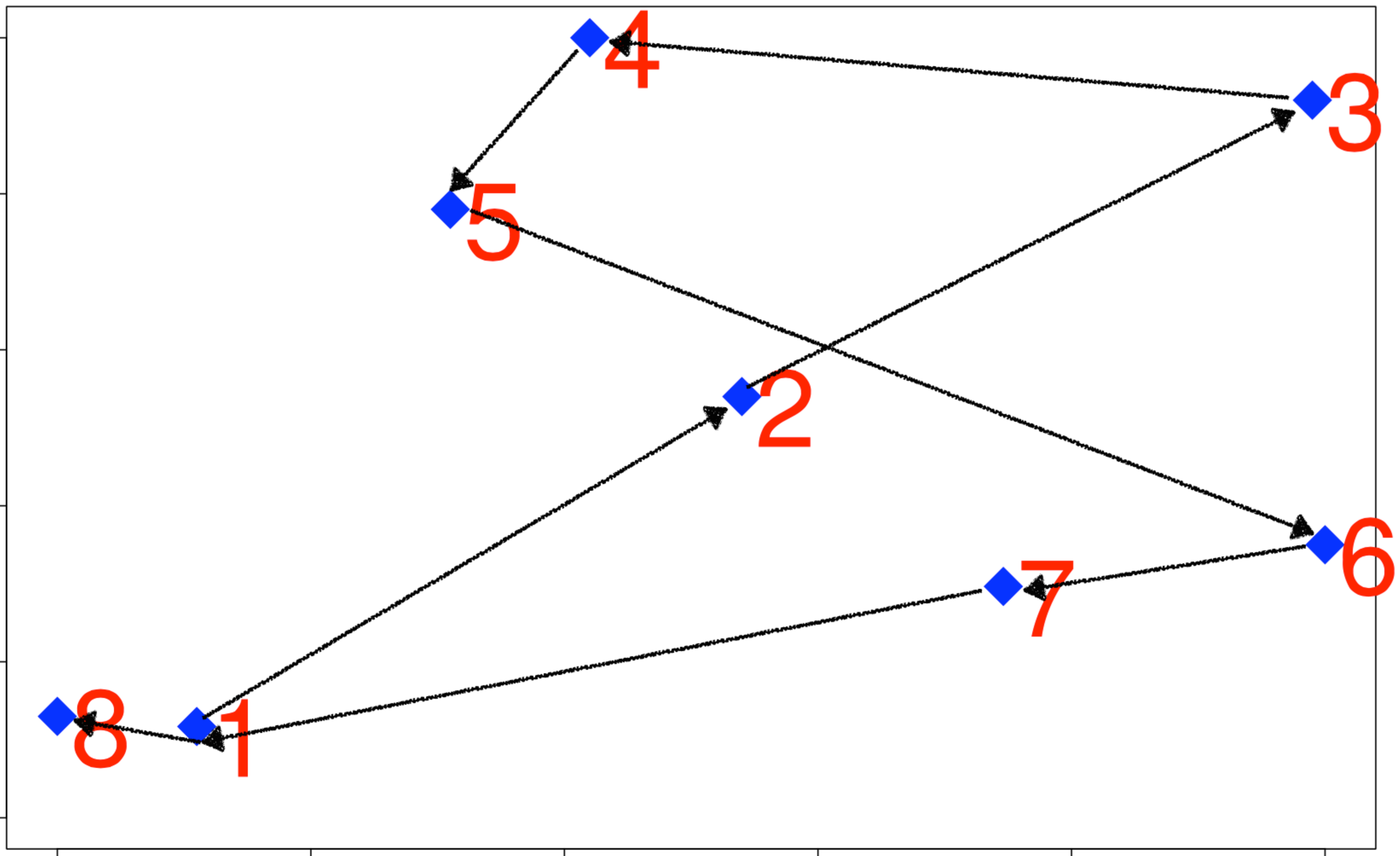}	
	\label{fig:oddone}
\end{wrapfigure}
Looking at the avg. `8' vs. `2' or `3' drawings above, it appears that a single feature like $y@t8$, i.e., vertical hand position at the end, should be discriminative alone; as `8' tends to end at the \textit{top} vs. others at the bottom.

	Interestingly, none of the 1-d \ps~on $y@t8$ is pure like the 4-d one output. A non-anomalous sample it contains is shown on the right, which is an `8' that starts and {ends} at the \textit{bottom} just like most `2' and `3's.


{\textbf{Case VI:} \cancer~}  Finally, \textbf{breast cancer dataset} contains 239 malign (anomalous) and 444 benign cancer instances. \method~finds 5 \ps~listed in Table \ref{tab:cancer}, covering a total of 226 anomalies while also including 17 unique normal points in the \pg. {\em Pack} 1  characterizes 162 cases with high `chromatin'. Second 2-d \p~suggests large `clumpthickness' and `mitoses' (related to cell division and tissue growth) for 145 cases. 
Smaller pure 1-d \ps, 4 and 5, indicate very large `cellsize' and `nucleoili'. These findings are intuitive even to non-experts like us (although we lack the domain expertise to interpret \p~3).

\begin{table}[h]
	\vspace{-0.1in}
	\caption{\cancer: \method~finds five 1-d or 2-d \ps.}
	 \vspace{-0.075in}
	\centering
			\begin{tabular}{p{0.95cm}p{2cm}p{0.75cm}p{1.6cm}p{0.4cm}p{0.4cm}}
				\hline
				packID & feature & center & interval & $|\mA_k|$ & $|\mN_k|$ \\
				\midrule
			$k=$ 1 & chromatin	& 0.76 &(0.63, 0.88) & 162 & 11 \\ \hline
	$k=$ 2 & clumpthickness & 0.94	& (0.84, 1.00)	& 145	&5	\\
	& mitoses & 0.28	& (0.00, 0.63)	& 	& 	\\\hline
	$k=$ 3 & epicellsize& 0.33 & (0.24,0.42) & 97 & 2 \\
	& barenuclei& 0.11 & (0.09,0.14) &  &  \\\hline
	$k=$ 4 & nucleoili & 0.98  & (0.93, 1.00)	& 75 & 0 \\\hline
	$k=$ 5 & cellsize & 0.99	& (0.98, 1.00)	& 67	& 0	\\
				\hline
			\end{tabular}
	 \vspace{-0.1in}
	\label{tab:cancer}
\end{table}


\begin{table}[!htp]
	\centering
	\caption{Interpretability measures (a)--(d): \method~vs. Rule learners. Also given for reference is detection performance in AUPRC (See \S\ref{ssec:detection} for details). \label{tab:interpret}}
	\begin{tabular}{c||cccc|c}
		\toprule
		{\textbf{measure}}	&{\textbf{(a) \# of}}	   
		& {\textbf{(b) avg.}} & 
		{\textbf{(c) avg.}} & 
		{\textbf{(d) avg.}} & 
		{\textbf{detection} } \\
		{\textbf{/method}}	&	{ \textbf{groups}}	   
		& {\textbf{length}} & 
		{\textbf{impurity}} & 
		{\textbf{width}} & 
		{\textbf{performance}} \\ \toprule	
		DT-5 & 4.0000 & 2.9889 & 0.0233 &  0.4769 & 0.6252 \\ \hline
		DT-4 & 3.7778 & 2.7856 & 0.0422 &  0.4801 & 0.6070 \\ \hline
		DT-3 & 3.0000 & 2.4078 & 0.0700 &  0.4812 & 0.6210 \\ \hline
		DT-2 & 2.4444 & 1.8889 & 0.1378 &  0.4872 & 0.6236 \\ \hline
		DT-1 & \textbf{1.7778} & \textbf{1.0000} & 0.4056 &  0.5017 & 0.5656 \\ \hline
		RuleFit\footnotemark & 12.0000 & 1.7800 & 0.0229 &  0.4643 & 0.8471 \\ \hline
		Ripper & 2.4444 & 1.5244 & 0.0178 &  0.3889 & 0.7244 \\ \midrule
		\method & 2.5556 & 2.1000 & \textbf{0.0152} & \textbf{0.2333} & \textbf{0.8781} \\ \bottomrule
	\end{tabular}
\end{table}
\footnotetext{Note that RuleFit is averaged over seven datasets due to underspecified regression in \arr~and~\yeast}
\begin{figure}[!th]
	\centering
	\begin{tabular}{cc}
		\includegraphics[width=.4\textwidth]{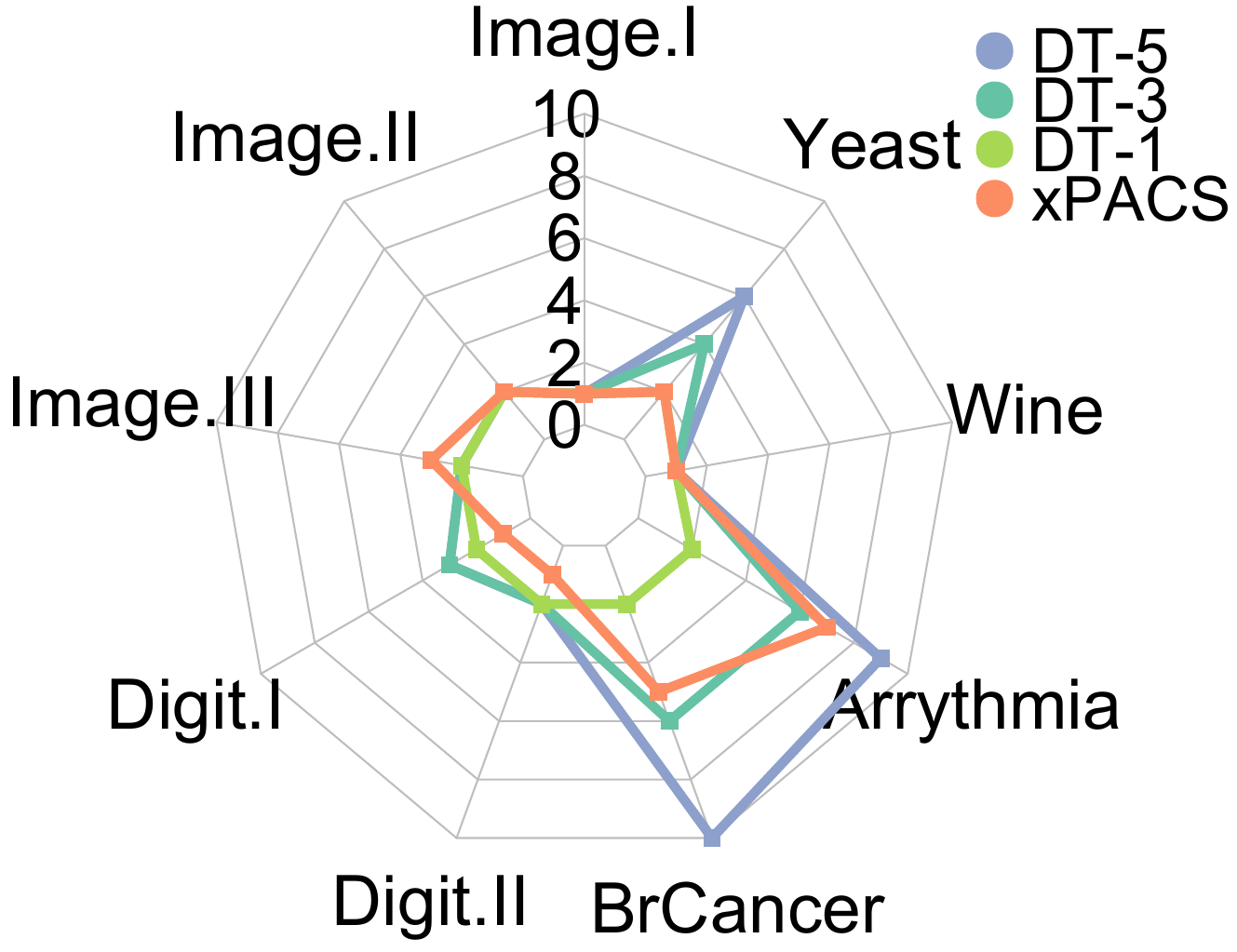}
		&
		\includegraphics[width=.4\textwidth]{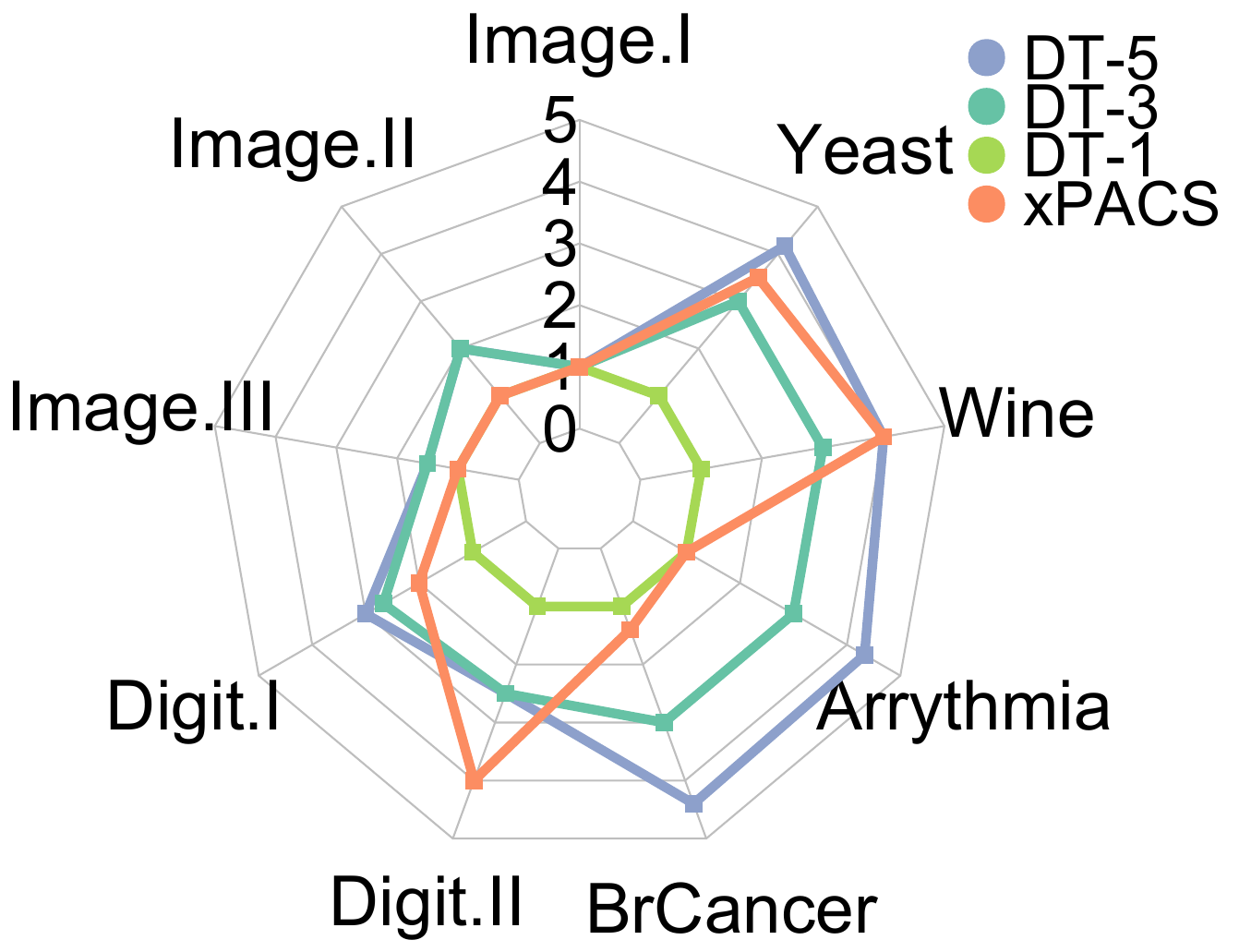}\\ \\
		(a) \# groups &
		(b) avg. length \\ \\
		\includegraphics[width=.4\textwidth]{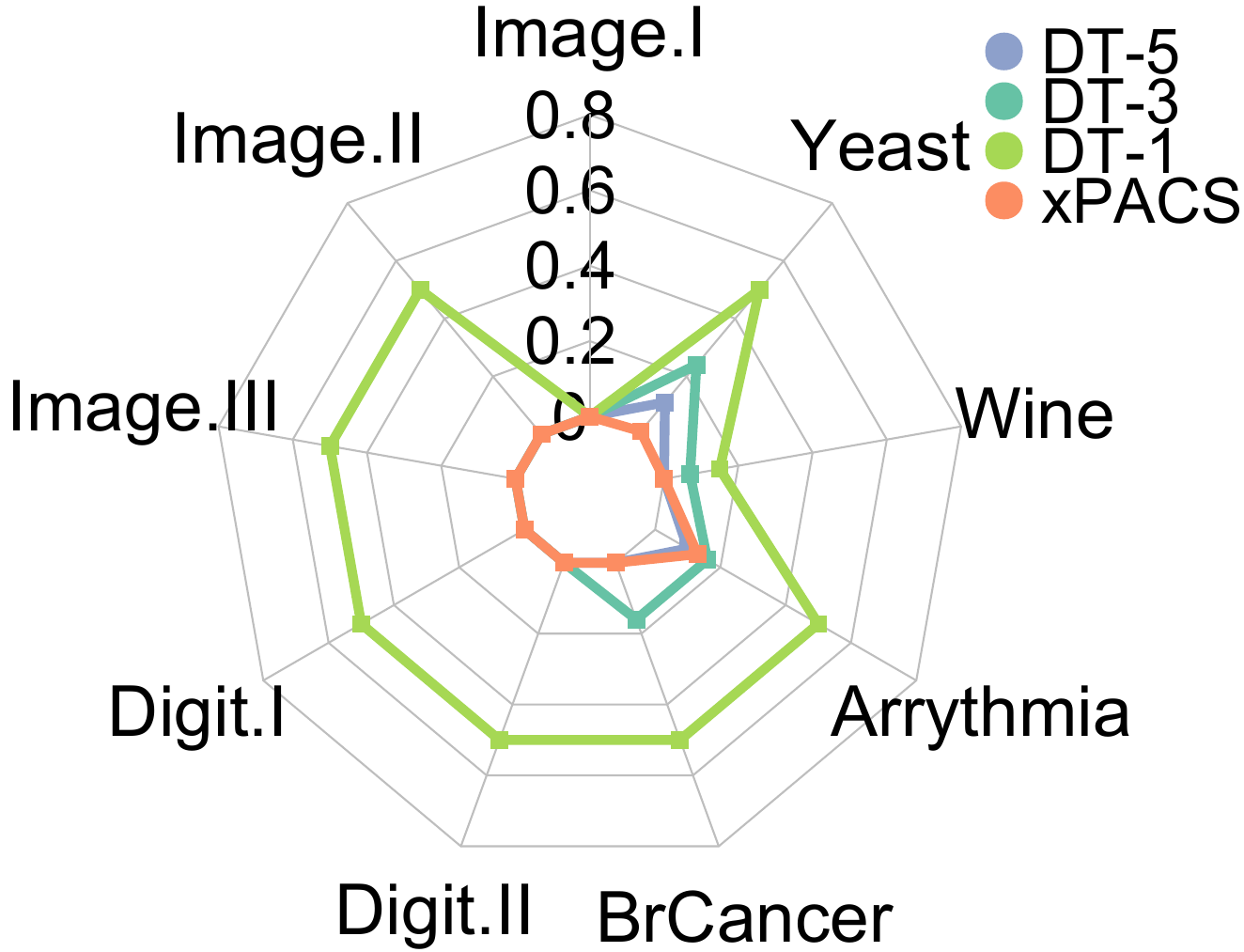}
		&
		\includegraphics[width=.4\textwidth]{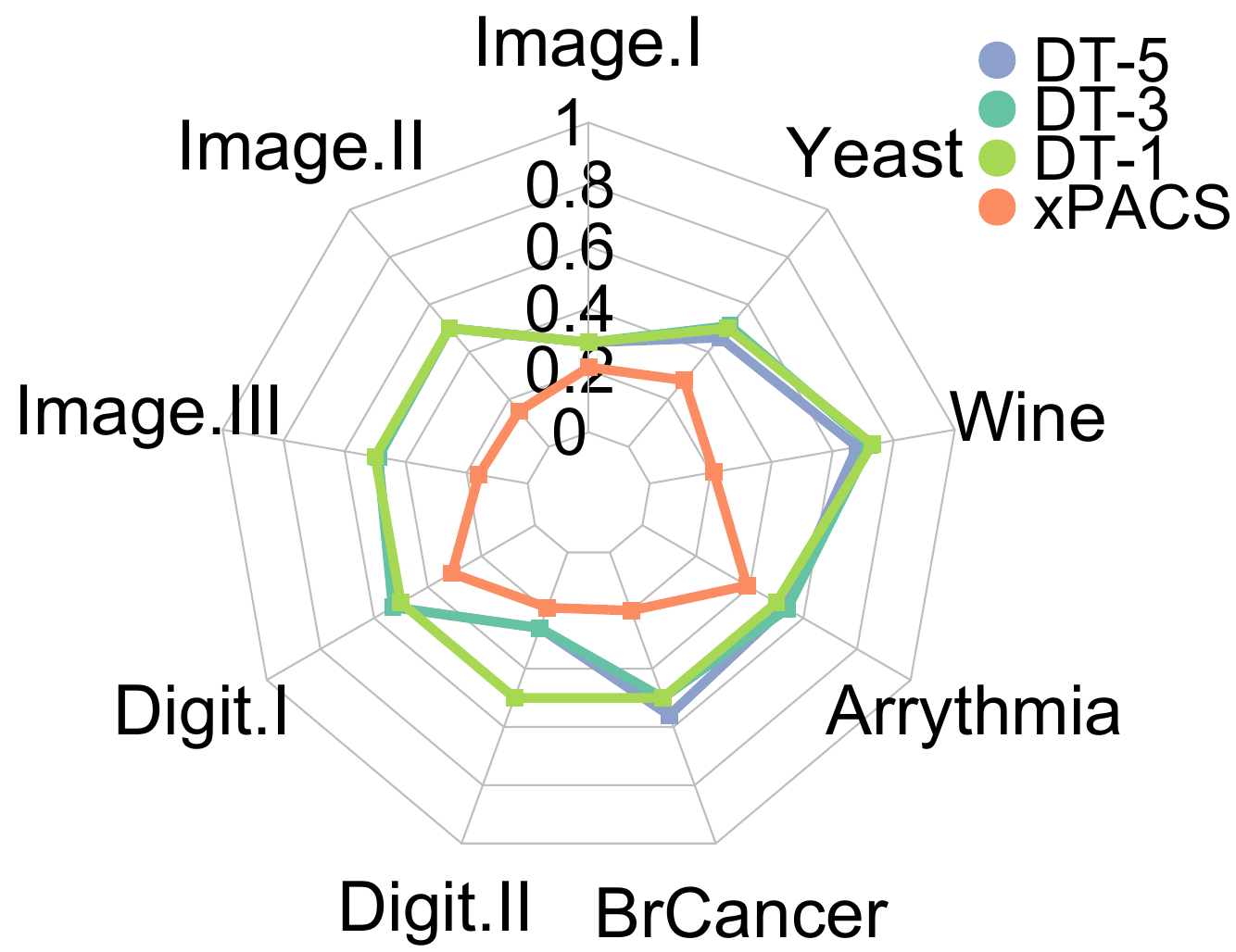}\\\\
		(c) avg. impurity &
		(d) avg. width \\\\
		\includegraphics[width=.4\textwidth]{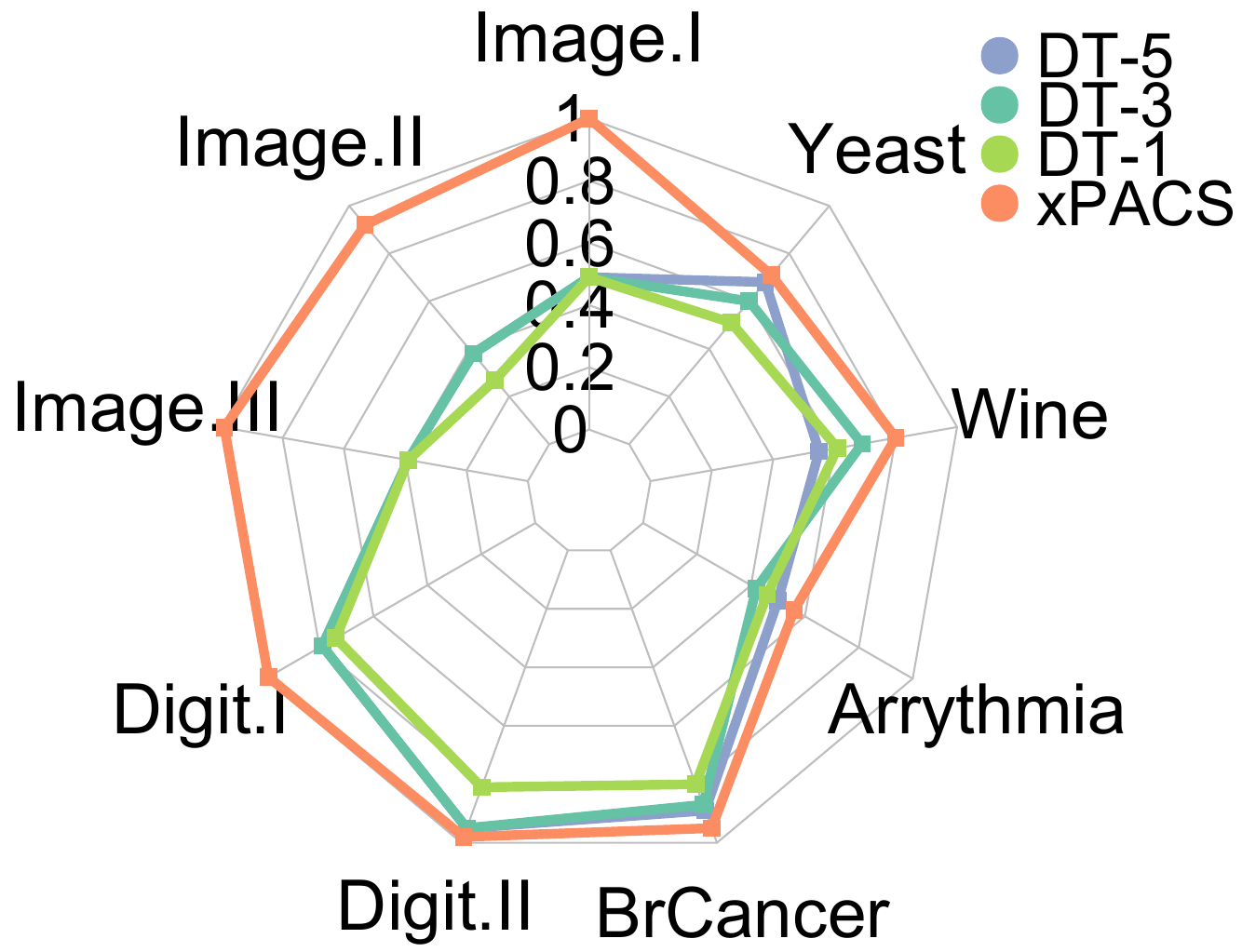}&\multicolumn{1}{b{.4\linewidth}}{\caption{\method~achieves the best balance between interpretability measures (a)--(d) [lower is better for all of them], and is significantly better at detection (e) [higher is better], as compared to several rule learners.\label{fig:radar}}}\\\\
		(e) AUPRC \\
		
	\end{tabular}
\end{figure}

\subsubsection{\method~vs. Rule Learners\\}
\label{ssec:dt}

Since in our work, we view anomalies as an already defined class, explaining anomalies is equivalent to describing an under represented target class \cite{wrobel1997algorithm}. Hence, we compare \method~to techniques that explain labeled data. To this end, we consider interpretable supervised models, specifically, inductive rule based learners that aim to extract rules from a labeled data set that are discriminative in nature. We argue that linear classifiers like logistic regression are not comparable to \method~for two key reasons. First, they {\em do not  group} the anomalies, but rather output a single separating hyperplane. Second, they {\em do not provide rules on the features}, but only feature coefficients, which could be negative (hard to interpret). Further, techniques aiming to explain black box predictions are not directly comparable to our method since most of the works aim to explain one instance at a time compared to the group wise explanations \method~provides.

We compare \method~to the following popular rule based learners. 
\begin{enumerate}
\item Decision Tree (DT): DT aims to partition (or group) the labeled data into pure leaves. 
We treat the leaves containing at least two anomalies analogous to our \ps. Each such leaf is characterized by the feature rules (or predicates) on the path from the root.
\item Ripper \cite{cohen1995fast}: Ripper is a popular inductive rule learner that sequentially mines for feature rules with high accuracy and coverage with the aim to achieve generalization. We use a publicly available implementation of Ripper in the Weka repository for our experiments and consider rules that are labeled anomalous.
\item RuleFit \cite{friedman2008predictive}: RuleFit is an ensemble learner where the base learner is a rule generated by a decision tree. A regression/classification is setup using the base learners to identify the rules that are important in discriminating the different classes. We use the publicly available RuleFit\footnote{R package pre : \url{https://CRAN.R-project.org/package=pre}} implementation and use the rules with non-zero coefficients with atleast two anomalies.
\end{enumerate}
  
To compare \method~with rule learners, it is not fair to use description length since the listed techniques do not explicitly optimize it. Instead, we use the following external interpretability measures proposed in \cite{journals/corr/LakkarajuKCL17} (all being lower the better):
(a) number of groups (anomalous packs), 
(b) avg. length of rules (pack dimensionality),
(c) avg. fraction of normal points within packs (impurity divided by $n$), and
(d) avg. interval width across feature rules.
In other words, an explanation with fewer groups, fewer rules, fewer exceptions, 
and smaller spread in features is considered more interpretable. 

DT has no means to choose the number of packs automatically.
Therefore, we report DT results for depths 1--5 as compared to \method~in Table \ref{tab:interpret},
averaged across all datasets. In addition to the interpretability measures, we report the detection performance in AUPRC (area under precision-recall curve) on held-out data (80-20 split) that quantifies the generalization of the subspace rules.
Results on individual datasets per measure are shown with radar charts in Fig. \ref{fig:radar}.
(See Table \ref{tab:big} for detailed results.)
Notice the trade-offs between the measures for DT:  while (c) and (d) tend to decrease with increasing depth, (a) and (b) increase. The lack of rule summarization in RuleFit is evident in the number of groups (a) where \method~consistently produces smaller number of explanations across various data sets. We also note that \method~produces tighter intervals (d) compared to Ripper emphasizing the concreteness of the explanations. Overall, \method~achieves the best trade-off with lower overall values  across the interpretability measures. Moreover, our signatures are significantly  better at detecting future anomalies. We present more detailed experiments on detection next.

\begin{table*}[!thb]
\vspace{-0.1in}
\caption{Rule learners and DT (with respective depths 1--5) compared to \method~across datasets on interpretability measures (a)--(d) [all lower the better] as well as detection performance AUPRC  [higher the better]. RuleFit leads to underspecified regression in \arr~and~\yeast~which we denote by NA.\label{tab:big}}
\centering
\scalebox{0.9}{\begin{tabular}{|c|c|c|c|c|c|c|c|c|c|c|}
\toprule
Measure & \begin{tabular}[c]{@{}c@{}}Dataset/\\ Model\end{tabular} & \begin{tabular}[c]{@{}c@{}}Image \\ I\end{tabular} & \begin{tabular}[c]{@{}c@{}}Image \\ II\end{tabular} & \begin{tabular}[c]{@{}c@{}}Image \\ III\end{tabular} & \begin{tabular}[c]{@{}c@{}}Digit\\  I\end{tabular} & \begin{tabular}[c]{@{}c@{}}Digit\\  II\end{tabular} & \begin{tabular}[c]{@{}c@{}}Br\\ Cancer\end{tabular} & \begin{tabular}[c]{@{}c@{}}Arry\\ thmia\end{tabular} & Wine & Yeast \\ \toprule
\multirow{6}{*}{\begin{tabular}[c]{@{}c@{}} (a) \\ number \\ of \\ groups\end{tabular}} & DT-5 & 1.00 & 2.00 & 2.00 & 3.00 & 2.00 & 10.00 & 9.00 & 1.00 & 6.00 \\  
 & DT-4 & 1.00 & 2.00 & 2.00 & 3.00 & 2.00 & 8.00 & 9.00 & 1.00 & 6.00 \\  
 & DT-3 & 1.00 & 2.00 & 2.00 & 3.00 & 2.00 & 6.00 & 6.00 & 1.00 & 4.00 \\  
 & DT-2 & 1.00 & 2.00 & 2.00 & 3.00 & 2.00 & 4.00 & 4.00 & 1.00 & 3.00 \\  
 & DT-1 & 1.00 & 2.00 & 2.00 & 2.00 & 2.00 & 2.00 & 2.00 & 1.00 & 2.00 \\  
 & RuleFit & 3.00 & 15.00 & 10.00 & 13.00 & 8.00 & 24.00 & NA & 11.00 & NA \\
 & Ripper & 1.00 & 2.00 & 2.00 & 2.00 & 2.00 & 3.00 & 5.00 & 1.00 & 4.00 \\
 & \method & 1.00 & 2.00 & 3.00 & 1.00 & 1.00 & 5.00 & 7.00 & 1.00 & 2.00 \\ \hline
\multirow{6}{*}{\begin{tabular}[c]{@{}c@{}} (b) \\ avg. \\ length \\ of rules\end{tabular}} & DT-5 & 1.00 & 2.00 & 1.50 & 3.00 & 2.50 & 4.40 & 4.33 & 4.00 & 4.17 \\  
 & DT-4 & 1.00 & 2.00 & 1.50 & 3.00 & 2.50 & 3.63 & 3.78 & 4.00 & 3.67 \\  
 & DT-3 & 1.00 & 2.00 & 1.50 & 2.67 & 2.50 & 3.00 & 3.00 & 3.00 & 3.00 \\  
 & DT-2 & 1.00 & 2.00 & 1.50 & 2.00 & 2.50 & 2.00 & 2.00 & 2.00 & 2.00 \\  
 & DT-1 & 1.00 & 1.00 & 1.00 & 1.00 & 1.00 & 1.00 & 1.00 & 1.00 & 1.00 \\  
 & RuleFit & 1.00 & 1.40 & 1.80 & 1.80 & 1.92 & 2.67 & NA & 1.81 & NA \\
 & Ripper & 1.00 & 1.00 & 1.50 & 2.00 & 1.50 & 1.67 & 1.80 & 2.00 & 1.25 \\
 & \method & 1.00 & 1.00 & 1.00 & 2.00 & 4.00 & 1.40 & 1.00 & 4.00 & 3.50 \\ \hline
\multirow{6}{*}{\begin{tabular}[c]{@{}c@{}} (c) avg. \\ fraction \\ of \\ normal \\ points\end{tabular}} & DT-5 & 0.00 & 0.00 & 0.00 & 0.00 & 0.00 & 0.00 & 0.10 & 0.00 & 0.11 \\  
 & DT-4 & 0.00 & 0.00 & 0.00 & 0.00 & 0.00 & 0.11 & 0.11 & 0.00 & 0.16 \\  
 & DT-3 & 0.00 & 0.00 & 0.00 & 0.00 & 0.00 & 0.16 & 0.16 & 0.07 & 0.24 \\  
 & DT-2 & 0.00 & 0.00 & 0.00 & 0.32 & 0.01 & 0.25 & 0.25 & 0.08 & 0.33 \\  
 & DT-1 & 0.00 & 0.50 & 0.50 & 0.50 & 0.50 & 0.50 & 0.50 & 0.15 & 0.50 \\  
 & RuleFit & 0.00 & 0.02 & 0.01 & 0.02 & 0.05 & 0.01 & NA & 0.05 & NA \\
 & Ripper & 0.00 & 0.01 & 0.00 & 0.00 & 0.00 & 0.01 & 0.10 & 0.03 & 0.01 \\
 & \method & 0.00 & 0.00 & 0.00 & 0.00 & 0.00 & 0.00 & 0.13 & 0.00 & 0.01 \\ \hline
\multirow{6}{*}{\begin{tabular}[c]{@{}c@{}} (d) \\ avg. \\ interval \\ width\end{tabular}} & DT-5 & 0.29 & 0.50 & 0.49 & 0.53 & 0.26 & 0.56 & 0.53 & 0.68 & 0.46 \\  
 & DT-4 & 0.29 & 0.50 & 0.49 & 0.53 & 0.26 & 0.54 & 0.53 & 0.68 & 0.51 \\  
 & DT-3 & 0.29 & 0.50 & 0.49 & 0.53 & 0.26 & 0.50 & 0.54 & 0.72 & 0.51 \\  
 & DT-2 & 0.29 & 0.50 & 0.49 & 0.51 & 0.29 & 0.50 & 0.50 & 0.79 & 0.52 \\  
 & DT-1 & 0.29 & 0.50 & 0.50 & 0.50 & 0.50 & 0.50 & 0.50 & 0.73 & 0.50 \\  
 & RuleFit & 0.28 & 0.25 & 0.39 & 0.53 & 0.47 & 0.61 & NA & 0.72 & NA \\
 & Ripper & 0.24 & 0.15 & 0.31 & 0.32 & 0.28 & 0.7 & 0.33 & 0.79 & 0.38 \\
 & \method & 0.21 & 0.15 & 0.16 & 0.31 & 0.19 & 0.20 & 0.39 & 0.21 & 0.28 \\ \hline\hline
\multirow{6}{*}{AUPRC} & DT-5 & 0.49 & 0.38 & 0.39 & 0.79 & 0.95 & 0.89 & 0.50 & 0.55 & 0.68 \\  
 & DT-4 & 0.49 & 0.38 & 0.39 & 0.79 & 0.95 & 0.86 & 0.50 & 0.55 & 0.55 \\  
 & DT-3 & 0.49 & 0.38 & 0.39 & 0.79 & 0.95 & 0.87 & 0.42 & 0.69 & 0.60 \\  
 & DT-2 & 0.49 & 0.38 & 0.39 & 0.79 & 0.95 & 0.86 & 0.47 & 0.61 & 0.67 \\  
 & DT-1 & 0.49 & 0.27 & 0.39 & 0.74 & 0.81 & 0.80 & 0.46 & 0.61 & 0.51 \\  
 & RuleFit & 0.93 & 0.93 & 0.51 & 0.98 & 0.97 & 0.90 & NA & 0.71 & NA \\
 & Ripper & 1.00 & 0.35 & 0.43 & 0.97 & 0.98 & 0.91 & 0.43 & 0.70 & 0.75 \\
 & \method & 1.00 & 0.92 & 0.99 & 0.99 & 0.98 & 0.95 & 0.56 & 0.80 & 0.71 \\ \bottomrule
\end{tabular}}
\end{table*}

\subsubsection{Ablation Study}
We study the importance of the refinement step discussed in \S\ref{ssec:refine} by performing an ablation study. To this end, we omit the refinement of hyper-rectangles into hyper-ellipsoids in \method~(denoted as ablated~\method). Recall that the primary reason we perform the refinement step is to cover more anomalous points and reduce the number of normal points in the packs (See Fig. \ref{fig:refine}). Hence, to showcase the benefit, in Table \ref{tab:ablation}, we compare the proportion of anomalous (higher is better) and normal points (lower is better) covered in the final packs obtained using \method~and the ablated~\method. In addition, we also report the \%-savings (higher is better) achieved in both cases. In \method, the summarization step (See \S\ref{ssec:summarize}) transmits the center and the diagonal matrix of the packs \S\ref{sec:mdlcoding}. To accommodate hyper rectangles, we modify this to transmit the upper and lower bounds of the hyper rectangles in the ablated \method. 

\begin{table}[!htbp]
\centering
\caption{Ablation Study: \method~vs. ablated~\method~(no refinement to ellipsoids).  Coverage of anomalous points (higher is better), coverage of normal points (lower is better), and \% savings (higher is better).}
\label{tab:ablation}
\scalebox{0.75}{\begin{tabular}{ll||ccccccccc}
\hline
 & \textbf{Method} & \multicolumn{1}{c}{\textbf{\begin{tabular}[c]{@{}c@{}}Images\\      I\end{tabular}}} & \multicolumn{1}{c}{\textbf{\begin{tabular}[c]{@{}c@{}}Images\\      II\end{tabular}}} & \multicolumn{1}{c}{\textbf{\begin{tabular}[c]{@{}c@{}}Images\\ III\end{tabular}}} & \multicolumn{1}{c}{\textbf{\begin{tabular}[c]{@{}c@{}}Digit\\ I\end{tabular}}} & \multicolumn{1}{c}{\textbf{\begin{tabular}[c]{@{}c@{}}Digit\\ II\end{tabular}}} & \multicolumn{1}{c}{\textbf{\begin{tabular}[c]{@{}c@{}}Br\\ Cancer\end{tabular}}} & \multicolumn{1}{l}{\textbf{\begin{tabular}[c]{@{}l@{}}Arry\\ thmia\end{tabular}}} & \multicolumn{1}{l}{\textbf{Wine}} & \multicolumn{1}{l}{\textbf{Yeast}} \\ \hline
\multirow{2}{*}{\begin{tabular}[c]{@{}l@{}}Coverage of\\ anom. points\end{tabular}} & x-PACS & \textbf{1.00} & \textbf{0.89} & \textbf{1.00} & \textbf{1.00} & \textbf{1.00} & \textbf{0.95} & \textbf{0.93} & \textbf{0.96} & \textbf{0.68} \\ \cline{2-11} 
 & \begin{tabular}[c]{@{}l@{}}ablated\\ x-PACS\end{tabular} & \textbf{1.00} & \textbf{0.89} & \textbf{1.00} & 0.96 & 0.89 & 0.88 & 0.78 & 0.92 & 0.62 \\ \hline
\multirow{2}{*}{\begin{tabular}[c]{@{}l@{}}Coverage of\\ normal points\end{tabular}} & x-PACS & \textbf{0} & \textbf{0} & \textbf{0} & \textbf{0} & \textbf{0} & \textbf{0.03} & \textbf{0.35} & \textbf{0.11} & \textbf{0.10} \\ \cline{2-11} 
 & \begin{tabular}[c]{@{}l@{}}ablated\\ x-PACS\end{tabular} & \textbf{0} & \textbf{0} & \textbf{0} & 0.01 & 0.01 & 0.05 & 0.53 & 0.18 & 0.13 \\ \hline
\multirow{2}{*}{\%-savings} & x-PACS & \textbf{99.75} & \textbf{88.53} & \textbf{99.51} & \textbf{99.83} & \textbf{99.72} & \textbf{93.74} & \textbf{92.92} & \textbf{97.04} & \textbf{98.04} \\ \cline{2-11} 
 & \begin{tabular}[c]{@{}l@{}}ablated\\ x-PACS\end{tabular} & \textbf{99.75} & \textbf{88.53} & \textbf{99.51} & 92.11 & 87.21 & 85.68 & 78.16 & 91.42 & 90.51 \\ \hline
\end{tabular}}
\end{table}

From Table \ref{tab:ablation}, we observe that \method~is indeed able to cover more anomalous points, while avoiding normal points in the final packs for all the datasets.
These results demonstrate the utility of the refinement step.


\subsection{Detection Performance}
\label{ssec:detection}

While not our primary focus, \method~can also be used to \textit{detect} anomalies.
Specifically, given the \ps~identified from historical/training data, a future test instance that falls in any one of the \ps~(i.e., enclosed within any hyper-ellipsoid in the \pg) can be flagged as an anomaly.\footnote{Note that, like any supervised method, \method~could only detect future instances of anomalies of known types.}

To measure detection quality, we compare \method~to 7 competitive baselines on all datasets.

\begin{enumerate}
\setlength{\itemsep}{0.1in}
 
 \item Mixture of {\sc $K$-Gaussians} on the anomalous points. $K\in \{1,2,\ldots, 9\}$ chosen at the ``knee'' of likelihood. Anomaly score of test instance: maximum of the probabilities of being generated from each cluster.

 \item {\sc KDE} on the normal points. 
 Gaussian kernel bandwidth chosen by cross-validation.  Anomaly score:  negative of the density at test point.

\item  {\sc NN}. Anomaly score: distance of test point to its nearest neighbor (nn) normal point in training set, divided by the  distance of that nn point to its own nearest normal point in training set. 

 \item  {\sc PCA+SVDD} on all points \cite{journals/ml/TaxD04}. A \textit{single} hyperball that aims to enclose anomalous points in the \textit{PCA-reduced} space\footnote{SVDD optimization diverged for some high dimensional datasets, therefore, we performed PCA as a preprocessing step.}, for which the embedding dimensionality is chosen at the ``knee'' of the scree plot. 
 Anomaly score: distance of test point from the hyperball's center.
 
\item {\sc DT} on all points, where we balance the data for training and 
regularize by tree-depth, chosen from $\{1,2,\ldots,30\}$ via cross-validation.
Anomaly score: number of anomalous samples in the leaf the test point falls into divided by leaf size. 
  
\item[6-7)] {\sc SVM-lin} \& {\sc SVM-RBF} on all points. Hyperparameters set by cross-validation.
Anomaly score: ``confidence'', i.e., distance from decision boundary.

\end{enumerate}

	

\begin{table*}[!t]
	\setlength{\tabcolsep}{3pt}
	\caption{Area under precision-recall curve (AUPRC) on anomaly detection.} 
	{
	\begin{center}
		 \vspace{-0.2in}
		\begin{tabular}{l|ccccccccc} 
			\toprule
			{\bf Method} & \hspace{-0.08in} \imageo  & \hspace{-0.1in} \imaget  & \hspace{-0.1in} \imageth & \hspace{-0.1in} \digito & \hspace{-0.1in} \digitt &\hspace{-0.1in}  \cancer & \hspace{-0.05in}\arr& \wine & \yeast \\\toprule
			{\sc $K$-Gaussians}	&\hspace{-0.08in} 0.182  &\hspace{-0.1in} 0.239  & 0.184 & 0.162 & 0.333 & 0.613  & 0.227  & 0.258  &  0.265 \\\hline
			{\sc KDE} 			&\hspace{-0.08in} 0.952  &\hspace{-0.1in} 0.978  & 0.987 & 0.989 & 0.997 & 0.981   & 0.571 & 0.667  &  0.681 \\\hline
			{\sc NN}			&\hspace{-0.08in} 0.491  &\hspace{-0.1in} 0.472  & 0.659 & 0.967 & 0.821 & 0.520  & 0.546  & 0.562  &  0.348  	\\\hline
			{\sc PCA+SVDD}		&\hspace{-0.08in} 0.286  &\hspace{-0.1in} 0.217  & 0.212  & 0.331 & 0.529 & 0.861  & 0.295  & 0.566  &	 0.606  	 \\\hline
			{\sc DT}			&\hspace{-0.08in} 0.802  &\hspace{-0.1in} 0.764  & 0.812 & 0.831 & 0.961 & 0.884  & 0.516  & 0.637  &  0.673 \\\hline
			{\sc SVM-Lin}		&\hspace{-0.08in} {\textbf{1.000}} &\hspace{-0.1in} {\textbf{1.000}} & {\textbf{1.000}}& 0.999 & 0.999 & {\textbf{0.984}}  & 0.755  & {\textbf{0.994}}  &  0.823  \\\hline
			{\sc SVM-RBF}		&\hspace{-0.08in} {\textbf{1.000}}& \hspace{-0.1in} {\textbf{1.000}} & {\textbf{1.000}} & {\textbf{1.000}} & {\textbf{1.000}} & 0.964  & {\textbf{0.810}}  & 0.984  &  {\textbf{0.861}}\\
		\hline\hline
			\method &\hspace{-0.08in} {\textbf{1.000}}  &\hspace{-0.1in} 0.921  & 0.990 & 0.993 & 0.976 & 0.951  & 0.564 &  0.799  &  0.701\\\bottomrule			 	
		\end{tabular}
	
		\label{tab:auprc}
	\end{center}
}
\end{table*}


We create 3 folds of each dataset, and in turn use 2/3 for training and 1/3 for testing, except the {\tt Images} datasets with the fewest anomalies for which we do leave-one-out testing. 
All points receive an anomaly score by each method as described above. 
\method's anomaly score for a test instance $\mx$ is the maximum $h_k(\mx) = \mx^T\mU_k\mx + \mw_k^T\mx + w_{0k}$ among all $p_k$'s in the \pg~resulting from training data. 
We rank points in decreasing order of their score, and report the area under the precision-recall curve in Table \ref{tab:auprc}.

{\sc SVM}s  achieve the highest detection rate, as might be expected.
However, kernel SVM cannot be interpreted.
Linear SVM, like LR, does not identify anomalous patterns nor does it produce any explicit feature rules. 
Notably, \method~outperforms all other baselines considerably across datasets, including DT, which produces the most interpretable output among the baselines as discussed in \S\ref{ssec:dt}.

\subsection{Scalability}

Finally, we quantify the scalability of \method~empirically.
To this end, we implement a synthetic data generator, parameterized by data size, total dimensionality, maximum pack size and dimensionality and number of anomalous packs. Anomalies are sampled from a small range per feature within a subspace, and normal points are sampled from the reverse of the histogram densities derived from the anomalous points.

Fig. \ref{fig:runtime} shows the running time w.r.t. data size $m$, 
dimensionality $d$, average pack dimensionality $d_{\text{avg}}$, and total number of anomalies $a$. All plots demonstrate near-linear scalability. Recall that we showed  exponential complexity w.r.t. maximum pack dimensionality $d_{\max}$. Notably, we observe linear time growth on average.

\begin{figure}[h]
	\centering
\begin{tabular}{cp{0.05in}c} 
	\includegraphics[width=0.4\textwidth]{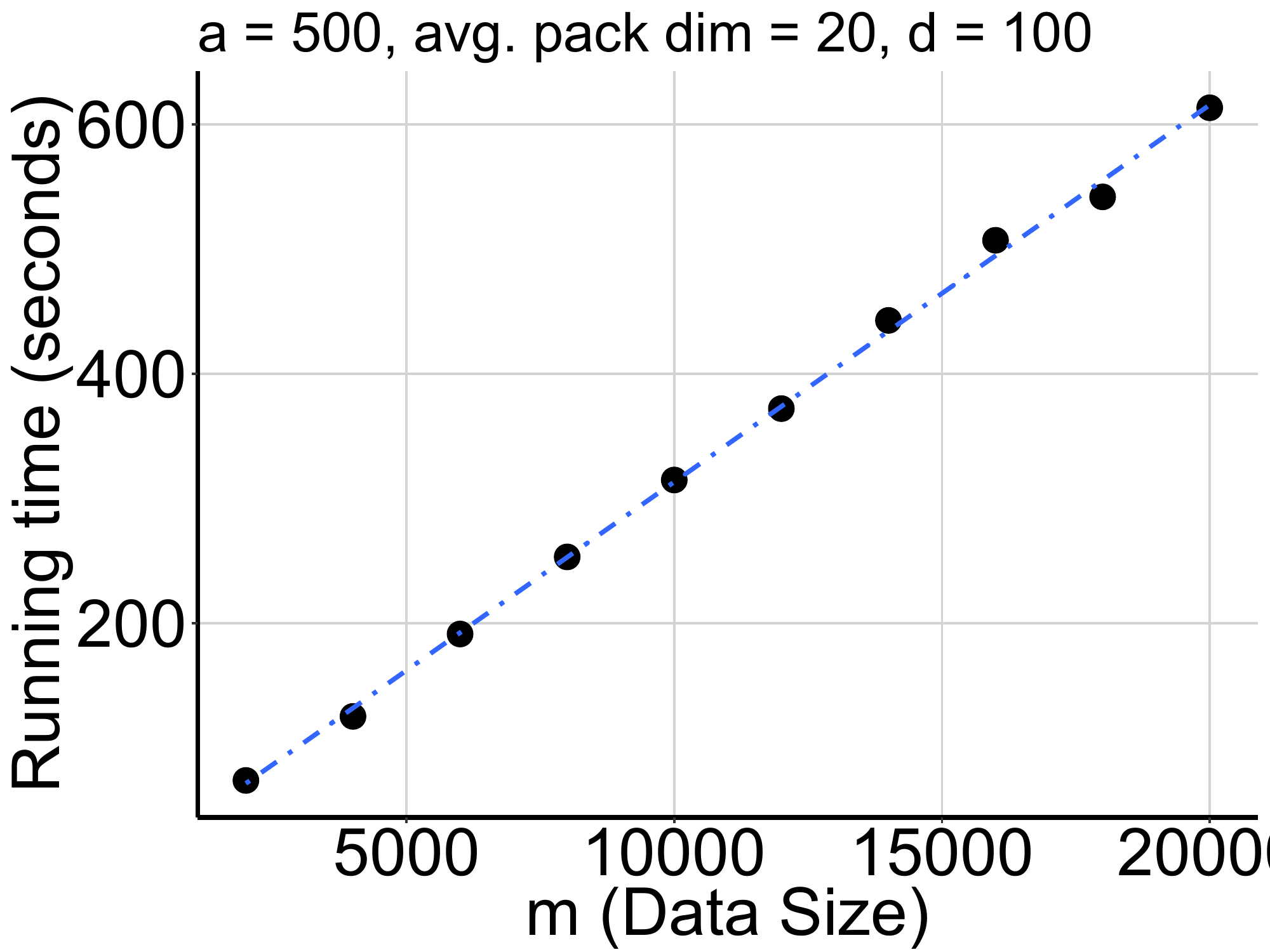} &&
	\includegraphics[width=0.4\textwidth]{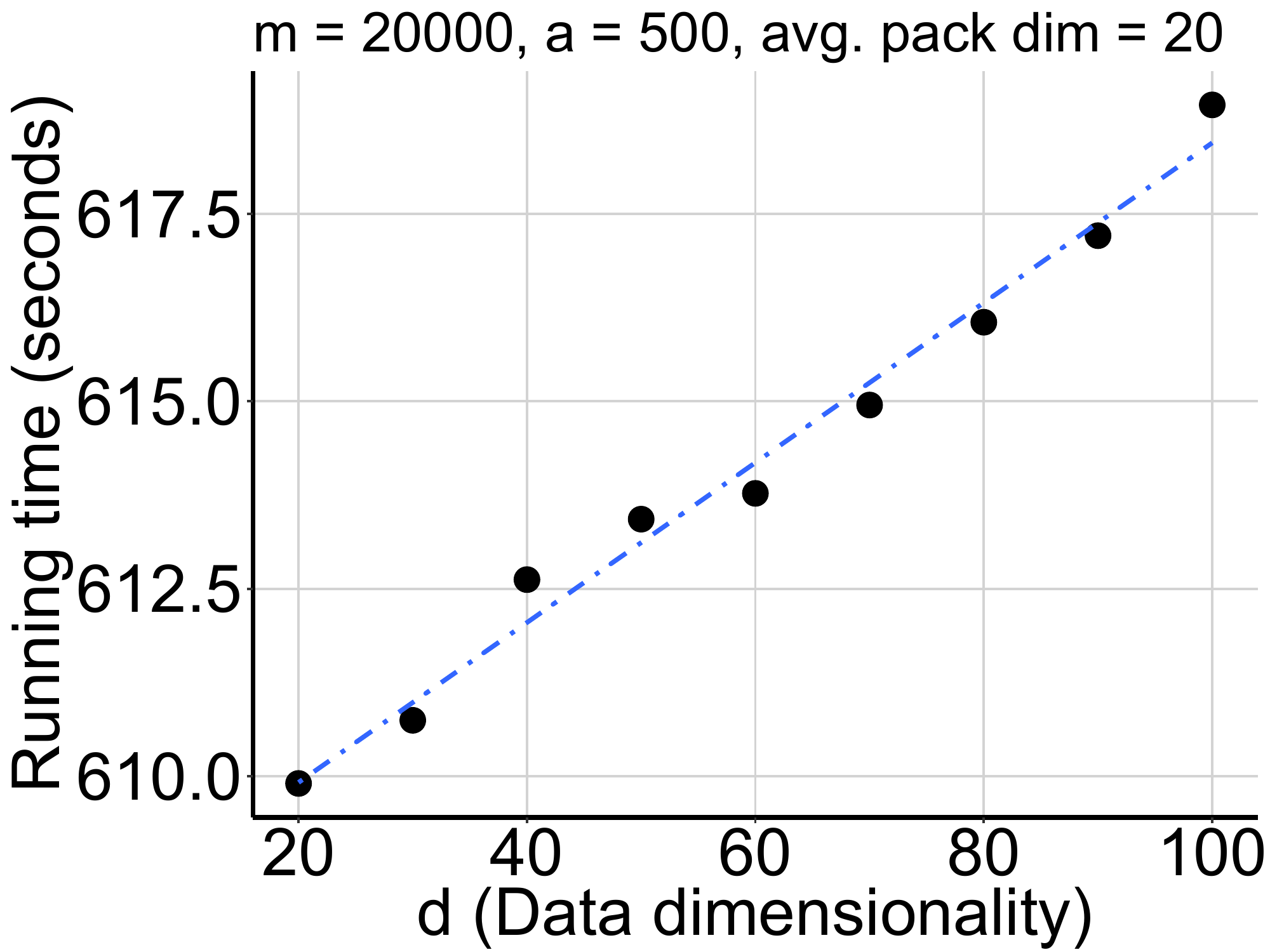}\\\\
 	\includegraphics[width=0.4\textwidth]{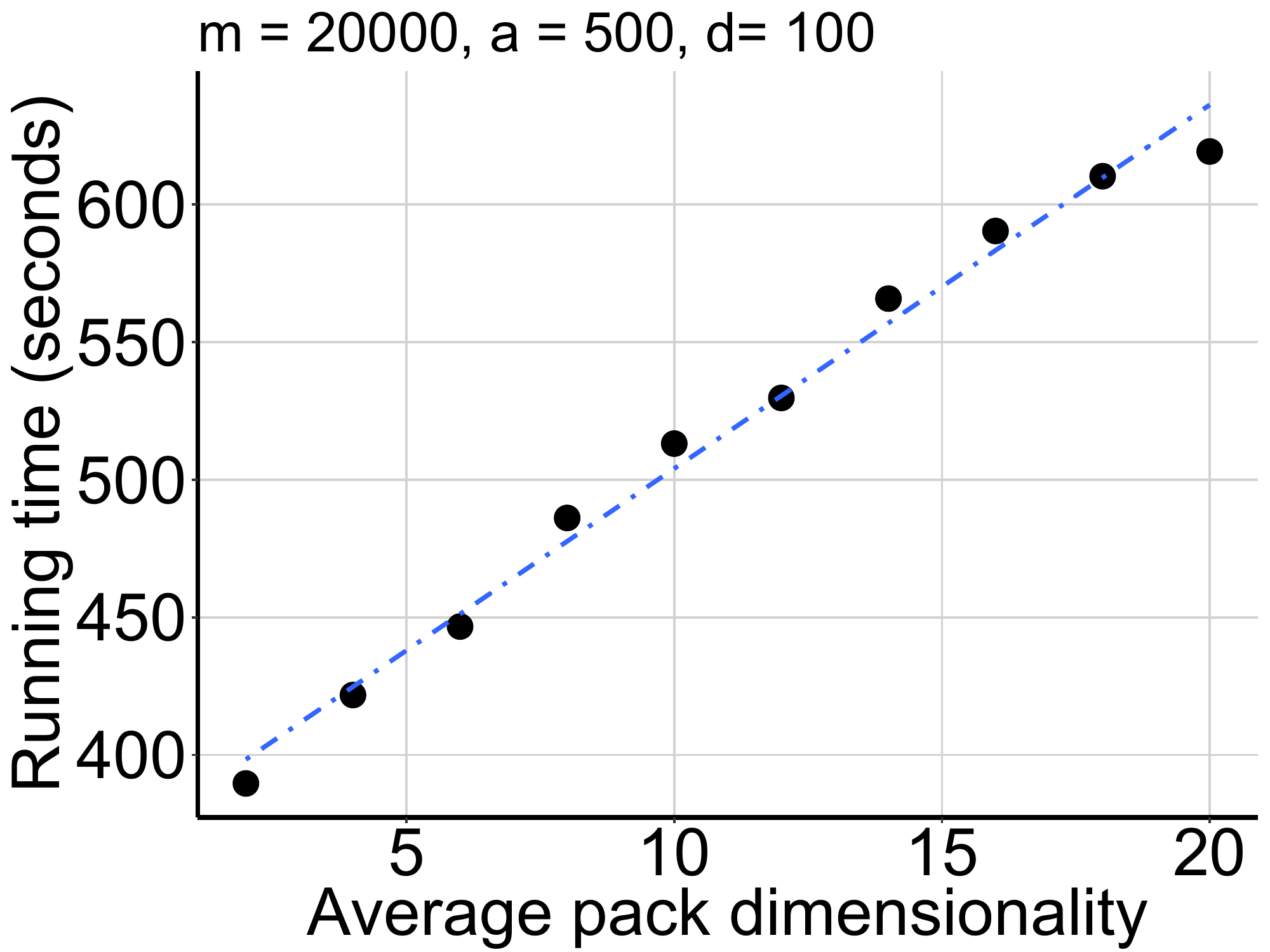}&&
	\includegraphics[width=0.4\textwidth]{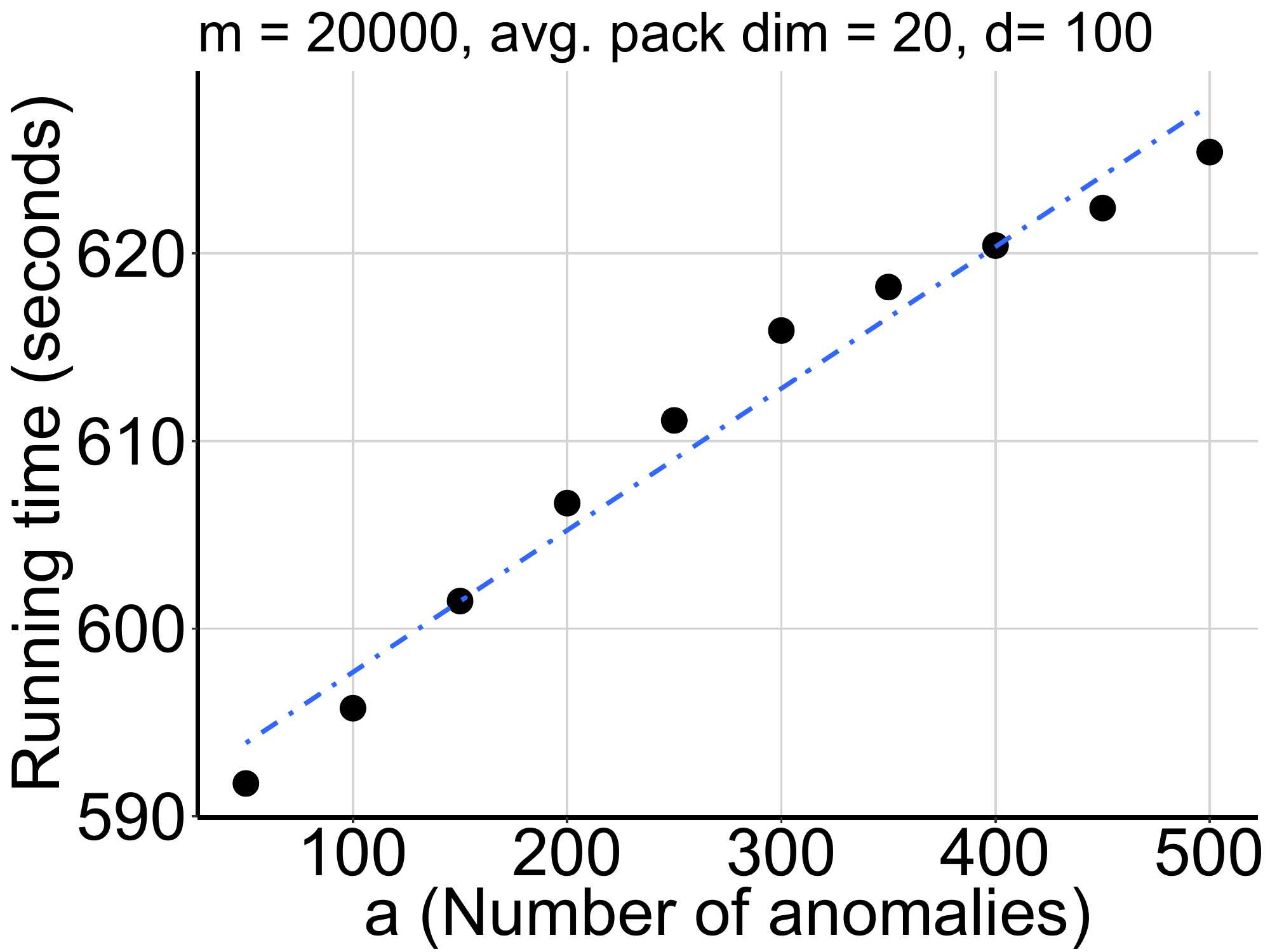} \\
\end{tabular}
\caption{\method~scales linearly with  input size.}
\label{fig:runtime}
 \vspace{-0.25in}
\end{figure}
\section{Related Work}
\label{sec:related}

Related areas of study span across outlier explanation, subspace clustering and subspace outlier detection, data description, subgroup discovery, rule learning, rare class discovery and approaches aiming to explain black box classifiers. 
We show the highlights of related work in the context of our desiderata in Table \ref{tab:related}.

\begin{table*}[!t]
	\caption{{Comparison of related work in terms of properties D1--D5 in reference to our Desiderata  (see \S\ref{ssec:des}).} } 
	\vspace{-0.15in}
	\begin{center}
		\scalebox{0.83}{\begin{tabular}{r||cccccc} 
			\toprule
			\begin{sideways}Property\end{sideways}  & \rotatebox{25}{Explain as-a-group?} & \hspace{-0.6in}\rotatebox{25}{Multiple groups?} & 
			\hspace{-0.4in}\rotatebox{25}{Find subspace?} & 
			\hspace{-0.4in}\rotatebox{25}{Rules on features?} & 
			\hspace{-0.5in}\rotatebox{25}{Discriminative?} & 
			\hspace{-0.4in}\rotatebox{25}{Minimal?} \\\toprule
			Subspace clustering \cite{conf/sigmod/AgrawalGGR98,conf/kdd/ChengFZ99,conf/icdm/SequeiraZ04,conf/icdm/KriegelKRW05,conf/icdm/MullerAGKS09} & \hspace{-0.9in} \checkmark & \hspace{-1.4in} \checkmark& \hspace{-1.1in} \checkmark& \hspace{-1.2in} \checkmark&  & \\
			Projected clustering \cite{conf/sigmod/AggarwalPWYP99,conf/icdm/MoiseSE06} & \hspace{-0.9in} \checkmark & \hspace{-1.4in} \checkmark& \hspace{-1.1in} \checkmark& \hspace{-1.2in} &  & \\\hline
			(data descr.)	SVDD  \cite{journals/ml/TaxD04}, 
			SSSVDD \cite{conf/pkdd/GornitzKB09}
			& \hspace{-0.9in} \checkmark & \hspace{-1.4in} & \hspace{-1.1in} & \hspace{-1.2in} \checkmark& \hspace{-1.2in} \checkmark & \\
			(rare category) RACH \cite{conf/icdm/HeTC10} 
			& \hspace{-0.9in} \checkmark & \hspace{-1.4in} & \hspace{-1.1in} & \hspace{-1.2in} \checkmark& \hspace{-1.2in} \checkmark & \\
			(rare category) PALM \cite{conf/sdm/HeC10} 
			& \hspace{-0.9in} \checkmark & \hspace{-1.4in} & \hspace{-1.1in}\checkmark & \hspace{-1.2in} \checkmark& \hspace{-1.2in} \checkmark & \\ \hline
			Knorr and Ng \cite{conf/vldb/KnorrN99} 
			& \hspace{-0.9in}  & \hspace{-1.4in} & \hspace{-1.1in}\checkmark & \hspace{-1.2in} & \hspace{-1.2in} \checkmark & \hspace{-0.8in} \checkmark\\ \hline
			RefOUT \cite{conf/cikm/KellerMWB13}, 
			CP \cite{conf/aaai/KuoD16}		
			LODI \cite{conf/pkdd/DangMAN13}, 
			LOGP \cite{conf/icde/DangANZS14} 
			& \hspace{-0.9in}  & \hspace{-1.4in} & \hspace{-1.1in}\checkmark & \hspace{-1.2in} & \hspace{-1.2in} \checkmark & \hspace{-0.8in} \\ \hline
			EXPREX \cite{journals/tkde/AngiulliFP13} 
			& \hspace{-0.9in}\checkmark  & \hspace{-1.4in} & \hspace{-1.1in}\checkmark & \hspace{-1.2in}\checkmark & \hspace{-1.2in} \checkmark & \hspace{-0.8in} \\ \hline
			(Explaining black box classifiers) LIME \cite{ribeiro2016should} 
			& \hspace{-0.9in}  & \hspace{-1.4in} & \hspace{-1.1in}\checkmark & \hspace{-1.2in} & \hspace{-1.2in} \checkmark & \hspace{-0.8in} \\ \hline
			EXstream \cite{conf/edbt/ZhangDM17}& \hspace{-0.9in}\checkmark  & \hspace{-1.4in} & \hspace{-1.1in}\checkmark & \hspace{-1.2in}\checkmark & \hspace{-1.2in} \checkmark & \hspace{-0.8in}\checkmark \\ \hline
			Explainer \cite{pevny2014interpreting} & \hspace{-0.94in}  & \hspace{-1.4in} & \hspace{-1.1in}\checkmark & \hspace{-1.24in} \checkmark& \hspace{-1.2in} \checkmark & \hspace{-0.8in} \\ \hline
			SRF \cite{kopp2014interpreting}, Krimp \cite{vreeken2011krimp}, RuleFit \cite{friedman2008predictive}, Ripper \cite{cohen1995fast}
			& \hspace{-0.94in} \checkmark & \hspace{-1.4in} \checkmark& \hspace{-1.1in}\checkmark & \hspace{-1.24in} \checkmark& \hspace{-1.2in} \checkmark & \hspace{-0.8in} \\ 

			\hline\hline
			\method~[this paper] 	& \hspace{-0.9in}\checkmark  & \hspace{-1.4in}\checkmark & \hspace{-1.1in}\checkmark & \hspace{-1.2in}\checkmark & \hspace{-1.2in} \checkmark & \hspace{-0.8in}\checkmark \\ \bottomrule
		\end{tabular}}
		\label{tab:related}
	\end{center}
\end{table*}

\vspace{0.025in}
\textbf{\em Outlier explanation:}
The seminal work \cite{conf/vldb/KnorrN99} provides what they call ``intensional knowledge'', per outlier, by identifying  the minimal subspaces in which it deviates.
To find the optimal subset of features that differentiate the outliers from normal points,  \cite{conf/aaai/KuoD16} formulates a constraint programming problem and \cite{conf/cikm/KellerMWB13} takes a subspace search route. Similarly, \cite{conf/icde/DangANZS14,conf/pkdd/DangMAN13,conf/icdm/MicenkovaNDA13}  aim to explain one outlier at a time by features that participate in projection directions that maximally separate them from normal points. 
\textit{All} existing work in this area assume the outliers are \textit{scattered} and strive to explain  them \textit{individually} rather than in groups. Therefore, they cannot identify anomalous patterns.
Moreover, they do not focus explicitly on 
shortest description, let alone in a principled, information-theoretic way as we address in this work.

Extending earlier work \cite{journals/tods/AngiulliFP09} on explaining single outliers,  \cite{journals/tkde/AngiulliFP13} aims to explain groups of outlier points or what they call sub-populations.
They search for $\langle$context, feature$\rangle$ pairs, where the (single) feature can differentiate as many outliers as possible from normal points that share the same context.
It is important to note that their goal is to explain a group (or set) of outliers collectively and not particularly explaining them with multiple groups.
Similarly, \cite{conf/edbt/ZhangDM17} describes anomalies grouped in time. They construct explanatory Conjunctive Normal Form rules using features with low segmentation entropy, which quantifies how intermixed normal and anomalous points are. They heuristically discard highly correlated features from the rules to get minimal explanations.
Again, they strive to explain all the anomalies as a group, and not in multiple groups.

 We found that SRF (sapling random forest) \cite{kopp2014interpreting} aims to explain and cluster outliers similar to our problem setting. They build on their earlier work \cite{pevny2014interpreting}, which explains outliers one at a time by learning an ensemble of  small decision trees (called saplings) and combining the rules (from root to leaf in which the outlier lies) across the trees. SRF then groups the outliers using k-means clustering based on the similarity of their explanations. However, there is no guarantee on the minimality of their overall description, since grouping is done as a post-processing step and by using a local-optima-prone clustering algorithm. Moreover, there is not much discussion in their paper on the choice of the number of clusters, nor the format of the final description after the anomalies are clustered. We are not aware of a publicly available implementation of SRF to compare with the proposed method and hence omit it from the experimental evaluation.

\vspace{0.025in}
\textbf{\em Subspace clustering and Outlier detection:}
There is a long list of work on subspace clustering \cite{conf/sigmod/AgrawalGGR98,conf/kdd/ChengFZ99,conf/icdm/KriegelKRW05,conf/icdm/MullerAGKS09,conf/icdm/SequeiraZ04}  that aim to find high-density clusters in feature subspaces. (See \cite{Parsons2004,Kriegel09} for reviews.)
Some others are projection based that work in {transformed} feature spaces \cite{conf/sigmod/AggarwalPWYP99,conf/icdm/MoiseSE06}.
However, these are unsupervised methods and their goal is not explaining labeled data, nor they focus on {minimal} explanations.
There is also a long list of subspace-based outlier \textit{detection} methods \cite{conf/icde/KellerMB12,conf/pakdd/KriegelKSZ09,conf/icdm/KriegelKSZ12,conf/icdm/MullerAIMB12,conf/icde/MullerASS08,conf/icde/MullerSS11}, however, they do not address the description problem.

\vspace{0.025in}
\textbf{\em Data description and Rare class discovery:}
Another line of related work is data description \cite{conf/pkdd/GornitzKB09,journals/ml/TaxD04} and rare class (or category) characterization \cite{conf/icdm/HeTC10,conf/sdm/HeC10}.
The main goal behind all of these work is to explain the data (normal and rare-class points, respectively) within a  separating hyperball.
However, all of them assume that those points cluster in a \textit{single} hyperball, and with the exception of \cite{conf/sdm/HeC10}, search for a  \textit{full-dimensional} enclosing hyperball. As such, they do not address the curse-of-dimensionality or identify multiple clusters embedded in different subspaces.

 \vspace{0.025in}
\textbf{\em Subgroup discovery and Rule learning:}
Classification rule learning algorithms have the objective of generating models consisting of a set of rules inducing properties of all the classes of the target variable, while in subgroup discovery the objective is to discover individual rules of interest (See \cite{herrera2011overview} for an overview). The seminal works in rule based learners Ripper \cite{cohen1995fast}, CN2 \cite{clark1989cn2} sequentially mine for rules with high accuracy and coverage. More recently, \cite{friedman2008predictive} propose RuleFit, an ensemble learner where the base learner is a rule generated by a decision tree. A regression/classification is setup using the base learners to identify the rules that are important in discriminating the different classes. Few other works in ensemble learners \cite{deng2014interpreting,hara2016making} build ensemble trees that are interpretable. While the rules are interpretable, they are learnt with an aim to achieve generalization. This is different from our work where we primarily focus on describing the under represented class (anomalies) without emphasizing the generalizability.

SubgroupMiner \cite{klosgen2002census} extends seminal works in subgroup discovery (MIDOS \cite{wrobel1997algorithm}, Explora \cite{klosgen1996explora}) to handle numerical and categorical attributes. SD \cite{gamberger2002expert} propose an interactive subgroup discovery technique based on the variation of beam search algorithms guided by expert knowledge. Krimp \cite{vreeken2011krimp} propose a greedy MDL based approach to mine frequent item sets to describe points of a given class. Further, a classifier is proposed by mining frequent item sets on various classes independently. Discriminative constrast pattern mining techniques \cite{conf/cikm/LoekitoB08} assume nominal features and aim to extract contrast patterns (itemsets) with large support difference across categories.

A key difference in the techniques discussed above to our work is the summarization scheme discussed in \S\ref{ssec:summarize}. Our MDL based encoding scheme leads to a submodular rule selection with theoretical guarantees that the current subgroup discovery or rule learning algorithms do not explore. The improvement of the summarization scheme is evident from the experiments (See Table \ref{tab:big}) comparing \method~to rule learners on various interpretability measures.


\vspace{0.025in}
\textbf{\em Explaining black-box classifiers:}
Approaches such as \cite{fong2017interpretable,koh2017understanding,montavon2017methods,ribeiro2016should} aim to explain the decision made by a black box predictor. LIME \cite{ribeiro2016should} finds nearest neighbors to single input labeled example to construct a linear interpretable model that is locally faithful to the predictor. Further, authors propose a sub modular optimization framework to pick instances that are representative of the predictions of a classifier. Other works \cite{fong2017interpretable,koh2017understanding} explain the model by perturbing the features to quantify the influence on prediction. All of these works do not aim to explain multiple instances collectively, as such they do not address the summarization of the explanations and are hence not comparable to the proposed method.

All in all, none of the existing related methods provides all of 1) \textit{collective}, rather than individual, explanations,
2) explanations for \textit{multiple} anomalous groups,
3) in characterizing \textit{subspaces}, 4) using {interpretable} \textit{feature rules} that can 5)  \textit{discriminate}  anomalies from normal points, 6) aiming to \textit{minimize} description length.



\section{Conclusion}
\label{sec:conclusion}

We considered the problem of explaining given anomalies in high-dimensional datasets in groups. Our key idea is to describe the data by the patterns it contains. We proposed \method~for
identifying a small number of low-dimensional anomalous patterns that  ``pack'' similar, clustered anomalies and ``compress'' the data most succinctly. 
In designing \method, we combined ideas from data mining (bottom-up algorithms with pruning), optimization
(nonlinear quadratic discrimination), information theory (data encoding with bits), and theory of algorithms (nonmonotone submodular function maximization).
Our notable contributions are listed as follows.
\bit
\setlength{\itemsep}{0.025in}
\item \textbf{A new desiderata} for the anomaly description problem, enlisting five desired properties (D1--D5),

\item \textbf{A new problem formulation}, for explaining a given set of anomalies \textit{in groups} (D1), 

\item \textbf{Description algorithm}  \method, which provides low-dimensional (D2), interpretable (D3), and discriminative (D4) feature rules per anomalous group, 

\item \textbf{A new anomaly encoding scheme}, based on the minimum description length (MDL) principle, 
that lends itself to efficient optimization to produce \textbf{minimal explanations} (D5) with guarantees.

%
%
\eit

Through experiments on real-world datasets, we showed the effectiveness of \method~both in explanation and detection and superiority to competitive baselines.
For reproducibility, all of our source code and datasets
are publicly released   at {{\url{https://github.com/meghanathmacha/xPACS}}}.

\section*{Acknowledgments}{\vspace{-0.1in}\small
	This research is sponsored by 
	NSF CAREER 1452425 and IIS 1408287, 
	ARO Young Investigator Program under Contract No. W911NF-14-1-0029, 
and the PwC Risk and Regulatory Services Innovation Center at Carnegie Mellon University.
	Any conclusions expressed in this material are of the authors and do not necessarily reflect the views, either expressed or implied, of the funding parties.}

\bibliographystyle{abbrv}
\bibliography{refs}

\begin{thebibliography}{10}

\bibitem{books/sp/Aggarwal2013}
C.~C. Aggarwal.
\newblock {\em Outlier Analysis}.
\newblock Springer, 2013.

\bibitem{conf/sigmod/AggarwalPWYP99}
C.~C. Aggarwal, C.~M. Procopiuc, J.~L. Wolf, P.~S. Yu, and J.~S. Park.
\newblock Fast algorithms for projected clustering.
\newblock In {\em SIGMOD}, pages 61--72, 1999.

\bibitem{conf/sigmod/AgrawalGGR98}
R.~Agrawal, J.~Gehrke, D.~Gunopulos, and P.~Raghavan.
\newblock Automatic subspace clustering of high dimensional data for data
  mining applications.
\newblock In {\em SIGMOD}, pages 94--105, 1998.

\bibitem{journals/tods/AngiulliFP09}
F.~Angiulli, F.~Fassetti, and L.~Palopoli.
\newblock Detecting outlying properties of exceptional objects.
\newblock {\em ACM Trans. Database Syst.}, 34(1), 2009.

\bibitem{journals/tkde/AngiulliFP13}
F.~Angiulli, F.~Fassetti, and L.~Palopoli.
\newblock Discovering characterizations of the behavior of anomalous
  subpopulations.
\newblock {\em IEEE TKDE}, 25(6):1280--1292, 2013.

\bibitem{conf/soda/BuchbinderFNS14}
N.~Buchbinder, M.~Feldman, J.~Naor, and R.~Schwartz.
\newblock Submodular maximization with cardinality constraints.
\newblock In {\em SODA}, pages 1433--1452, 2014.

\bibitem{conf/kdd/ChengFZ99}
C.~H. Cheng, A.~W.-C. Fu, and Y.~Zhang.
\newblock Entropy-based subspace clustering for mining numerical data.
\newblock In {\em KDD}, pages 84--93, 1999.

\bibitem{clark1989cn2}
P.~Clark and T.~Niblett.
\newblock The cn2 induction algorithm.
\newblock {\em Machine learning}, 3(4):261--283, 1989.

\bibitem{cohen1995fast}
W.~W. Cohen.
\newblock Fast effective rule induction.
\newblock In {\em Machine Learning Proceedings 1995}, pages 115--123. Elsevier,
  1995.

\bibitem{conf/icde/DangANZS14}
X.~H. Dang, I.~Assent, R.~T. Ng, A.~Zimek, and E.~Schubert.
\newblock Discriminative features for identifying and interpreting outliers.
\newblock In {\em ICDE}, pages 88--99, 2014.

\bibitem{conf/pkdd/DangMAN13}
X.~H. Dang, B.~Micenkov\'{a}, I.~Assent, and R.~T. Ng.
\newblock Local outlier detection with interpretation.
\newblock In {\em ECML/PKDD}, pages 304--320, 2013.

\bibitem{conf/sigcomm/DaveGZ12}
V.~Dave, S.~Guha, and Y.~Zhang.
\newblock Measuring and fingerprinting click-spam in ad networks.
\newblock In {\em SIGCOMM}, pages 175--186. ACM, 2012.

\bibitem{deng2014interpreting}
H.~Deng.
\newblock Interpreting tree ensembles with intrees.
\newblock {\em arXiv preprint arXiv:1408.5456}, 2014.

\bibitem{fong2017interpretable}
R.~C. Fong and A.~Vedaldi.
\newblock Interpretable explanations of black boxes by meaningful perturbation.
\newblock {\em arXiv preprint arXiv:1704.03296}, 2017.

\bibitem{friedman2008predictive}
J.~H. Friedman, B.~E. Popescu, et~al.
\newblock Predictive learning via rule ensembles.
\newblock {\em The Annals of Applied Statistics}, 2(3):916--954, 2008.

\bibitem{gamberger2002expert}
D.~Gamberger and N.~Lavrac.
\newblock Expert-guided subgroup discovery: Methodology and application.
\newblock {\em Journal of Artificial Intelligence Research}, 17:501--527, 2002.

\bibitem{conf/soda/GharanV11}
S.~O. Gharan and J.~Vondrak.
\newblock Submodular maximization by simulated annealing.
\newblock In {\em SODA}, pages 1098--1116. SIAM, 2011.

\bibitem{conf/pkdd/GornitzKB09}
N.~G\"{o}rnitz, M.~Kloft, and U.~Brefeld.
\newblock Active and semi-supervised data domain description.
\newblock In {\em ECML/PKDD}, pages 407--422, 2009.

\bibitem{hara2016making}
S.~Hara and K.~Hayashi.
\newblock Making tree ensembles interpretable.
\newblock {\em arXiv preprint arXiv:1606.05390}, 2016.

\bibitem{conf/sdm/HeC10}
J.~He and J.~G. Carbonell.
\newblock Co-selection of features and instances for unsupervised rare category
  analysis.
\newblock In {\em SDM}, pages 525--536, 2010.

\bibitem{conf/icdm/HeTC10}
J.~He, H.~Tong, and J.~G. Carbonell.
\newblock Rare category characterization.
\newblock In {\em ICDM}, pages 226--235, 2010.

\bibitem{herrera2011overview}
F.~Herrera, C.~J. Carmona, P.~Gonz{\'a}lez, and M.~J. Del~Jesus.
\newblock An overview on subgroup discovery: foundations and applications.
\newblock {\em Knowledge and information systems}, 29(3):495--525, 2011.

\bibitem{conf/icde/KellerMB12}
F.~Keller, E.~M\"{u}ller, and K.~B\"{o}hm.
\newblock {HiCS}: High contrast subspaces for density-based outlier ranking.
\newblock In {\em ICDE}, pages 1037--1048, 2012.

\bibitem{conf/cikm/KellerMWB13}
F.~Keller, E.~M\"{u}ller, A.~Wixler, and K.~B\"{o}hm.
\newblock Flexible and adaptive subspace search for outlier analysis.
\newblock In {\em CIKM}, pages 1381--1390. ACM, 2013.

\bibitem{klosgen1996explora}
W.~Kl{\"o}sgen.
\newblock Explora: A multipattern and multistrategy discovery assistant.
\newblock In {\em Advances in knowledge discovery and data mining}, pages
  249--271. American Association for Artificial Intelligence, 1996.

\bibitem{klosgen2002census}
W.~Kl{\"o}sgen and M.~May.
\newblock Census data mining—an application.
\newblock In {\em Proceedings of the 6th European Conference on Principles and
  Practice of Knowledge Discovery in Databases (PKDD), Helsinki, Finland},
  2002.

\bibitem{conf/vldb/KnorrN99}
E.~M. Knorr and R.~T. Ng.
\newblock Finding intensional knowledge of distance-based outliers.
\newblock In {\em VLDB}, pages 211--222, 1999.

\bibitem{koh2017understanding}
P.~W. Koh and P.~Liang.
\newblock Understanding black-box predictions via influence functions.
\newblock {\em arXiv preprint arXiv:1703.04730}, 2017.

\bibitem{kopp2014interpreting}
M.~Kopp, T.~Pevn\'{y}, and M.~Holena.
\newblock Interpreting and clustering outliers with sapling random forests.
\newblock In {\em ITAT}, 2014.

\bibitem{conf/icdm/KriegelKRW05}
H.-P. Kriegel, P.~Kr\"oger, M.~Renz, and S.~H.~R. Wurst.
\newblock A generic framework for efficient subspace clustering of
  high-dimensional data.
\newblock In {\em ICDM}, 2005.

\bibitem{conf/pakdd/KriegelKSZ09}
H.-P. Kriegel, P.~Kr\"{o}ger, E.~Schubert, and A.~Zimek.
\newblock Outlier detection in axis-parallel subspaces of high dimensional
  data.
\newblock In {\em PAKDD}, pages 831--838, 2009.

\bibitem{Kriegel09}
H.-P. Kriegel, P.~Kr\"{o}ger, and A.~Zimek.
\newblock Clustering high-dimensional data: A survey on subspace clustering,
  pattern-based clustering, and correlation clustering.
\newblock {\em ACM Trans. Knowl. Discov. Data}, 3(1):1--58, 2009.

\bibitem{conf/icdm/KriegelKSZ12}
H.-P. Kriegel, P.~Kröger, E.~Schubert, and A.~Zimek.
\newblock Outlier detection in arbitrarily oriented subspaces.
\newblock In {\em ICDM}, pages 379--388, 2012.

\bibitem{conf/aaai/KuoD16}
C.-T. Kuo and I.~Davidson.
\newblock A framework for outlier description using constraint programming.
\newblock In {\em AAAI}, pages 1237--1243, 2016.

\bibitem{journals/corr/LakkarajuKCL17}
H.~Lakkaraju, E.~Kamar, R.~Caruana, and J.~Leskovec.
\newblock Interpretable and explorable approximations of black box models.
\newblock {\em CoRR}, abs/1707.01154, 2017.

\bibitem{conf/kdd/LazarevicK05}
A.~Lazarevic and V.~Kumar.
\newblock Feature bagging for outlier detection.
\newblock In {\em KDD}, pages 157--166, 2005.

\bibitem{conf/icwsm/LeeEC11}
K.~Lee, B.~D. Eoff, and J.~Caverlee.
\newblock Seven months with the devils: A long-term study of content polluters
  on twitter.
\newblock In {\em ICWSM}, 2011.

\bibitem{conf/icdm/LiuTZ08}
F.~T. Liu, K.~M. Ting, and Z.-H. Zhou.
\newblock Isolation forest.
\newblock In {\em ICDM}, 2008.

\bibitem{conf/cikm/LoekitoB08}
E.~Loekito and J.~Bailey.
\newblock Mining influential attributes that capture class and group contrast
  behaviour.
\newblock In {\em CIKM}, pages 971--980. ACM, 2008.

\bibitem{conf/icdm/MicenkovaNDA13}
B.~Micenkov\'{a}, R.~T. Ng, X.~H. Dang, and I.~Assent.
\newblock Explaining outliers by subspace separability.
\newblock In {\em ICDM}, pages 518--527, 2013.

\bibitem{conf/icdm/MoiseSE06}
G.~Moise, J.~Sander, and M.~Ester.
\newblock P3c: A robust projected clustering algorithm.
\newblock In {\em ICDM}, pages 414--425, 2006.

\bibitem{montavon2017methods}
G.~Montavon, W.~Samek, and K.-R. M{\"u}ller.
\newblock Methods for interpreting and understanding deep neural networks.
\newblock {\em Digital Signal Processing}, 2017.

\bibitem{conf/icwsm/MukherjeeV0G13}
A.~Mukherjee, V.~Venkataraman, B.~Liu, and N.~S. Glance.
\newblock What yelp fake review filter might be doing?
\newblock In {\em ICWSM}, 2013.

\bibitem{conf/icdm/MullerAGKS09}
E.~M\"{u}ller, I.~Assent, S.~G\"{u}nnemann, R.~Krieger, and T.~Seidl.
\newblock Relevant subspace clustering: Mining the most interesting
  non-redundant concepts in high dimensional data.
\newblock In {\em ICDM}, pages 377--386. IEEE, 2009.

\bibitem{conf/icdm/MullerAIMB12}
E.~M\"{u}ller, I.~Assent, P.~I. Sanchez, Y.~Mulle, and K.~B\"{o}hm.
\newblock Outlier ranking via subspace analysis in multiple views of the data.
\newblock In {\em ICDM}, pages 529--538, 2012.

\bibitem{conf/icde/MullerASS08}
E.~M\"{u}ller, I.~Assent, U.~Steinhausen, and T.~Seidl.
\newblock Outrank: ranking outliers in high dimensional data.
\newblock In {\em ICDE Workshops}, pages 600--603, 2008.

\bibitem{conf/icde/MullerSS11}
E.~M\"{u}ller, M.~Schiffer, and T.~Seidl.
\newblock Statistical selection of relevant subspace projections for outlier
  ranking.
\newblock In {\em ICDE}, pages 434--445, 2011.

\bibitem{Parsons2004}
L.~Parsons, E.~Haque, and H.~Liu.
\newblock {Subspace clustering for high dimensional data: a review}.
\newblock 6(1):90--105, 2004.

\bibitem{pelleg00xmeans}
D.~Pelleg and A.~Moore.
\newblock ${X}$-means: {E}xtending ${K}$-means with efficient estimation of the
  number of clusters.
\newblock In {\em ICML}, pages 727--734, 2000.

\bibitem{pevny2014interpreting}
T.~Pevn\'{y} and M.~Kopp.
\newblock Explaining anomalies with sapling random forests.
\newblock In {\em ITAT}, 2014.

\bibitem{ribeiro2016should}
M.~T. Ribeiro, S.~Singh, and C.~Guestrin.
\newblock Why should i trust you?: Explaining the predictions of any
  classifier.
\newblock In {\em Proceedings of the 22nd ACM SIGKDD International Conference
  on Knowledge Discovery and Data Mining}, pages 1135--1144. ACM, 2016.

\bibitem{Rissanen78}
J.~Rissanen.
\newblock Modeling by shortest data description.
\newblock {\em Automatica}, 14:465--471, 1978.

\bibitem{conf/icdm/SequeiraZ04}
K.~Sequeira and M.~J. Zaki.
\newblock Schism: A new approach for interesting subspace mining.
\newblock In {\em ICDM}, pages 186--193, 2004.

\bibitem{silverman2018density}
B.~W. Silverman.
\newblock {\em Density estimation for statistics and data analysis}.
\newblock Routledge, 2018.

\bibitem{journals/ml/TaxD04}
D.~M.~J. Tax and R.~P.~W. Duin.
\newblock Support vector data description.
\newblock {\em Machine Learning}, 54(1):45--66, 2004.

\bibitem{vreeken2011krimp}
J.~Vreeken, M.~Van~Leeuwen, and A.~Siebes.
\newblock Krimp: mining itemsets that compress.
\newblock {\em Data Mining and Knowledge Discovery}, 23(1):169--214, 2011.

\bibitem{wrobel1997algorithm}
S.~Wrobel.
\newblock An algorithm for multi-relational discovery of subgroups.
\newblock In {\em European Symposium on Principles of Data Mining and Knowledge
  Discovery}, pages 78--87. Springer, 1997.

\bibitem{conf/edbt/ZhangDM17}
H.~Zhang, Y.~Diao, and A.~Meliou.
\newblock {EXstream}: Explaining anomalies in event stream monitoring.
\newblock In {\em EDBT}, pages 156--167, 2017.

\end{thebibliography}

%

\end{document}